\documentclass[12pt]{article}
\usepackage[utf8]{inputenc}
\usepackage{cite}
\usepackage{amsfonts}
\usepackage{amsmath}
\usepackage{amsthm}
\usepackage{enumitem}
\usepackage{amssymb}
\usepackage{caption}
\usepackage{graphicx}
\usepackage{algorithm}
\usepackage{algpseudocode}
\usepackage{multirow}
\usepackage{xcolor}
\usepackage{setspace}
\usepackage{booktabs}
\usepackage{latexsym, mathdots, array, extarrows}
\usepackage{float}
\usepackage{lipsum}
\usepackage{pifont}
\usepackage{diagbox}

\usepackage{amsmath,amsfonts,amsthm,amssymb}
\usepackage{bm,dsfont}
\usepackage{booktabs}
\usepackage{graphicx}
\usepackage{array}
\usepackage{xcolor}
\usepackage{sidecap}
\usepackage{listings}
\usepackage{times}
\usepackage{url}
\usepackage{hyperref}
\usepackage{subcaption}
\usepackage{makecell} 

\usepackage{algorithm}      

\usepackage{cleveref}

\usepackage{pgfplots}
\pgfplotsset{compat=1.18}
\usepackage{subcaption}

\usepackage{geometry}
\newtheorem{theorem}{Theorem}
\newtheorem{proposition}{Proposition}
\newtheorem{lemma}{Lemma}
\newtheorem{assumption}{Assumption}
\newtheorem{corollary}{Corollary}

\newcommand{\ba}{\bm{a}}

\newcommand{\bG}{\bm{G}}
\newcommand{\bD}{\bm{D}}

\newcommand{\bQ}{\bm{Q}}
\newcommand{\bH}{\bm{H}}

\newcommand{\bx}{\bm{x}}
\newcommand{\bU}{\bm{U}}
\newcommand{\bz}{\bm{z}}
\newcommand{\bA}{\bm{A}}
\newcommand{\bB}{\bm{B}}
\newcommand{\bu}{\bm{u}}

\newcommand{\bb}{\bm{b}}

\newcommand{\bW}{\bm{W}}

\newcommand{\bv}{\bm{v}}
\newcommand{\bg}{\bm{g}}

\newcommand{\bSig}{\bm{\Sigma}}

\newcommand{\bxi}{\bm{\xi}}
\newcommand{\br}{\bm{r}}
\newcommand{\bzero}{\bm{0}}
\newcommand{\bJ}{\bm{J}}
\newcommand{\be}{\bm{e}}
\newcommand{\bI}{\bm{I}}

\newcommand{\bM}{\bm{M}}
\newcommand{\mJ}{\mathcal{J}}
\newcommand{\bd}{\bm{d}}
\newcommand{\bR}{\bm{R}}
\newcommand{\bE}{\bm{E}}
\newcommand{\bP}{\bm{P}}
\newcommand{\bone}{\bm{1}}
\newcommand{\bw}{\bm{w}}
\newcommand{\bmm}{\bm{m}}
\newcommand{\bC}{\bm{C}}
\newcommand{\bN}{\bm{N}}

\newcommand{\bdelta}{\bm{\delta}}
\newcommand{\bq}{\bm{q}}

\allowdisplaybreaks[4]

\title{
Statistical Inference for Rank Allocation\\
in Low-Rank Adaptation
}

\author{Yihang Gao\thanks{Department of Mathematics, National University of Singapore. Email: gaoyh@connect.hku.hk}\and Vincent Y. F.  Tan\thanks{Department of Mathematics and Department of Electrical and Computer Engineering, National University of Singapore. Email: vtan@nus.edu.sg}}

\date{}

\begin{document}

\maketitle

\begin{abstract}
Low-rank adaptation (LoRA) has become a widely used parameter-efficient fine-tuning method for large language models, owing to its computational efficiency, empirical effectiveness, and ease of implementation. 
Since different modules and layers may contribute unequally to downstream adaptation, allocating rank resources under a fixed parameter budget is an important problem for balancing efficiency, expressiveness, and generalization. 
Existing adaptive rank methods address this problem mainly through carefully designed importance scores constructed from gradient-derived sensitivity and uncertainty measures, without an explicit statistical interpretation.
In this paper, we formulate LoRA rank allocation as a statistical hypothesis testing problem and propose StatLoRA, a statistical inference-based rank allocation method. 
StatLoRA associates each LoRA component with a test statistic and uses estimated p-values to determine which components should be retained or pruned under a prescribed rank budget. 
Components with the weakest statistical evidence of contribution are removed first, yielding an allocation rule with explicit uncertainty quantification. 
The proposed testing procedure is supported by our central limit theory for stochastic optimizer trajectories. 
In particular, we establish asymptotic normality for a broad class of commonly used optimizers in deep learning, including AdamW, Adam, and Adafactor, and derive the corresponding asymptotic distributions for the proposed component scores used in hypothesis testing.
We evaluate StatLoRA on LoRA fine-tuning of DeBERTaV3-base, BART-Large, and Qwen2.5-7B across natural language understanding, natural language generation, and question answering tasks. 
Experiments show that StatLoRA achieves comparable or better performance than vanilla LoRA, AdaLoRA, and IGU-LoRA under matched rank budgets.
Sensitivity analyses and empirical diagnostics further support the stability of the proposed hypothesis-testing-based allocation rule and provide empirical evidence for the asymptotic theory of component scores.

\end{abstract}

\section{Introduction}

Large language models (LLMs) and related foundation models have demonstrated remarkable capabilities across a wide range of applications, including natural language processing~\cite{brown2020language,devlin2019bert}, computer vision~\cite{radford2021learning,li2023blip}, and scientific tasks~\cite{singhal2023large,m2024augmenting}. Their successful deployment to downstream tasks, especially in professional domains and specialized communities, often relies on fine-tuning to adapt pretrained models to task-specific data and objectives~\cite{lee2020biobert,chung2024scaling}. 
While full fine-tuning can be effective, updating all model parameters is computationally expensive and memory intensive, particularly for modern large-scale scenarios.
Parameter-efficient fine-tuning methods alleviate this difficulty by introducing a small number of trainable parameters while keeping most pretrained weights fixed, thereby reducing the computational and storage costs of model adaptation~\cite{houlsby2019parameter,li2021prefix}. 
Among these fine-tuning methods, low-rank adaptation (LoRA)~\cite{hu2022lora} has become a widely used approach due to its simplicity, efficiency, and strong empirical performance. LoRA represents task-specific updates through trainable low-rank matrices inserted into selected model modules, reducing the number of trainable parameters while preserving sufficient adaptation capacity.

A central question in LoRA fine-tuning is how to allocate a limited adaptation budget across layers and modules~\cite{zhang2023adaptive,valipour2023dylora}. Different parts of a model may play different roles: some modules may encode general-purpose knowledge that should be largely preserved, whereas others may need stronger adaptation to capture domain-specific information. 
Since the rank of each LoRA matrix controls the adaptation capacity, an effective allocation should assign more resources to components with strong statistical evidence while avoiding unnecessary over-parameterization. 
This creates a high-dimensional resource allocation problem under stochastic training dynamics, where observed component scores serve as noisy measures of component contributions and are affected by mini-batch noise and optimizer-induced randomness. 
From a statistical perspective, adaptive rank allocation should not rely merely on the magnitude of empirical component scores but should also account for the uncertainty in estimating component contributions.
Motivated by this viewpoint, this paper develops a statistically principled resource allocation method for LoRA based on hypothesis testing.

Due to the matrix factorization structure of LoRA, the adaptation in each module can be decomposed into a sum of rank-one components.
Consequently, allocating the rank budget is equivalent to deciding which rank-one components should be retained or removed across layers and modules. 
This perspective naturally leads to a component selection problem: each candidate component represents a potential direction of adaptation, and its empirical magnitude reflects the extent to which the pretrained model is modified along that direction. 
However, during fine-tuning, these quantities are generated by stochastic optimization algorithms and are affected by mini-batch sampling, adaptive step sizes (learning rates), momentum, and other sources of training noise. 
Therefore, deciding whether a LoRA component is useful should not be viewed merely as a deterministic ranking problem. 
Rather, it can be formulated as a statistical inference problem in which one seeks to distinguish genuine adaptation signals from stochastic observations.

Several adaptive-rank methods have been proposed to address LoRA rank allocation. AdaLoRA \cite{zhang2023adaptive} is among the first methods to explicitly identify the importance of adaptive resource allocation and develops a practical procedure based on component-wise importance scores. 
These scores are constructed from gradient-derived sensitivity measures and smoothed magnitude information, with gradients interpreted as indicators of how sensitive the training objective is to each component. 
More recently, IGU-LoRA~\cite{cui2026igulora} further refined this idea by incorporating uncertainty-aware and layer-wise grouping corrections into the component importance measure, aiming to stabilize rank allocation under noisy training dynamics, inspired by signal-to-noise ratio. 
Other automatic rank allocation methods include dynamic rank training~\cite{valipour2023dylora}, gating-based rank allocation~\cite{liu2024alora}, and meta-learning-based rank selection~\cite{zhang2024autolora}.
Although these methods are empirically effective, their importance scores are constructed through carefully designed combinations of gradient-based magnitude and sensitivity. 
Such combinations are intuitively motivated by stochasticity in training dynamics, but they are not derived from an explicit statistical model, an optimization criterion, or a sampling distribution for the component signals.
In particular, they do not provide p-values, error control, or a formal criterion for determining whether an observed component signal is distinguishable from stochastic optimization noise.
Bayesian approaches provide another route to uncertainty-aware fine-tuning by placing prior distributions on model parameters and using posterior inference to quantify uncertainty~\cite{neal2012bayesian,blundell2015weight}. 
Recent Bayesian variants of LoRA adapt this idea to parameter-efficient fine-tuning by treating LoRA parameters or low-rank adaptation components as random variables and approximating their posterior distributions~\cite{yang2024bayesian,duan2026bara}.
However, these methods typically modify the training procedure by introducing priors, variational approximations, and posterior sampling, which may increase computational cost and alter the optimization dynamics of standard LoRA fine-tuning. 
These limitations motivate the development of a framework that preserves the standard training procedure while providing statistical decisions for rank allocation.

\subsection{Overview: Statistical Inference for Rank Allocation}
In this paper, we address LoRA rank allocation through a statistical inference framework.
The central question is to determine which rank-one LoRA components provide sufficient statistical evidence of contribution to downstream adaptation.
Let $s_{\ell,j}^{\star}$ denote the importance score of the $j$-th component in the $\ell$-th module, where $\ell=1,\cdots,L$ and $j=1,\cdots,r$.
This score measures the contribution of the corresponding rank-one update, with larger values indicating stronger evidence for retaining the component.
In practice, $s_{\ell,j}^{\star}$ is not directly observed and must be estimated from the stochastic training trajectory.
Since LoRA fine-tuning involves mini-batch sampling, momentum, and adaptive step sizes, the resulting empirical component scores are random quantities generated by the optimizer dynamics. 
This observation motivates a statistical formulation of rank allocation. Instead of treating empirical component scores as deterministic measures of importance, we view them as noisy estimates of underlying adaptation signals and make retain-or-prune decisions through hypothesis testing.
Specifically, for a pre-specified threshold $\Delta > 0$, we consider the one-sided hypothesis testing problem
\begin{equation}
\mathcal{H}_{\ell,j}^{0}:~s_{\ell,j}^{\star} \geq \Delta
\qquad \text{versus} \qquad
\mathcal{H}_{\ell,j}^{1}:~s_{\ell,j}^{\star} < \Delta .
\end{equation}
Here, the null hypothesis represents the event that the component has a sufficient contribution to be retained, whereas the alternative represents that its contribution falls below the threshold.
Under this formulation, components with strong evidence against $\mathcal{H}_{\ell,j}^{0}$ are removed first until the prescribed rank budget is reached.

The key step in implementing this testing procedure is to characterize the distribution of score estimates for components along the training trajectory. 
Motivated by asymptotic theory for stochastic approximation and stochastic gradient descent~\cite{polyak1992acceleration,kushner2003stochastic}, we establish central limit theory for adaptive optimizers commonly used in LoRA fine-tuning, including AdamW~\cite{loshchilov2018decoupled}, Adam~\cite{kingma2015adam}, and Adafactor~\cite{shazeer2018adafactor}. 
In particular, for the optimizer state or iterate sequence $\{\bz_t\}$, we show that, under suitable regularity conditions,
\begin{equation}
\gamma_t^{-1/2}\left(\bz_t-\bz^{\star}\right)
\overset{d}{\longrightarrow}
\mathcal{N}\left(\bzero,\bSig\right),
\end{equation}
and the Polyak--Ruppert average trajectory satisfies
\begin{equation}
\sqrt{t}\left(\frac{1}{t}\sum_{\tau=1}^{t}\bz_{\tau}-\bz^{\star}\right)
\overset{d}{\longrightarrow}
\mathcal{N}\left(\bzero,\bm{\Omega}\right),
\end{equation}
where $\bz^{\star}$ denotes the limiting point of the optimizer dynamics, $\{\gamma_t\}$ denotes the step-size sequence, and $\bSig,\bm{\Omega}\succeq \bzero$ are covariance matrices determined by the local curvature, stochastic gradient noise, and optimizer-specific dynamics.

We define the empirical component score $s_{\ell,j}(t)$, which is an estimator for $s_{\ell,j}^{\star}$, as a function of the current iterate $\bz_t$ at time step $t$, then the delta method further yields the weak convergence
\begin{equation}
\label{eq_clt_s}
\sqrt{t}\left(
\frac{1}{t}\sum_{\tau=1}^{t}s_{\ell,j}(\tau)
-s_{\ell,j}^{\star}
\right)
\overset{d}{\longrightarrow}
\mathcal{N}\left(0,\bar\sigma_{\ell,j}^{2}\right),
\end{equation}
where $\bar\sigma_{\ell,j}^{2}$ depends on the covariance $\bm{\Omega}$ and the form of the component score function. Based on \eqref{eq_clt_s}, we construct the test statistic in the boundary case
\begin{equation*}
T_{\ell,j}(t)
=
\frac{
\sqrt{t}
\left(
\frac{1}{t}\sum_{\tau=1}^{t}s_{\ell,j}(\tau)-\Delta
\right)
}{\widehat{\bar\sigma}_{\ell,j}},
\end{equation*}
where $\widehat{\bar\sigma}_{\ell,j}$ is estimated from the historical score trajectory using the batch-means estimator. The corresponding one-sided p-value is given by
\begin{equation*}
p_{\ell,j}(t)
=
\Phi\left(T_{\ell,j}(t)\right),
\end{equation*}
where $\Phi(\cdot)$ denotes the cumulative distribution function of the standard normal distribution. 
A small p-value provides evidence against the null hypothesis, indicating that the component has insufficient statistical evidence for retention relative to the threshold $\Delta$.

This leads to our proposed method, StatLoRA, which integrates statistical inference with adaptive LoRA rank allocation. 
Compared with existing methods, StatLoRA provides a statistically interpretable decision rule for component selection while preserving the standard LoRA training procedure.
In particular, it does not require modifying the training loss, changing the model parameterization, or altering the optimizer dynamics. Instead, it uses the stochastic trajectory and score quantities already generated during training to estimate the uncertainty of component contributions and guide retain-or-prune decisions.

\subsection{Related Literature and Contributions}

Our theoretical development is closely related to the asymptotic theory of stochastic approximation and stochastic optimization~\cite{robbins1951stochastic,borkar2008stochastic}.
Classical results on Polyak–Ruppert averaging establish asymptotic normality and optimal covariance properties for averaged stochastic approximation recursions, with averaged stochastic gradient descent as a central example~\cite{polyak1992acceleration,ruppert1988efficient}. 
Subsequent studies have developed central limit theory and asymptotic distributions for stochastic quasi-Newton~\cite{leluc2023asymptotic}, non-smooth~\cite{davis2024asymptotic}, constrained~\cite{duchi2021asymptotic}, and inexact~\cite{na2025statistical} methods. 
These distributional results provide a basis for statistical inference from optimization trajectories, including uncertainty quantification~\cite{toulis2017asymptotic,toulis2014statistical,wang2026inference}, confidence intervals~\cite{fang2018online,chen2024online,li2018statistical}, hypothesis testing~\cite{chen2020statistical,chen2021statistical}, and convergence diagnostics~\cite{chee2018convergence,cowles1996markov}.
However, most existing statistical inference results focus on standard stochastic gradient methods and their averaged variants. 
Modern deep learning, such as LoRA fine-tuning, in contrast, is typically performed with adaptive optimizers such as AdamW~\cite{loshchilov2018decoupled}, Adam~\cite{kingma2015adam}, and Adafactor~\cite{shazeer2018adafactor}, whose momentum terms, coordinate-wise scaling, and adaptive learning rates lead to more complex stochastic dynamics. 
Our work extends this literature by establishing central limit theory for iterates generated by adaptive optimizers and by using the resulting asymptotic distributions to construct hypothesis tests for LoRA rank allocation.

The main contributions of this paper are summarized as follows:
\begin{itemize}
    \item We formulate adaptive rank allocation for LoRA as a component-wise statistical hypothesis testing problem, providing a statistical inference perspective on LoRA resource allocation under stochastic training dynamics.

    \item We establish central limit theory for the trajectories of adaptive optimizers commonly used in deep learning and large-scale LLM fine-tuning, including AdamW, Adam, and Adafactor, under suitable regularity conditions. These results provide asymptotic distributions for optimizer states and, through the delta method, for empirical LoRA component scores.

    \item We propose StatLoRA, a practical inference-based rank allocation algorithm that uses the derived asymptotic distributions to quantify uncertainty in component scores. Under a prescribed total rank budget, StatLoRA retains components with sufficient statistical evidence for contribution and prunes components with weak evidence for retention.

    \item We conduct experiments on LoRA fine-tuning with DeBERTaV3-base, BART-Large, and Qwen2.5-7B across natural language understanding, natural language generation, and question answering tasks. The results show that StatLoRA achieves competitive performance compared with vanilla LoRA, AdaLoRA, and IGU-LoRA under matched rank budgets. We further conduct sensitivity analyses for the statistical hyperparameters, examine the stability of the rank allocation, and provide empirical diagnostics for the asymptotic normality of the empirical component scores.
\end{itemize}

\subsection{Organization and Notation}

The remainder of this paper is organized as follows. Section~2 introduces the background on LoRA fine-tuning and formulates LoRA rank allocation as a hypothesis testing problem on each component. Section~3 develops asymptotic theory for adaptive optimizers commonly used in deep learning, with particular emphasis on AdamW. Section~4 derives central limit results for empirical component scores and presents the proposed StatLoRA algorithm, including practical variance estimation and p-value construction. Section~5 reports the experimental results. Section~6 concludes the paper and discusses possible directions for future work.

Throughout this paper, we use lowercase letters, bold lowercase letters, and bold uppercase letters, such as $a$, $\ba$, and $\bA$, to denote scalars, vectors, and matrices, respectively. The notation $\|\cdot\|$ represents the Euclidean norm for vectors and the spectral norm for matrices, while $\|\cdot\|_{\mathrm{F}}$ denotes the Frobenius norm. 
For two deterministic positive sequences $\{a_t\}$ and $\{b_t\}$, we write $a_t=o(b_t)$ if $a_t/b_t\to 0$, and $a_t=\mathcal{O}(b_t)$ if $a_t/b_t$ is bounded. For random sequences $\{X_t\}$ and positive deterministic sequences $\{a_t\}$, we write $X_t=o_p(a_t)$ if $X_t/a_t\to 0$ in probability, and $X_t=\mathcal{O}_p(a_t)$ if $X_t/a_t$ is bounded in probability.
For vector- or matrix-valued quantities, these asymptotic notations are understood with respect to the norm $\|\cdot\|$.
For vectors $\ba,\bb\in\mathbb{R}^d$, we use $\ba\odot\bb$ to denote the Hadamard product, i.e., element-wise multiplication. In particular, $\ba^{\odot 2}:=\ba\odot\ba$ denotes the element-wise square of $\ba$.
Other vector operations, including $\sqrt{\ba}$, $\ba/\bb$, and $\ba+\epsilon$, are understood entry-wise whenever they appear.
Moreover, for a vector $\ba\in\mathbb{R}^d$, $\operatorname{diag}(\ba)\in\mathbb{R}^{d\times d}$ denotes the diagonal matrix whose diagonal entries are given by the entries of $\ba$.
When stochastic processes are involved, all random variables are defined on a filtered probability space
$    (\Omega,\mathcal F,\{\mathcal F_t\}_{t\ge 0},\mathbb P),$
where $\{\mathcal F_t\}_{t\ge 0}$ is an increasing sequence of sub-$\sigma$-algebras of $\mathcal F$. Here, $\mathcal F_t$ represents the information available before the stochastic update at iteration $t$. The notation $\mathbb E[\cdot\mid\mathcal F_t]$ denotes conditional expectation with respect to $\mathcal F_t$.

\section{LoRA Rank Allocation as a Statistical Decision Problem}

In this section, we first review the basic formulation of LoRA and the associated rank allocation problem. 
We then use the rank-one decomposition of LoRA updates to show that rank allocation can be represented as a component selection problem. 
Unlike existing methods that primarily rely on various designs of importance measures, we formulate component selection as a statistical decision problem and introduce a hypothesis testing framework for LoRA rank allocation.

\subsection{Low-Rank Adaptation and Rank Allocation}

In transfer learning, a fine-tuned parameter matrix $\bW_{\text{ft}}$ is obtained by adapting a pretrained parameter matrix $\bW_{\text{pt}}$:
\begin{equation*}
    \bW_{\text{ft}} = \bW_{\text{pt}} + \Delta\bW,
\end{equation*}
where $\bW_{\text{ft}}, \bW_{\text{pt}},\Delta\bW \in \mathbb{R}^{n_1 \times n_2}$. Standard full fine-tuning directly updates the entire adaptation matrix $\Delta\bW$, providing full flexibility on the adaptation matrix $\Delta\bW$, but incurring substantial computational, memory, and storage costs. 
Low-rank adaptation (LoRA)~\cite{hu2022lora} addresses this issue by imposing a low-rank structure on the adaptation matrix. Specifically, the update is parameterized as
\begin{equation*}
    \Delta\bW = \bB \bA,
\end{equation*}
where $\bA \in \mathbb{R}^{r \times n_2}$, $\bB \in \mathbb{R}^{n_1 \times r}$, and $r \ll \min\{n_1, n_2\}$.
Instead of training a full matrix of size $n_1 n_2$, LoRA only updates the factor matrices $\bA$ and $\bB$, reducing the number of trainable parameters to $r(n_1 + n_2)$. 
Beyond computational efficiency, the low-rank parameterization restricts adaptation to a lower-dimensional subspace, which can serve as an implicit regularization mechanism and improve the stability and generalization of fine-tuning~\cite{zeng2024expressive,zhu2024asymmetry,malladi2023kernel}.

For LoRA-based fine-tuning of a multi-layer neural network $\phi(\cdot;\Theta)$,  such as a transformer-based large language model~\cite{vaswani2017attention}, let $\Theta := \{\bW_1, \bW_2, \cdots, \bW_{L}\}$ denote the collection of parameter matrices to which LoRA is applied. For the $\ell$-th parameter matrix, LoRA takes the form
\begin{equation}
\label{eq_lora}
    \bW_{\text{ft},\ell} = \bW_{\text{pt},\ell} + \Delta\bW_{\ell} = \bW_{\text{pt},\ell} + \bB_{\ell} \bA_{\ell},
\end{equation}
where $\bW_{\text{ft},\ell}$, $\bW_{\text{pt},\ell}$, and $\Delta\bW_{\ell} \in \mathbb{R}^{n \times n}$ denote the fine-tuned parameter matrix, the pretrained parameter matrix, and the adaptation matrix, respectively. 
Here $\bA_{\ell} \in \mathbb{R}^{r \times n}$ and $\bB_{\ell} \in \mathbb{R}^{n \times r}$ are the LoRA factor matrices associated with the $\ell$-th parameter matrix. For notational simplicity, we present the formulation for square matrices of size $n \times n$, and the same discussion extends directly to rectangular matrices.

Since different layers and modules may play different roles in the neural network, a uniform rank allocation across all parameter matrices may be inefficient~\cite{zhang2023adaptive}. 
Some modules may encode general-purpose knowledge that should be largely preserved, whereas others may require stronger adaptation to capture task-specific or domain-specific information.
Consequently, a small rank may be sufficient for some parameter matrices, while a larger rank may be needed for others. 
On the other hand, assigning unnecessarily large ranks can increase computational cost and may introduce excessive flexibility, potentially harming stability and generalization. 
In many standard LoRA implementations, these ranks are treated as fixed hyperparameters, which may lead to suboptimal performance.
This motivates rank allocation, which allows different layers or modules to use different ranks rather than imposing a uniform rank across the model~\cite{valipour2023dylora,cui2026igulora}.

Several methods have been proposed for automatic rank allocation in LoRA fine-tuning~\cite{valipour2023dylora,liu2024alora,zhang2024autolora}. 
AdaLoRA~\cite{zhang2023adaptive} is a pioneering work in this direction, introducing a budget allocation procedure based on importance scores for rank-one components in the LoRA update $\bB_{\ell}\bA_{\ell}$.
Its score construction uses gradient-derived sensitivity and magnitude information to identify components that are likely to contribute more to model adaptation. 
More recently, IGU-LoRA~\cite{cui2026igulora} further refines component scoring by incorporating integrated-gradient information within the layer and uncertainty-aware adjustments, with the goal of stabilizing importance estimation under noisy training dynamics. 
Although these methods are empirically useful, their importance scores are typically constructed using a carefully designed combination of gradient magnitude and sensitivity. Such scoring rules are intuitively motivated by stochasticity in training dynamics, but they are not derived from an explicit statistical model, a statistical inference criterion, or a sampling distribution for component signals. 
As a result, it remains unclear how to interpret the numerical value of such scores, how to compare them across layers and modules, or how to distinguish genuine adaptation signals from optimizer-induced stochastic noise.

\subsection{Statistical Formulation of LoRA Rank Allocation}

We now formulate the LoRA rank allocation as a statistical decision problem. Suppose that, for each layer or module $\ell$, LoRA is initialized with a relatively large rank $r$. The adaptation matrix can then be decomposed into a sum of rank-one components: 
\begin{equation}
    \Delta\bW_{\ell} = \bB_{\ell} \bA_{\ell} = \sum_{j=1}^{r} \bb_{\ell,j} \ba_{\ell,j}^{\top},
\end{equation}
where $\bb_{\ell,j} \in \mathbb{R}^{n}$ denotes the $j$-th column of $\bB_{\ell}$, and $\ba_{\ell,j}^{\top} \in \mathbb{R}^{1 \times n}$ denotes the $j$-th row of $\bA_{\ell}$. Thus, each product $\bb_{\ell,j} \ba_{\ell,j}^{\top}$ represents a rank-one adaptation component.

Under this decomposition, rank allocation can be viewed as selecting which rank-one components should be retained or removed. 
Let $r_{\ell} \leq r$ denote the number of retained components in the $\ell$-th module. Given a total rank budget $R$, the selected ranks are required to satisfy  \begin{equation*}
    \sum_{\ell=1}^{L} r_{\ell} \leq R.
\end{equation*}
Equivalently, among the $rL$ candidate rank-one components $\bb_{\ell,j} \ba_{\ell,j}^{\top}$, for $\ell=1,\cdots,L$ and $j = 1, \cdots, r$, one needs to select at most $R$ components to retain while removing the rest. 
Exhaustively searching over all possible subsets is computationally infeasible, since the search space grows combinatorially with $rL$. 
Moreover, evaluating the downstream performance of many candidate allocations would require repeated fine-tuning or validation, which is computationally expensive for large language models.

To avoid such an exhaustive search, we propose to make component-wise decisions through statistical inference. Let
\begin{equation*}
    \Delta\bW_{\ell,j}^{\star} = \bb_{\ell,j}^{\star}\ba_{\ell,j}^{\star\top}
\end{equation*}
denote the $j$-th rank-one component in the $\ell$-th module.
Here, the superscript $\star$ refers to the limiting value of the corresponding LoRA components under the stochastic fine-tuning dynamics, rather than to an observed finite-iteration estimate.
In this sense, $\Delta\bW_{\ell,j}^{\star}$ plays the role of the population-level target component in the LoRA framework.
We measure the strength of this component and define the population score by the squared Frobenius norm
\begin{equation}
\label{eq_score_population}
    s_{\ell,j}^{\star} = \left\|\Delta\bW_{\ell,j}^{\star} \right\|_{\mathrm{F}}^{2} = \left\|\bb_{\ell,j}^{\star}\ba_{\ell,j}^{\star\top} \right\|_{\mathrm{F}}^{2} = \left\| \bb_{\ell,j}^{\star}\right\|^{2} \cdot \left\| \ba_{\ell,j}^{\star}\right\|^{2}. 
\end{equation}
The quantity $s_{\ell,j}^{\star}$ represents the underlying adaptation signal of the corresponding component. In practice, however, $s_{\ell,j}^{\star}$ is not directly observed. We only observe empirical estimates generated along the stochastic training trajectory, which are affected by mini-batch sampling, optimizer momentum, and adaptive step sizes.

We therefore formulate component selection through the following one-sided hypothesis test. 
For a pre-specified threshold $\Delta>0$, consider
\begin{equation}
\label{hypothesis}
\mathcal{H}_{\ell,j}^{0}:~s_{\ell,j}^{\star} \geq \Delta
\qquad
\text{versus}
\qquad
\mathcal{H}_{\ell,j}^{1}:~s_{\ell,j}^{\star} < \Delta .
\end{equation}
Here, the null hypothesis represents that the component has a sufficiently large adaptation signal and should be retained, while the alternative represents that its contribution falls below the threshold.
This choice of hypotheses is aligned with the pruning objective: a small p-value provides evidence against $\mathcal{H}_{\ell,j}^{0}$, suggesting that the corresponding component has insufficient statistical evidence for retention relative to $\Delta$.
Given the total rank budget $R$, we compute p-values for all candidate components and prune those with the strongest evidence against the null hypothesis until the budget constraint is satisfied.

The remaining key ingredient is the distribution of the empirical component scores, which will be defined later. Since the exact scores $s_{\ell,j}^{\star}$ are inaccessible, they must be estimated from the iterates generated during LoRA fine-tuning. This is nontrivial because the training trajectory is a complex stochastic process affected by mini-batch sampling, momentum, coordinate-wise scaling, and adaptive step sizes.
As a result, the empirical component scores, being functions of these iterates, inherit the uncertainty and temporal dependence induced by the training dynamics. 
If this uncertainty can be characterized through an asymptotic distribution of the empirical score process, then the hypothesis test in \eqref{hypothesis} can be implemented using estimated p-values. 
Thus, the central theoretical question is to characterize the asymptotic distribution of iterates generated by commonly used adaptive optimizers in LoRA fine-tuning, and then to transfer this distributional result to component score estimates.
We address this question in the next section by first establishing central limit theory for stochastic optimizer trajectories and then applying the delta method to obtain asymptotic distributions for empirical component scores.

\section{Asymptotic Theory for Stochastic Optimization}

To implement the hypothesis testing framework in \eqref{hypothesis}, we need to characterize the distributional behavior of the quantities generated during LoRA fine-tuning.
In this section, we study the asymptotic distribution of iterates produced by stochastic optimization algorithms commonly used in LoRA fine-tuning of LLMs, including AdamW, Adam, and Adafactor. 
We first develop a general central limit theory for stochastic approximation recursions. We then show that these adaptive optimizers can be represented within this framework by introducing suitable augmented state variables and optimizer-specific mean fields. These results provide the distributional foundation for the component score inference developed in the next section.

\subsection{General Stochastic Recursion}

We consider the stochastic optimization problem
\begin{equation}
\label{eq0}
\min_{\bx \in \mathbb{R}^d} f(\bx) := \mathbb{E}_{\zeta \sim \mathcal{P}}\left[\mathcal{L}(\bx;\zeta)\right],
\end{equation}
where $\zeta$ denotes a random observation or mini-batch drawn from a sampling distribution $\mathcal{P}$.
The loss function $\mathcal{L}(\bx;\zeta)$ is evaluated at the given sample $\zeta$, and $f:\mathbb{R}^{d}\to\mathbb{R}$ is the population objective, assumed to be continuously differentiable. The formulation in \eqref{eq0} covers a broad class of problems arising in statistical learning and machine learning. In deep learning, $\bx$ denotes the collection of trainable parameters of a neural network,
$\mathcal{L}\left(\bx;\zeta\right)$ is the loss function evaluated at mini-batches of samples,
and $f(\bx)$ is the corresponding population risk or expected training loss~\cite{goodfellow2016deep}. 
In statistical learning, $\bx$ may represent the parameters of a statistical model, $\mathcal{L}\left(\bx;\zeta\right)$ may correspond to a negative log-likelihood or a regularized risk evaluated at a data sample, and $f(\bx)$  is the population objective associated with an $M$-estimator or expected risk minimization~\cite{vapnik1998statistical}.

We use $\bz_t$ to denote the algorithmic state at iteration $t$ for a stochastic optimization algorithm applied to \eqref{eq0}. The parameter vector $\bx_t$ can be recovered from $\bz_t$, while $\bz_t$ may also include auxiliary variables needed to describe the algorithmic dynamics, such as momentum or adaptive second-moment estimates. We consider the following general stochastic approximation recursion:
\begin{equation}
\label{eq_general_recursion}
\bz_{t+1}
=
\bz_t
+
\gamma_t \left( F(\bz_t) + \bxi_{t} + \br_{t} \right),
\end{equation}
where $t \in \mathbb{N}$, $F$ represents the \emph{mean field} (or drift) of the algorithm, $\bxi_{t}$ denotes the stochastic noise in the update direction, and $\br_t$ represents a remainder term, accounting for approximation errors or additional algorithmic effects not captured by the leading mean field $F$. 
The step-size sequence is chosen as
\begin{equation}
\label{eq_step_size}
\gamma_t = \gamma_0 t^{-\kappa}, \quad \kappa \in (1/2,1),
\end{equation}
where $\gamma_0 > 0$. This corresponds to a \emph{polynomially decaying} (or power-law) step size schedule.

To characterize the uncertainty and asymptotic behavior of the stochastic recursion \eqref{eq_general_recursion} under the step-size schedule \eqref{eq_step_size}, we impose the following regularity assumptions on the mean dynamics, stochastic noise, and remainder terms.

\begin{assumption}
\label{assumption1}
The following conditions hold.
\begin{enumerate}[label=(A\arabic*), leftmargin=2.2em]
    \item \textbf{Almost sure convergence to a limit.}
    There exists a point $\bz^\star$ such that
    \[
        F(\bz^\star)=\bzero,
    \]
    and $\bz_t\to\bz^\star$ almost surely as $t\to\infty$.

    \item \textbf{Local differentiability of the mean field.}
    The mean field $F:\mathbb{R}^d\to\mathbb{R}^d$ is continuously differentiable and $\nabla F$ is locally Lipschitz in a neighborhood of $\bz^\star$. Let
    \[
        \bJ:=\nabla F(\bz^\star).
    \]
    Then $F(\bz^\star+\be)=F(\bz^\star)+\bJ\be+\rho(\be)$, where
    \[
        \|\rho(\be)\|=\mathcal{O}(\|\be\|^2),\quad \text{as }~\be\to\bzero.
    \]

    \item \textbf{Martingale-difference noise with bounded $(2+\delta)$-moment.}
    With respect to the pre-update filtration $\{\mathcal F_t\}$, the noise sequence $\{\bxi_t\}$ satisfies
    \[
        \mathbb E[\bxi_t\mid \mathcal F_t]=\bzero.
    \]
    Moreover, for some $\delta>0$,
    \[
        \sup_{t\ge 0}
        \mathbb E\left[\|\bxi_t\|^{2+\delta}\mid\mathcal F_t\right]
        <\infty,\qquad \text{a.s.}
    \]

    \item \textbf{Asymptotic conditional covariance.}
    There exists a positive semidefinite matrix $\bQ$ such that
    \[
        \mathbb E[\bxi_t\bxi_t^\top\mid\mathcal F_t]\to\bQ,
        \qquad \text{a.s.}
    \]

    \item \textbf{Stability of the linearized dynamics.}
    The Jacobian $\bJ$ is Hurwitz; that is, all eigenvalues of $\bJ$ have strictly negative real parts.

    \item \textbf{Negligible remainder.}
    The remainder term $\br_t$ in \eqref{eq_general_recursion} satisfies
    \[
        \mathbb E[\|\br_t\|^2\mid\mathcal F_t]
        =
        \mathcal O(\gamma_t^2),
        \qquad \text{a.s.}
    \]
\end{enumerate}
\end{assumption}

\noindent\textbf{Remarks.}
Assumption~\ref{assumption1} collects standard regularity conditions for deriving asymptotic normality of stochastic approximation recursions. Their roles are summarized as follows:
\begin{enumerate}[
    label=(A\arabic*),
    leftmargin=3.2em,
    labelwidth=2.2em,
    labelsep=0.6em,
    itemsep=0pt,
    parsep=0pt,
    topsep=0pt
]
    \item The convergence condition localizes the analysis around a limiting point $\bz^\star$, so that the local distributional behavior of the iterates can be studied after convergence.

    \item The differentiability condition justifies a local linear approximation of the mean field $F$ near $\bz^\star$. The Jacobian $\bJ=\nabla F(\bz^\star)$ determines the leading deterministic dynamics.

    \item The martingale-difference and moment conditions control the stochastic noise. They are standard requirements for martingale central limit arguments and tail control.

    \item The covariance condition ensures that the conditional covariance of the noise has a well-defined asymptotic limit, which determines the noise level in the limiting distribution.

    \item The stability condition requires the linearized mean dynamics to be locally attracting, ensuring that the deterministic drift pulls the iterates toward the limiting point.

    \item The remainder condition ensures that remainder terms are asymptotically negligible and do not affect the leading mean field and the limiting distribution.
\end{enumerate}
Overall, these assumptions are standard in central limit theory for stochastic approximation and stochastic optimization~\cite{chen2020statistical,chen2024online,duchi2021asymptotic,davis2024asymptotic,na2025statistical}.

\begin{theorem}
[Central limit theorem for the general stochastic recursion]
\label{theorem1}
    Under Assumptions (A1)--(A6) in \Cref{assumption1}, the stochastic approximation sequence $\{\bz_t\}$ generated by \eqref{eq_general_recursion} with step sizes $\gamma_t = \gamma_0 t^{-\kappa}$, $\gamma_0>0$, and $\kappa \in (1/2,1)$, is asymptotically normal around the limiting point $\bz^{\star}$.
More precisely, we have
\begin{equation}
\gamma_t^{-1/2}(\bz_t-\bz^{\star})
\;\xrightarrow{d}\;
\mathcal{N}(\bzero,\bSig),
\end{equation}
where the covariance matrix $\bSig$ is characterized as 
the unique positive semidefinite solution to the Lyapunov equation
\begin{equation}
\bJ\bSig + \bSig \bJ^\top + \bQ = \bzero.
\end{equation}
Moreover, the Polyak--Ruppert averaged iterate $\Bar{\bz}_t = \frac{1}{t}\sum_{\tau=1}^{t} \bz_{\tau}$ satisfies
\begin{equation}
\sqrt{t}(\Bar{\bz}_t-\bz^{\star})
\;\xrightarrow{d}\;
\mathcal{N}(\bzero,\bJ^{-1}\bQ\bJ^{-\top}).
\end{equation}
\end{theorem}

\noindent\textit{Proof sketch.}
The detailed proof is provided in Appendix~\ref{appendix_proof1}. We briefly summarize the main argument. Let $\be_t=\bz_t-\bz^\star$. By the local expansion of $F$ around $\bz^\star$, the recursion can be written as
\[
\be_{t+1}
=
(\bI+\gamma_t\bJ)\be_t
+\gamma_t\bxi_t
+\underbrace{\gamma_t \mathcal{O}(\|\be_t\|^2)}_{\text{Taylor remainder}}
+\underbrace{\gamma_t\br_t}_{\text{original remainder}}.
\]
Under Assumptions (A1)--(A6), the Taylor remainder term and the original remainder term $\br_t$ are asymptotically negligible.
Hence, the leading term is governed by a stable linear stochastic recursion driven by martingale-difference noise. Applying a martingale central limit theorem to the resulting weighted martingale array gives the asymptotic normality of $\gamma_t^{-1/2}(\bz_t-\bz^\star)$, with covariance characterized by the Lyapunov equation (see Lemma \ref{lemma4}). For the Polyak--Ruppert average, summing the linearized recursion results in the asymptotic representation
\[
    \sqrt{t}(\bar{\bz}_t-\bz^\star)
    =
    -\bJ^{-1}\frac{1}{\sqrt{t}}\sum_{\tau=1}^{t}\bxi_\tau
    +o_p(1).
\]
The martingale central limit theorem, together with Slutsky's theorem, then implies the stated averaged central limit theorem.

\Cref{theorem1} provides the limiting distributions for the stochastic recursion \eqref{eq_general_recursion}. 
Such results are classical in stochastic approximation theory and are closely related to Polyak--Ruppert averaging and asymptotic normality for stochastic recursive algorithms~\cite{polyak1992acceleration,kushner2003stochastic,borkar2008stochastic}.
The theorem shows that, under standard local regularity and stability conditions, both the last iterate and the Polyak--Ruppert averaged iterate admit Gaussian limits after appropriate normalization. These results provide the theoretical foundation for studying the local asymptotic behavior of stochastic optimization algorithms. In the remainder of this section, we show how several adaptive optimizers used in LoRA fine-tuning can be represented in the form of \eqref{eq_general_recursion}, and derive their corresponding asymptotic distributions.

\subsection{Central Limit Theory for AdamW}

We next specialize the general stochastic approximation framework to AdamW, a widely used adaptive optimizer in deep learning and LoRA fine-tuning~\cite{loshchilov2018decoupled}. Consider applying AdamW to the stochastic optimization problem in \eqref{eq0}, where the population gradient is observed through stochastic gradient estimates. The updates are given by
\begin{equation}
\label{eq_adamw}
\begin{split}
\bmm_{t+1} &= (1-\alpha_t)\bmm_t + \alpha_t \bg_t,\\
\bv_{t+1} &= (1-\beta_t)\bv_t + \beta_t \bg_t^{\odot 2},\\
\bx_{t+1} &= \bx_t - \gamma_t\frac{\bmm_{t+1}}{\sqrt{\bv_{t+1}}+\varepsilon} - \gamma_t\lambda \bx_t,
\end{split}
\end{equation}
where $\bg_t:=\nabla \mathcal{L}(\bx_t;\zeta_t)$ denotes the stochastic gradient evaluated at $\bx_t$ with data sample $\zeta_t$, $\alpha_t=\alpha\gamma_t$ and $\beta_t=\beta\gamma_t$ are step sizes for momentum terms with $\alpha, \beta>0$, $\lambda>0$ is the weight decay parameter, $\gamma_t$ is the step size as defined in \eqref{eq_step_size}, and $\varepsilon>0$ is a numerical stabilization constant. Here, $\bmm_t$ and $\bv_t$ denote the first- and second-moment variables, respectively. All vector operations in \eqref{eq_adamw}, including division and square root, are understood element-wise.

We define the augmented state variable
\begin{equation}
\label{eq_augmented_z}
\bz_t := (\bx_t,\bmm_t,\bv_t)\in\mathbb R^{3d}.
\end{equation}
Let $\nabla f(\bx_t)$ denote the population gradient at $\bx_t$. The stochastic gradient is written as
\[
\bg_t = \nabla \mathcal{L}(\bx_{t}; \zeta_t) = \nabla f(\bx_t)+\bxi_t^{(g)},
\]
where $\bxi_t^{(g)} := \nabla \mathcal{L}(\bx_{t}; \zeta_t) - \nabla f(\bx_t)$ is the stochastic gradient noise. 
We assume that the stochastic gradient is conditionally unbiased with respect to the pre-update filtration $\mathcal F_t$, that is,
\begin{equation*}
\mathbb{E}\left[\bxi_t^{(g)}\mid\mathcal{F}_t\right]=\bzero,
    \qquad
    \mathbb{E}\left[\nabla \mathcal{L}(\bx_t;\zeta_t)\mid\mathcal{F}_t\right]
    =
    \nabla f(\bx_t).
\end{equation*}
This corresponds to the standard mini-batch setting in stochastic optimization, where the mini-batch is sampled conditionally independently given the current iterate, such as neural network training and fine-tuning~\cite{goodfellow2016deep}.
We define the second-moment function of the stochastic gradient by
\begin{equation}
\label{eq_q}
    \bq(\bx) := \mathbb{E}\left[\nabla \mathcal{L}(\bx;\zeta)^{\odot 2}\right].
\end{equation}
Then the first- and second-moment updates in \eqref{eq_adamw} can be written as
\begin{align*}
& \bmm_{t+1}
=
\bmm_t+\alpha_t\bigl(\nabla f(\bx_t)-\bmm_t\bigr)+\alpha_t\bxi_t^{(g)},\\
& \bv_{t+1}
= \bv_t + \beta_t\big(\bq(\bx_t) - \bv_t\big)
+ \beta_t\left(2\nabla f(\bx_t)\odot \bxi_{t}^{(g)} + \bxi_{t}^{(g)\odot 2} - \mathbb{E}\left[\bxi_t^{(g)\odot 2}\mid \mathcal{F}_{t}\right]\right).
\end{align*}
The parameter update can be rewritten as
\begin{equation*}
\bx_{t+1}
=
\bx_t
+
\gamma_t\left(
-\frac{\bmm_t}{\sqrt{\bv_t}+\varepsilon}
-\lambda \bx_t
\right)
+
\gamma_t \br_t^{(x)},
\end{equation*}
where
\begin{equation*}
\br_t^{(x)}
=
\frac{\bmm_t}{\sqrt{\bv_t}+\varepsilon}
-
\frac{\bmm_{t+1}}{\sqrt{\bv_{t+1}}+\varepsilon}
\end{equation*}
collects the correction caused by replacing $(\bmm_{t+1},\bv_{t+1})$ with $(\bmm_t,\bv_t)$ in the leading drift.

Therefore, AdamW can be represented in the form of the general stochastic recursion \eqref{eq_general_recursion}. Specifically, for $\bz=(\bx,\bmm,\bv)$, the mean-field mapping $F:\mathbb{R}^{3d}\to\mathbb{R}^{3d}$ is
\begin{equation}
\label{eq_mean_field_adamw}
F(\bz)
=
\begin{pmatrix}
-\dfrac{\bmm}{\sqrt{\bv}+\varepsilon}-\lambda \bx\\[2mm]
\alpha\bigl(\nabla f(\bx)-\bmm\bigr)\\[1mm]
\beta\bigl(\bq(\bx)-\bv\bigr)
\end{pmatrix}.
\end{equation}
The martingale noise is given by
\begin{equation}
\label{eq_noise_adamw}
\bxi_t=
\begin{pmatrix}
\bzero\\
\alpha\bxi_t^{(g)}\\
\beta\Bigl(
2\nabla f(\bx_t)\odot \bxi_t^{(g)}
+
\hat{\bxi}_t^{(g)}
\Bigr)
\end{pmatrix},
\end{equation}
where
\begin{equation}
\label{eq_second_noise}
    \hat{\bxi}_t^{(g)}
=
\bxi_t^{(g)\odot 2}
-
\mathbb E\!\left[\bxi_t^{(g)\odot 2}\mid \mathcal F_t\right],
\end{equation}
and the remainder term is
\begin{equation}
\label{eq_remainder_adamw}
\br_t=
\begin{pmatrix}
\br_t^{(x)}\\
\bzero\\
\bzero
\end{pmatrix}.
\end{equation}
Consequently, the AdamW iteration admits the stochastic approximation representation \eqref{eq_general_recursion}, with mean field, martingale noise, and remainder term given by \eqref{eq_mean_field_adamw}, \eqref{eq_noise_adamw}, and \eqref{eq_remainder_adamw}, respectively.

The following proposition verifies the convergence condition and local stability requirement in Assumption~\ref{assumption1} for the AdamW mean-field dynamics. Due to the decoupled weight decay term in \eqref{eq_adamw}, the limiting point of AdamW does not necessarily coincide with a stationary point of the population objective $f$, which differs from optimizers without decoupled weight decay. We further characterize the Jacobian of the AdamW mean field at the limiting point and provide a sufficient condition under which the local stability condition in Assumption (A5) holds.
The proof is available in Appendix~\ref{appendix_prop2}.

\begin{proposition}
\label{proposition2}
Consider the AdamW algorithm and its mean-field mapping in \eqref{eq_mean_field_adamw}. The following statements hold.

\begin{enumerate}
    \item Let $\bz^\star=(\bx^\star,\bmm^\star,\bv^\star)$ satisfies the AdamW mean-field equation $F(\bz^\star)=\bzero$. Then,
    \begin{equation}
    \label{eq:adamw_eq_conditions}
    \bmm^\star = \nabla f(\bx^\star),
    \qquad
    \bv^\star = \bq(\bx^{\star}),
    \end{equation}
    and $\bx^\star$ satisfies
    \begin{equation}
    \label{eq:adamw_eq_x}
    \frac{\nabla f(\bx^\star)}{\sqrt{\bv^\star}+\varepsilon}
    +
    \lambda \bx^\star
    =
    \bzero.
    \end{equation}

    \item Suppose that the $i$-th entry of the vector $\bv^{\star}$ is nonzero for all $i=1,\dots,d$, i.e., $v_i^\star>0$, and define
    \[
    \bD^\star
    :=
    \operatorname{diag}\!\left(
    \frac{1}{\sqrt{\bv^{\star}}+\varepsilon
    }\right),
    \]
    \[
    \bC^\star
    :=
    \operatorname{diag}\!\left(
    \frac{\bmm^\star}{
    2\sqrt{\bv^\star}(\sqrt{\bv^\star}+\varepsilon)^{\odot 2}
    }\right),
    \]
    where all vector operations are understood entry-wise.
    Then the Jacobian matrix of the AdamW mean field at $\bz^\star$ is
    \begin{equation}
    \label{eq_J_adamw}
    \bJ=
    \begin{pmatrix}
    -\lambda \bI_d & -\bD^\star & \bC^\star\\[2mm]
    \alpha \nabla^2 f(\bx^\star) & -\alpha \bI_d & \bzero\\[2mm]
    \beta\nabla \bq(\bx^{\star})
    & \bzero & -\beta \bI_d
    \end{pmatrix}.
    \end{equation}
    Moreover, suppose that $\nabla^2 f(\bx^\star)\succ \bzero$, then there exists $\epsilon>0$ such that if $\|\nabla f(\bx^\star)\|<\epsilon$, the Jacobian matrix $\bJ$ in \eqref{eq_J_adamw} is Hurwitz.
\end{enumerate}
\end{proposition}

We now specialize Assumption~\ref{assumption1} to the AdamW recursion in \eqref{eq_adamw}. The following conditions provide optimizer-specific regularity assumptions under which AdamW fits into the general stochastic approximation framework.

\begin{assumption}
\label{assumption3}
For the AdamW algorithm \eqref{eq_adamw}, the following conditions hold.
\begin{enumerate}[
    label=(B\arabic*),
    leftmargin=3.2em,
    labelwidth=2.2em,
    labelsep=0.6em,
    itemsep=0.25em,
    topsep=0.25em
]
    \item \textbf{Almost sure convergence.}
    The augmented iterate $\bz_t=(\bx_t,\bmm_t,\bv_t)$ converges almost surely to a point
    \[
        \bz^\star=(\bx^\star,\bmm^\star,\bv^\star),
    \]
    satisfying \eqref{eq:adamw_eq_conditions} and \eqref{eq:adamw_eq_x}. Moreover, each entry of $\bv^\star=\bq(\bx^\star)$ is strictly positive.

    \item \textbf{Local smoothness.}
    The population objective $f$ is twice continuously differentiable, the Hessian $\nabla^2 f$ is Lipschitz continuous, the second-moment function
    \[
        \bq(\bx)
        =
        \mathbb{E}\left[
        \nabla \mathcal{L}(\bx;\zeta)^{\odot 2}
        \right]
    \]
    is continuously differentiable, and $\nabla \bq$ is Lipschitz continuous in a neighborhood of $\bx^\star$.

    \item \textbf{Gradient noise with bounded moments.}
    The stochastic gradient noise $\{\bxi_t^{(g)}\}$ is a martingale-difference sequence with respect to the pre-update filtration $\{\mathcal{F}_t\}$, namely
    \[
        \mathbb{E}\left[\bxi_t^{(g)}\mid\mathcal{F}_t\right]
        =
        \bzero .
    \]
    Moreover, there exists $\delta>0$ such that
    \[
        \sup_{t\geq 1}
        \mathbb{E}\left[
        \|\bxi_t^{(g)}\|^{4+\delta}
        \mid
        \mathcal{F}_t
        \right]
        <
        \infty,
        \qquad \text{a.s.}
    \]

    \item \textbf{Asymptotic conditional covariance.}
    There exist positive semidefinite matrices $\bQ^{(g)}$ and $\widehat{\bQ}^{(g)}$, and a matrix $\bR^{(g)}$, such that
    \[
        \mathbb{E}\left[
        \bxi_t^{(g)}\bxi_t^{(g)\top}
        \mid
        \mathcal{F}_t
        \right]
        \to
        \bQ^{(g)},
        \qquad \text{a.s.},
    \]
    \[
        \mathbb{E}\left[
        \widehat{\bxi}_t^{(g)}
        \widehat{\bxi}_t^{(g)\top}
        \mid
        \mathcal{F}_t
        \right]
        \to
        \widehat{\bQ}^{(g)},
        \qquad \text{a.s.},
    \]
    and
    \[
        \mathbb{E}\left[
        \bxi_t^{(g)}
        \widehat{\bxi}_t^{(g)\top}
        \mid
        \mathcal{F}_t
        \right]
        \to
        \bR^{(g)},
        \qquad \text{a.s.}
    \]
    Here, $\widehat{\bxi}_t^{(g)}$ is defined in \eqref{eq_second_noise}.

    \item \textbf{Stability of the linearized dynamics.}
    The Jacobian matrix $\bJ$ defined in \eqref{eq_J_adamw} is Hurwitz; that is, all eigenvalues of $\bJ$ have strictly negative real parts.
\end{enumerate}
\end{assumption}

\noindent\textbf{Remarks.}
The above conditions are sufficient to verify Assumption~\ref{assumption1} for the AdamW recursion defined by \eqref{eq_mean_field_adamw}, \eqref{eq_noise_adamw}, and \eqref{eq_remainder_adamw}. Their roles are analogous to those discussed after Assumption~\ref{assumption1}, and we therefore do not repeat the full discussion here. We only highlight two AdamW-specific points. First, the bounded $(4+\delta)$-moment condition on the stochastic gradient noise is required because the second-moment update in \eqref{eq_adamw} involves squared stochastic gradients. Second, the conditional covariance assumptions in (B4) require the noise structure, including the second-order gradient noise, to stabilize asymptotically. These conditions allow the local behavior of the AdamW iterates to be characterized through the central limit theory for the augmented stochastic recursion.

\begin{theorem}
[Central limit theorem for AdamW]
\label{theorem3}
Suppose that \Cref{assumption3} holds. Then Assumptions~(A1)--(A6) in \Cref{assumption1} are satisfied for the AdamW iterates~\eqref{eq_augmented_z} generated by \eqref{eq_adamw}. Consequently,
\begin{equation*}
\gamma_t^{-1/2}(\bz_t-\bz^\star)
\;\xrightarrow{d}\;
\mathcal N(\bzero,\bSig),
\end{equation*}
and
\begin{equation*}
    \sqrt{t}\left(\Bar{\bz}_t - \bz^{\star}\right)\;\xrightarrow{d}\;
\mathcal N(\bzero,\bJ^{-1}\bQ\bJ^{-\top}),
\end{equation*}
where $\bSig$ is the unique positive semidefinite solution to the Lyapunov equation
\begin{equation*}
\bJ\bSig+\bSig\bJ^\top+\bQ=\bzero.
\end{equation*}
Here, $\bJ$ is given by \eqref{eq_J_adamw}, and the asymptotic covariance matrix $\bQ$ takes the block form
\begin{equation*}
\bQ=
\begin{pmatrix}
\bzero & \bzero & \bzero\\
\bzero & \bQ_{22} & \bQ_{23}\\
\bzero & \bQ_{32} & \bQ_{33}
\end{pmatrix},
\end{equation*}
with 
\begin{align*}
\bQ_{22}
&= \alpha^2 \bQ^{(g)},\\[1mm]
\bQ_{23}
&= \alpha\beta\bigl(2\bQ^{(g)}\operatorname{diag}\!\bigl(\bmm^{\star}\bigr)+\bR^{(g)}\bigr),\\[1mm]
\bQ_{32}
&= \alpha\beta\bigl(2\operatorname{diag}\!\bigl(\bmm^{\star}\bigr)\bQ^{(g)}+\bR^{(g)\top}\bigr),\\[1mm]
\bQ_{33}
&=\beta^2\Bigl(
4\operatorname{diag}\!\bigl(\bmm^{\star}\bigr)\bQ^{(g)}\operatorname{diag}\!\bigl(\bmm^{\star}\bigr)
+2\operatorname{diag}\!\bigl(\bmm^{\star}\bigr)\bR^{(g)}
+2\bR^{(g)\top}\operatorname{diag}\!\bigl(\bmm^{\star}\bigr)
+\hat{\bQ}^{(g)}
\Bigr).
\end{align*}
\end{theorem}

\noindent\textit{Proof sketch.}
The detailed proof is deferred to Appendix~\ref{appendix_proof2}. 
The proof proceeds by verifying that the AdamW recursion can be written in the general stochastic approximation form \eqref{eq_general_recursion}. 
Under Assumption~\ref{assumption3}, the augmented state $\bz_t=(\bx_t,\bmm_t,\bv_t)$ converges to a limiting point, the AdamW mean field is locally differentiable, the induced noise sequence is a martingale difference with a limiting conditional covariance, and the remainder term is asymptotically negligible. Hence, Assumptions (A1)--(A6) in Assumption~\ref{assumption1} hold for the AdamW recursion. 
The last-iterate and Polyak--Ruppert averaged central limit theorems then follow directly from Theorem~\ref{theorem1}. 
The block covariance matrix $\bQ$ is obtained by computing the conditional covariance of the martingale noise vector in \eqref{eq_noise_adamw}.

Theorem~\ref{theorem3} gives an explicit asymptotic distribution for the AdamW iterates and their Polyak--Ruppert averages. 
Compared with standard stochastic gradient methods, the analysis is more involved because AdamW contains first-moment averaging, second-moment adaptation, coordinate-wise scaling with momentum terms, and decoupled weight decay. 
These algorithmic components require an augmented-state formulation and lead to a nontrivial covariance structure.
To the best of our knowledge, few existing results provide such an explicit central limit theorem for AdamW with both last-iterate and averaged-iterate limits. 
This result forms the main theoretical basis for the component score inference and hypothesis testing procedure developed in the next section.

\subsection{Extensions to Adam and Adafactor}

Adam~\cite{kingma2015adam} and Adafactor~\cite{shazeer2018adafactor} are also commonly used optimizers in LoRA fine-tuning and large-scale model training. Both algorithms can be analyzed within the same stochastic approximation framework developed above, after introducing suitable augmented state variables and optimizer-specific mean-field mappings. Compared with AdamW, the main differences lie in the form of weight decay and the construction of the adaptive preconditioner. These differences lead to different limiting points, Jacobian matrices, and limiting covariance structures, but the underlying proof strategy remains similar: one verifies the corresponding stochastic approximation assumptions and then applies the general central limit theorem in Theorem~\ref{theorem1}.

To keep the main content focused, we defer the detailed formulations and proofs for Adam and Adafactor to Appendices~\ref{appendix_adam} and \ref{appendix_adafactor}. For completeness, we state the following general extension result.

\begin{corollary}[Extensions to Adam and Adafactor]
Adam and Adafactor can be represented as augmented stochastic approximation recursions of the form \eqref{eq_general_recursion}. Under optimizer-specific analogues of Assumption~\ref{assumption1}, their last iterates and Polyak--Ruppert averaged iterates satisfy central limit theorems of the same form as Theorem~\ref{theorem1}.
\end{corollary}

The precise augmented states, mean-field mappings, Jacobian matrices, limiting covariance matrices, and proofs are provided in Appendices~\ref{appendix_adam} and \ref{appendix_adafactor}. 
The resulting covariance formulas differ across optimizers because they depend on the local linearization of the corresponding mean field and on the conditional covariance of the induced martingale noise.

The unified formulation also suggests that the same analysis can be extended to some other optimizers in deep learning, provided that they admit an augmented-state representation satisfying the regularity conditions in Assumption~\ref{assumption1}. Recent work has studied central limit theory for Adam-type methods using optimizer-specific arguments~\cite{dereich2026central,an2026polyak,barakat2021convergence}. In contrast, our approach emphasizes a common stochastic approximation structure, which allows AdamW, Adam, and Adafactor to be treated within a unified framework.

\section{Hypothesis Testing for LoRA Rank Allocation}

In this section, we develop the statistical methodology and algorithm for LoRA rank allocation. The central objects of inference are the empirical component scores generated along stochastic optimizer trajectories. Building on the central limit theory for optimizer iterates established in the previous section, we first transfer the asymptotic analysis from iterates to LoRA component scores and derive the limiting distribution of averaged empirical scores. We then introduce a practical variance estimation procedure to approximate the corresponding long-run variance from the observed score trajectory. These ingredients allow us to construct one-sided hypothesis tests for evaluating the significance of individual rank-one components. Finally, we incorporate the resulting p-values into an adaptive pruning procedure, leading to the proposed StatLoRA algorithm for statistically guided LoRA rank allocation.

\subsection{Central Limit Theory for Empirical Scores}

The population component scores $s_{\ell,j}^{\star}$, defined in \eqref{eq_score_population}, are unknown in practice. We therefore estimate them from the training trajectory generated by stochastic optimization. For the $j$-th rank-one component in the $\ell$-th LoRA module, define the empirical score at iteration $t$ by
\begin{equation}
\label{eq_empirical_score}
    s_{\ell,j}(t) = \left\| \bb_{\ell,j}(t)\right\|^2 \cdot \left\| \ba_{\ell,j}(t)\right\|^2,
\end{equation}
where $\ba_{\ell,j}(t)$ and $\bb_{\ell,j}(t)$ denote the corresponding LoRA factors at iteration $t$, after running the stochastic optimization algorithms (e.g., AdamW) for fine-tuning. 
In the LoRA fine-tuning setting, the state variable $\bz$ contains all trainable LoRA parameters $\{\bA_{\ell},\bB_{\ell}\}_{\ell=1}^{L}$, possibly together with optimizer-specific auxiliary variables. 
Let $\bP_{\ell,j}$ denote the selection matrices that extract $\bw_{\ell,j} := (\ba_{\ell,j},\bb_{\ell,j})$ from the full state vector $\bz$, respectively. That is,
\[
    \bw_{\ell,j}(t)=\bP_{\ell,j}\bz_t .
\]
Define the induced joint covariance matrices
\begin{equation}
\label{eq_cov_component}
    \bSig_{\ell,j}^{w}:=\bP_{\ell,j}\bSig\bP_{\ell,j}^{\top},
    \qquad
    \bar{\bSig}_{\ell,j}^{w}
    :=
    \bP_{\ell,j}\bJ^{-1}\bQ\bJ^{-\top}\bP_{\ell,j}^{\top}.
\end{equation}
where $\bSig$ is the limiting covariance matrix for last-iterate, while $\bJ$ and $\bQ$ are Jacobian and noise covariance matrices appearing in \Cref{theorem3}. 
By the central limit theory for iterates established in \Cref{theorem3}, these component-wise LoRA factors inherit asymptotic normality from the full optimizer state. In particular,
\begin{equation}
\label{eq_asymp_a}
    \gamma_{t}^{-1/2}\left(\bw_{\ell,j}(t) - \bw_{\ell,j}^{\star}\right) 
    \;\xrightarrow{d}\;
\mathcal{N}\left(\bzero,\bSig_{\ell,j}^{w}\right),
\end{equation}
and
\begin{equation}
\label{eq_asymp_avg_w}
    \sqrt{t}\left(\frac{1}{t}\sum_{\tau=1}^{t}\bw_{\ell,j}(\tau) - \bw_{\ell,j}^{\star}\right) 
    \;\xrightarrow{d}\;
\mathcal{N}\left(\bzero,\bar{\bSig}_{\ell,j}^{w}\right),
\end{equation}
where $\bw^{\star}_{\ell,j} := (\ba^{\star}_{\ell,j},\bb^{\star}_{\ell,j})$.
These relations show that the joint LoRA component vector inherits an asymptotic distribution from the full optimizer state.
We next transfer these results to the empirical component score $s_{\ell,j}(t)$.

Since the empirical score $s_{\ell,j}(t)$ is a smooth function of the LoRA factors $\ba_{\ell,j}(t)$ and $\bb_{\ell,j}(t)$, the central limit theory for optimizer iterates can be transferred to the score process through the delta method. The following theorem characterizes the asymptotic distribution of both the last-iterate score and the averaged empirical score.

The first-order delta method may degenerate at the zero-score boundary $s_{\ell,j}^{\star}=0$, where one of the two LoRA factors may vanish. Since our null hypothesis is formulated as $s_{\ell,j}^{\star}\geq \Delta$ with $\Delta>0$, the relevant testing regime is away from this degenerate boundary. In this regime, both $\ba_{\ell,j}^{\star}$ and $\bb_{\ell,j}^{\star}$ are nonzero, and the first-order Taylor expansion of the score function provides the leading asymptotic term.

\begin{theorem}
    [Central limit theorem for empirical scores]
\label{theorem_clt_score}
Suppose that the assumptions for LoRA fine-tuning with AdamW in Assumption~\ref{assumption3} hold. Fix a LoRA component $(\ell,j)$ and assume that
$s_{\ell,j}^{\star}=
    \|\ba_{\ell,j}^{\star}\|^{2}
    \|\bb_{\ell,j}^{\star}\|^{2}
    \geq \Delta > 0$. Then the empirical score defined in \eqref{eq_empirical_score} satisfies
    \begin{equation}
    \label{eq_score_clt_last}
        \gamma_{t}^{-1/2}\left(s_{\ell,j}(t) - s_{\ell,j}^{\star}\right) 
    \;\xrightarrow{d}\;
\mathcal{N}(0,\sigma_{\ell,j}^{2}),
    \end{equation}
and the averaged empirical score satisfies
    \begin{equation}
    \label{eq_score_clt_avg}
        \sqrt{t}\left(\frac{1}{t}\sum_{\tau=1}^{t}s_{\ell,j}(\tau) - s_{\ell,j}^{\star}\right) 
    \;\xrightarrow{d}\;
\mathcal{N}(0,\Bar{\sigma}_{\ell,j}^{2}),
    \end{equation}
for some $\sigma_{\ell,j}^{2}, \Bar{\sigma}_{\ell,j}^{2} \geq 0$.
\end{theorem}

To express the limiting variances compactly, define the score function
\[
    h(\bw_{\ell,j})
    :=\|\ba_{\ell,j}\|^{2}\|\bb_{\ell,j}\|^{2}.
\]
Then the variances in \Cref{theorem_clt_score} are given by
\[
    \sigma_{\ell,j}^{2}
    =
    \nabla h(\bw_{\ell,j}^{\star})^{\top}
    \bSig_{\ell,j}^{w}
    \nabla h(\bw_{\ell,j}^{\star}),
    \qquad
    \bar{\sigma}_{\ell,j}^{2}
    =
    \nabla h(\bw_{\ell,j}^{\star})^{\top}
    \bar{\bSig}_{\ell,j}^{w}
    \nabla h(\bw_{\ell,j}^{\star}),
\]
where
\[
    \nabla h(\bw_{\ell,j}^{\star})
    =
    \begin{pmatrix}
    2\|\bb_{\ell,j}^{\star}\|^{2}\ba_{\ell,j}^{\star}\\   2\|\ba_{\ell,j}^{\star}\|^{2}\bb_{\ell,j}^{\star}
    \end{pmatrix},
\]
and covariance matrices are defined in \eqref{eq_cov_component}.

\noindent\textit{Proof sketch.}
The detailed proof is deferred to Appendix~\ref{appendix_proof3}. The proof follows from the delta method. Since
$s_{\ell,j}(t)
    =
    h(\bw_{\ell,j}(t)),$
a first-order Taylor expansion around $\bw_{\ell,j}^{\star}$ yields
\[
    s_{\ell,j}(t)-s_{\ell,j}^{\star}
    =
    \nabla h(\bw_{\ell,j}^{\star})^{\top}
    \left(
        \bw_{\ell,j}(t)-\bw_{\ell,j}^{\star}
    \right)
    +
    o_p\left(
        \|\bw_{\ell,j}(t)-\bw_{\ell,j}^{\star}\|
    \right).
\]
The condition $s_{\ell,j}^{\star}\geq\Delta>0$ excludes the zero-score boundary at which the first-order expansion may degenerate. Combining the above expansion with the component-wise central limit theory inherited from the optimizer state gives \eqref{eq_score_clt_last}. 
The averaged-score result \eqref{eq_score_clt_avg} follows by averaging the Taylor expansion along the trajectory and showing that the averaged second-order remainder is $o_p(t^{-1/2})$.

\Cref{theorem_clt_score} extends the optimizer-level central limit theory to the empirical component scores used for LoRA rank allocation. Thus, the empirical score is treated not merely as a deterministic importance measure, but as an estimator with an asymptotic distribution. This provides the basis for uncertainty quantification and for the hypothesis testing procedure developed below.

\subsection{Hypothesis Testing for LoRA Components}

The asymptotic distributions in \Cref{theorem_clt_score} provide the basis for hypothesis testing for each LoRA component. Our goal is to use the empirical score trajectory to evaluate whether a population score $s_{\ell,j}^{\star}$ is sufficiently large, and hence whether the corresponding LoRA component should be retained or pruned. To implement such tests, however, one must estimate the asymptotic variance appearing in the central limit theorem for empirical scores.

A direct approach is to use the explicit variance formula in \Cref{theorem_clt_score}. This leads to a plug-in estimator~\cite{newey1994asymptotic,chen2020statistical,na2025statistical}, where the optimizer-level covariance matrices, Jacobians, and score derivatives are estimated separately and substituted into the theoretical expression. Although conceptually straightforward, this approach is often impractical in large-scale fine-tuning, since it requires estimating high-dimensional covariance matrices and local derivative quantities that may be unstable or expensive to compute. 

As an alternative, we use a batch-means estimator~\cite{zhu2021online,kuang2025Online} to estimate the long-run variance of the scalar score process directly from its observed trajectory. Batch-means estimators are standard tools for variance estimation in central limit theorems for dependent stochastic processes and are particularly natural for averaged estimators. This motivates us to base the testing procedure on the Polyak--Ruppert averaged empirical score rather than the last-iterate score. Specifically, we use the averaged-score central limit theorem
\begin{equation*}
    \sqrt{t}
    \left(
        \frac{1}{t}\sum_{\tau=1}^{t}s_{\ell,j}(\tau)
        -
        s_{\ell,j}^{\star}
    \right)
    \xrightarrow{d}
    \mathcal{N}
    \left(
        0,
        \bar{\sigma}_{\ell,j}^{2}
    \right),
\end{equation*}
and estimate the long-run variance $\bar{\sigma}_{\ell,j}^{2}$ from the score trajectory. The following proposition states the consistency of the batch-means estimator under standard regularity conditions.

\begin{proposition}[Consistency of the batch-means variance estimator]
\label{prop_batch_mean}
Let $s_t=s_{\ell,j}(t)$ denote the empirical score process for a fixed LoRA component $(\ell,j)$. Suppose that the averaged score process satisfies a functional central limit theorem with long-run variance $\bar{\sigma}_{\ell,j}^{2}$. Let $t=Mb$, where $M$ is the number of batches and $b$ is the batch size. Assume that
\[
    b\to\infty,
    \qquad
    M\to\infty,
    \qquad
    b/t\to 0 .
\]
Define the batch means and the overall average by
\[
    \bar{s}_{m,b}
    =
    \frac{1}{b}
    \sum_{\tau=(m-1)b+1}^{mb}
    s_\tau,
    \qquad
    m=1,\ldots,M,
\]
and
\[
    \bar{s}_{t}
    =
    \frac{1}{t}
    \sum_{\tau=1}^{t}
    s_\tau .
\]
The batch-means variance estimator is defined as
\[
    \widehat{\bar{\sigma}}_{\ell,j}^{2}
    =
    \frac{b}{M-1}
    \sum_{m=1}^{M}
    \left(
        \bar{s}_{m,b}
        -
        \bar{s}_{t}
    \right)^2 .
\]
Then, under the standard moment and weak-dependence conditions for batch-means estimation~\cite{glynn1991estimating,flegal2010batch},
\[
    \widehat{\bar{\sigma}}_{\ell,j}^{2}
    \xrightarrow{p}
    \bar{\sigma}_{\ell,j}^{2}.
\]
\end{proposition}

\Cref{prop_batch_mean} provides a practical way to estimate the unknown long-run variance in the averaged-score central limit theorem. Instead of estimating the full optimizer-state covariance matrix and the derivatives appearing in the delta-method formula, the batch-means estimator uses only the observed scalar score trajectory.
In our setting, temporal dependence arises because the empirical scores are generated along a single stochastic optimizer trajectory. Under the usual functional central limit theorem, moment, weak-dependence, and batch-size conditions for batch-means consistency, the estimator is consistent. We defer the detailed regularity conditions and proof to Appendix~\ref{appendix_batch_mean}.

The conditions $b\to\infty$, $M\to\infty$, and $b/t\to0$ are asymptotic requirements for consistency of the batch-means estimator. 
With the variance estimate $\widehat{\bar{\sigma}}_{\ell,j}^{2}$, we define the test statistic
\begin{equation*}
    T_{\ell,j}(t)
    =
    \frac{
    \sqrt{t}
    \left(
        \bar{s}_{\ell,j}(t)
        -
        \Delta
    \right)
    }{
    \widehat{\bar{\sigma}}_{\ell,j}
    },
    \qquad
    \bar{s}_{\ell,j}(t)
    =
    \frac{1}{t}
    \sum_{\tau=1}^{t}
    s_{\ell,j}(\tau).
\end{equation*}
Here, $\widehat{\bar{\sigma}}_{\ell,j}$ denotes the batch-means estimate of $\bar{\sigma}_{\ell,j}$. 
Under the null hypothesis
\begin{equation}
\label{eq_hypothesis_testing}
    \mathcal{H}_{\ell,j}^{0}:~s_{\ell,j}^{\star}\geq \Delta
    \qquad
    \text{versus}
    \qquad
    \mathcal{H}_{\ell,j}^{1}:~s_{\ell,j}^{\star}<\Delta,
\end{equation}
the corresponding one-sided p-value is
\begin{equation*}
    p_{\ell,j}(t)
    =
    \Phi\left(T_{\ell,j}(t)\right),
\end{equation*}
where $\Phi(\cdot)$ denotes the cumulative distribution function of the standard normal distribution. A small p-value provides evidence against the null hypothesis, indicating that the component has insufficient population score relative to the threshold $\Delta$. 
Such components are therefore viewed as having weak statistical evidence for retention and are prioritized for pruning.

\subsection{The StatLoRA Algorithm}

We now describe the proposed StatLoRA algorithm, a statistical inference-based algorithm for LoRA rank allocation. 
StatLoRA starts from an over-parameterized LoRA model with an initial rank $r$ for each selected module and gradually removes rank-one components that have weak statistical evidence for retention during fine-tuning. The method uses the empirical score trajectory of each component to estimate its signal strength and uncertainty, and then performs hypothesis testing based on the p-values. 

\Cref{prop_batch_mean} provides a batch-means estimator for the long-run variance of the averaged empirical score. The consistency result is asymptotic and requires both the batch size $b$ and the number of batches $M$ to increase. In large-scale model fine-tuning, however, it is not practical to let $b$ and $M$ grow indefinitely. 
We therefore use a fixed batch size $b$ in finite training runs, and the number of batches is determined by the current trajectory length, $M=t/b$ for notational simplicity.
This provides a computationally efficient finite-sample approximation to the long-run variance.

\begin{algorithm}[t!]
\caption{StatLoRA}
\label{alg:statlora}
\begin{algorithmic}[1]
\State \textbf{Input:} Dataset $\mathcal{D}$; total iterations $T$; budget schedule $\{b^{(t)}\}_{t=1}^{T}$; initial rank $r$; threshold quantile $q$; batch-means batch size $b$; optimizer hyperparameters.
\State Initialize LoRA parameters $\{\bA_\ell,\bB_\ell\}_{\ell=1}^{L}$ with rank $r$.
\For{$t=1,\ldots,T$}
    \State Sample a mini-batch from $\mathcal{D}$ and perform one standard LoRA optimizer update.
    \State Compute component scores $s_{\ell,j}(t)=\|\bb_{\ell,j}(t)\|^2\|\ba_{\ell,j}(t)\|^{2}$ for all active components.
    \State Estimate averaged scores $\bar{s}_{\ell,j}(t)$ and variances $\widehat{\bar{\sigma}}_{\ell,j}^2(t)$ using batch means with $M = t/b$.
    \State Set the threshold $\Delta^{(t)}$ in \eqref{eq_hypothesis_testing} as the $q$-quantile of $\{\bar{s}_{\ell,j}(t)\}$ over active components.
    \State Compute p-values
    \[
        p_{\ell,j}(t)
        =
        \Phi\left(
        \frac{
        \sqrt{t}\big(\bar{s}_{\ell,j}(t)-\Delta^{(t)}\big)
        }{
        \widehat{\bar{\sigma}}_{\ell,j}(t)
        }
        \right).
    \]
    \State Prune components with the smallest p-values according to the budget $b^{(t)}$.
\EndFor
\State \textbf{Output:} Fine-tuned LoRA parameters with the final adaptive rank allocation.
\end{algorithmic}
\end{algorithm}

At a pruning step $t$, let $\mathcal A_t$ denote the set of active LoRA components, namely, the components that have not been pruned before iteration $t$. For each $(\ell,j)\in\mathcal A_t$, the score trajectory is partitioned into $M$ consecutive batches, each of length $b$, so that $t=Mb$. We prune a component when the null hypothesis in \eqref{eq_hypothesis_testing} is rejected, providing sufficient evidence that its limiting score is below the pre-specified threshold $\Delta$. Specifically, define the batch means by

\[
    \bar{s}_{\ell,j}^{(m)}(t)
    =
    \frac{1}{b}
    \sum_{\tau=(m-1)b+1}^{mb}
    s_{\ell,j}(\tau),
    \qquad
    m=1,\ldots,M.
\]
Then,
\begin{equation}
\label{eq_mean}
    \bar{s}_{\ell,j}(t)
    =
    \frac{1}{t}
    \sum_{\tau=1}^{t}
    s_{\ell,j}(\tau) = \frac{1}{M}
    \sum_{m=1}^{M}
    \bar{s}_{\ell,j}^{(m)}(t).
\end{equation}
Therefore, the long-run variance is estimated by
\begin{equation}
\label{eq_var_est_batch_means}
    \widehat{\bar{\sigma}}_{\ell,j}^{2}(t)
    =
    \frac{b}{M-1}
    \sum_{m=1}^{M}
    \left(
        \bar{s}_{\ell,j}^{(m)}(t)
        -
        \bar{s}_{\ell,j}(t)
    \right)^2 .
\end{equation}


The choice of the threshold $\Delta$ in \eqref{eq_hypothesis_testing} is another important practical issue. Rather than fixing an absolute threshold, we choose it adaptively from the empirical score distribution. 
Therefore, we use the notation $\Delta^{(t)}$ to denote the threshold in hypothesis testing at different time steps.
Specifically, for a pre-specified quantile level $q\in(0,1)$, we set
\[
    \Delta^{(t)}
    =
    \operatorname{Quantile}_{q}
    \left(
        \{\bar{s}_{\ell,j}(t): (\ell,j)\in\mathcal{A}_t\}
    \right),
\]
where $\mathcal{A}_t$ denotes the set of active components at iteration $t$ and $\operatorname{Quantile}_q$ denotes the empirical $q$-quantile of the finite set of averaged component scores. For example, $q=0.3$ sets $\Delta^{(t)}$ to the lower 30\% empirical quantile of the averaged component scores among active components.

Given $\bar{s}_{\ell,j}(t)$, $\widehat{\bar{\sigma}}_{\ell,j}(t)$, and $\Delta^{(t)}$, the algorithm uses the statistic:
\[
T_{\ell,j}(t)
=
\frac{
\sqrt{t}\left(\bar s_{\ell,j}(t)-\Delta^{(t)}\right)
}{
\widehat{\bar\sigma}_{\ell,j}(t)
}.
\]
The corresponding p-value is
\[
p_{\ell,j}(t)=\Phi(T_{\ell,j}(t)).
\]
Since the null hypothesis is $ \mathcal{H}^{0}_{\ell,j}:s_{\ell,j}^{\star}\geq \Delta^{(t)} $, a small p-value provides evidence that the component has insufficient population-level signal. StatLoRA therefore prunes components with the smallest p-values according to the prescribed budget schedule. The full algorithm is summarized in \Cref{alg:statlora}.

StatLoRA introduces little additional computational overhead. The component scores are computed from vector norms of existing LoRA factors, and the averaged scores, batch-means variances, test statistics, and p-values are scalar quantities updated during training. Thus, the dominant computational cost remains the standard LoRA fine-tuning procedure. In our experiments, the additional computation is small relative to the cost of model training.

\section{Experiments}

In this section, we evaluate the proposed StatLoRA method for LoRA fine-tuning across a range of language tasks, datasets, and pretrained models. The experiments are designed to validate both downstream task performance and the statistical behavior of the proposed rank allocation method. We first compare StatLoRA with existing LoRA rank allocation baselines on benchmark tasks, including natural language understanding, natural language generation, and question answering. 
We then conduct sensitivity analyses to examine the effect of the hyperparameters in StatLoRA and study the stability of rank allocation strategies. 
Finally, we provide empirical validations for the derived asymptotic normality for empirical scores.

\subsection{Experimental Setups}

We evaluate the proposed StatLoRA method on LoRA fine-tuning across three categories of language tasks: natural language understanding, natural language generation, and question answering. These experiments are designed to assess whether the proposed statistical rank allocation method is effective across different tasks, datasets, pretrained language models, and rank budgets.
The task-specific datasets, evaluation metrics, and pretrained models are described in the corresponding subsections.

We compare StatLoRA with vanilla LoRA~\cite{hu2022lora}, AdaLoRA~\cite{zhang2023adaptive}, and IGU-LoRA~\cite{cui2026igulora}. Vanilla LoRA uses a fixed rank allocation, where the same rank is assigned to each selected module. AdaLoRA adaptively allocates ranks by constructing importance scores for rank-one components based on gradient-derived sensitivity estimates. 
IGU-LoRA further builds on AdaLoRA by modifying the importance-score construction through gradient uncertainty and layer-wise information. 
These methods represent fixed-rank and adaptive-rank allocation strategies for LoRA, and therefore provide a natural comparison for the proposed hypothesis-testing-based allocation rule.

For fair comparison, we follow the training schedule used in AdaLoRA~\cite{zhang2023adaptive}, consisting of an initial warm-up phase, a rank-reduction phase, and a final fine-tuning phase after pruning. Across methods, we keep the pretrained model, dataset, LoRA target modules, total rank budget, optimizer, training steps, and evaluation protocol fixed whenever applicable. 
We use AdamW~\cite{loshchilov2018decoupled} as the optimizer throughout the experiments. 
In StatLoRA, following the theoretical setup, the step size is chosen in the polynomially decaying form $\gamma_t=\gamma_0 t^{-\kappa}$. We set $\kappa=0.51$, which satisfies the condition $\kappa\in(1/2,1)$ used in the asymptotic theory, and choose $\gamma_0$ so that the resulting learning rate scale is comparable to that used in the AdaLoRA implementation. The momentum parameters $\alpha$ and $\beta$ are selected analogously so that the effective update magnitudes are comparable across methods.
With these choices, the main difference among these methods lies in the criterion used to prune or retain rank-one components during the rank-reduction phase.

The final rank budget is set to one half of the initial total LoRA rank across all selected modules.
For StatLoRA, the method-specific hyperparameters are the number of batches $M$, the batch size $b$ used in the batch-means variance estimator, and the threshold quantile $q$.
Since $t=Mb$, we mainly choose $b$ according to the length of the warm-up and pruning phases, typically in the range of $50$ to $150$. 
The threshold quantile is set to $q=0.3$ in the main experiments, and its sensitivity is examined later.

Each experiment follows the same three-stage training procedure. During the warm-up phase, no pruning is performed, and the learning rate gradually increases and stabilizes. During the rank-reduction phase, rank-one components are pruned according to the specific allocation rule. 
After the target budget is reached, the selected ranks are fixed, and the model continues fine-tuning in the final phase with a gradually decreasing learning rate. 
Unless otherwise specified, experiments are repeated over five random seeds, and we report the mean performance together with the empirical standard deviation. All experiments are conducted on a single NVIDIA A100 GPU.

\subsection{Natural Language Understanding}

We start with the natural language understanding (NLU) tasks. NLU verifies whether a model can understand and compare natural language inputs, including sentence classification, textual entailment, semantic similarity, paraphrase detection, and linguistic acceptability. 
These tasks provide a standard testbed for examining whether a rank allocation method can maintain or improve downstream performance under a fixed rank budget.

We conduct experiments on seven tasks from the GLUE benchmark~\cite{wang2018glue}: MNLI, SST-2, CoLA, STS-B, QNLI, RTE, and MRPC. MNLI and RTE are natural language inference tasks, where the model predicts the logical relationship between a premise and a hypothesis. SST-2 is a binary sentiment classification task. CoLA evaluates linguistic acceptability by asking whether a sentence is grammatically acceptable. STS-B measures the semantic similarity between two sentences. QNLI is a question-answering-derived entailment task. MRPC is a paraphrase identification task, where the model determines whether two sentences are semantically equivalent.

We use DeBERTaV3-base~\cite{he2023debertav} as the pretrained language model and apply LoRA-based fine-tuning to the same set of modules for all compared methods. The baselines include vanilla LoRA with a fixed rank allocation, AdaLoRA, and IGU-LoRA. Following the standard GLUE evaluation protocol~\cite{wang2018glue}, we report accuracy for MNLI, SST-2, QNLI, and RTE; Matthews correlation for CoLA; Pearson correlation for STS-B; and F1 score for MRPC. All methods are evaluated on the validation set under the same training schedule and final rank budget.

For all NLU experiments, each LoRA module is initialized with rank $r=2$. For rank allocation methods, the final rank budget is chosen so that the average retained rank is one per module. Thus, the total rank budget is one-half of the initial total LoRA rank across the model. 
Under this budget, some modules may be assigned rank zero, whereas others may retain one or more rank-one components. This setting tests whether the allocation method can identify where the limited adaptation capacity should be used. 
For StatLoRA, we use the batch-means variance estimator with batch size $b=120$, $M=t/b$ batches, and set the threshold quantile to $q=0.3$. All methods are trained with AdamW and follow the same warm-up, rank-reduction, and final fine-tuning schedule.

\begin{table}[t]
\centering
\caption{Results on NLU tasks from the GLUE benchmark using DeBERTaV3-base. 
Higher values indicate better performance ($\uparrow$).
Results are reported as the mean over five random seeds, with the empirical standard deviation shown in the subscript. The best result is shown in bold, and the second-best result is underlined.}
\label{tab:glue_results}
\resizebox{\textwidth}{!}{
\begin{tabular}{lccccccc}
\toprule
Method & MNLI & SST-2 & CoLA & STS-B & QNLI & RTE & MRPC \\
\midrule
LoRA
& $90.26_{\scriptscriptstyle \pm 0.28}$
& $95.32_{\scriptscriptstyle \pm 0.50}$
& $69.10_{\scriptscriptstyle \pm 0.28}$
& $\underline{91.82}_{\scriptscriptstyle \pm 0.12}$
& $94.17_{\scriptscriptstyle \pm 0.22}$
& $87.40_{\scriptscriptstyle \pm 0.85}$
& $90.10_{\scriptscriptstyle \pm 0.51}$ \\
AdaLoRA
& $\underline{90.52}_{\scriptscriptstyle \pm 0.17}$
& $\underline{95.83}_{\scriptscriptstyle \pm 0.15}$
& $\underline{70.57}_{\scriptscriptstyle \pm 0.56}$
& $\mathbf{91.94}_{\scriptscriptstyle \pm 0.16}$
& $\mathbf{94.49}_{\scriptscriptstyle \pm 0.13}$
& $87.81_{\scriptscriptstyle \pm 0.93}$
& $\underline{90.24}_{\scriptscriptstyle \pm 0.53}$ \\
IGU-LoRA
& $90.36_{\scriptscriptstyle \pm 0.08}$
& $95.50_{\scriptscriptstyle \pm 0.26}$
& $69.98_{\scriptscriptstyle \pm 0.99}$
& $91.78_{\scriptscriptstyle \pm 0.14}$
& $94.38_{\scriptscriptstyle \pm 0.10}$
& $\mathbf{88.22}_{\scriptscriptstyle \pm 0.65}$
& $89.37_{\scriptscriptstyle \pm 0.41}$ \\
StatLoRA
& $\mathbf{90.62}_{\scriptscriptstyle \pm 0.15}$
& $\mathbf{95.93}_{\scriptscriptstyle \pm 0.10}$
& $\mathbf{70.98}_{\scriptscriptstyle \pm 0.98}$
& $91.81_{\scriptscriptstyle \pm 0.12}$
& $\underline{94.46}_{\scriptscriptstyle \pm 0.11}$
& $\underline{88.09}_{\scriptscriptstyle \pm 0.26}$
& $\mathbf{90.54}_{\scriptscriptstyle \pm 0.22}$ \\
\bottomrule
\end{tabular}
}
\end{table}

As shown in Table~\ref{tab:glue_results}, StatLoRA achieves the best performance on four of the seven reported tasks and obtains the second-best performance on two additional tasks. Compared with the fixed-rank LoRA baseline, StatLoRA improves the mean performance on all tasks under the same total rank budget. This suggests that adaptive allocation of rank-one components can improve parameter utilization when the number of adaptation components is limited.
Compared with AdaLoRA and IGU-LoRA, StatLoRA remains competitive across the GLUE tasks, with favorable results on MNLI, SST-2, CoLA, and MRPC. The gains are not uniform across all tasks, which is expected because the GLUE tasks differ in sample size, task structure, and sensitivity to the available adaptation capacity.

Overall, the NLU results show that the proposed hypothesis-testing-based allocation rule provides a reliable alternative to heuristic importance-score criteria for LoRA rank allocation. By incorporating uncertainty estimates into component selection, StatLoRA can allocate the limited rank budget in a way that is at least competitive with existing adaptive-rank methods and often improves upon the fixed-rank baseline.

\subsection{Natural Language Generation}

We next evaluate StatLoRA on natural language generation (NLG) tasks. In contrast to NLU tasks, which are typically formulated as classification or matching problems over input texts, NLG tasks require the model to generate fluent and semantically faithful output sequences. We focus on abstractive summarization, where the goal is to produce a concise summary that captures the main information in a source document. This setting provides a useful testbed for examining whether an adaptive rank allocation method can preserve generation quality under a fine-tuning budget.

We conduct experiments on two widely used summarization datasets: XSum~\cite{narayan2018don} and CNN/DailyMail~\cite{hermann2015teaching,nallapati2016abstractive}. XSum is an extreme summarization dataset in which each input document is paired with a short, one-sentence summary. It therefore requires the model to generate highly compressed summaries and is sensitive to whether the model can identify salient information. CNN/DailyMail consists of news articles paired with multi-sentence summaries, and evaluates the model's ability to generate longer summaries while preserving the main semantic content.

We use BART-Large~\cite{lewis2020bart} as the pretrained language model and apply LoRA-based fine-tuning to the same set of selected modules for all compared methods. The baselines include vanilla LoRA, AdaLoRA, and IGU-LoRA. All methods use the same pretrained model, training schedule, optimizer, target modules, and final rank budget. As in the NLU experiments, each LoRA module is initialized with rank $r=2$, and the final rank budget is set to one-half of the initial total LoRA rank. For StatLoRA, we use the batch-means variance estimator with batch size $b=100$, $M=t/b$ batches, and set the threshold quantile to $q=0.3$. Rank-one components are pruned during the rank-reduction phase according to the method-specific allocation rule.

Following standard evaluation practice for summarization, we report ROUGE-1, ROUGE-2, and ROUGE-L scores~\cite{lin2004rouge}. ROUGE-1 measures unigram overlap between the generated and reference summaries, ROUGE-2 measures bigram overlap, and ROUGE-L is based on the longest common subsequence. These metrics provide complementary measures of lexical overlap between generated summaries and reference summaries.

\begin{table}[t!]
\centering
\caption{
Results on NLG tasks using BART-Large. We report ROUGE-1, ROUGE-2, and ROUGE-L scores, where $\uparrow$ indicates that higher values are better. Results are reported as the mean over five random seeds, with the empirical standard deviation shown in the subscript. The best result is shown in bold, and the second-best result is underlined.}
\label{tab:nlg_results}
\begin{tabular}{lcccccc}
\toprule
\multirow{2}{*}{Method} 
& \multicolumn{3}{c}{XSum $(\uparrow)$} 
& \multicolumn{3}{c}{CNN/DailyMail $(\uparrow)$} \\
\cmidrule(lr){2-4} \cmidrule(lr){5-7}
& ROUGE-1 & ROUGE-2 & ROUGE-L 
& ROUGE-1 & ROUGE-2 & ROUGE-L \\
\midrule
LoRA
& $\underline{34.79}_{\scriptscriptstyle \pm 1.46}$
& $\underline{15.82}_{\scriptscriptstyle \pm 0.66}$
& $\underline{28.16}_{\scriptscriptstyle \pm 1.18}$
& $36.95_{\scriptscriptstyle \pm 1.60}$
& $17.24_{\scriptscriptstyle \pm 0.68}$
& $25.13_{\scriptscriptstyle \pm 1.11}$ \\

AdaLoRA
& $34.13_{\scriptscriptstyle \pm 3.19}$
& $15.40_{\scriptscriptstyle \pm 1.31}$
& $27.58_{\scriptscriptstyle \pm 2.49}$
& $37.77_{\scriptscriptstyle \pm 0.55}$
& $\underline{17.61}_{\scriptscriptstyle \pm 0.25}$
& $25.72_{\scriptscriptstyle \pm 0.34}$ \\

IGU-LoRA
& $33.53_{\scriptscriptstyle \pm 1.80}$
& $14.28_{\scriptscriptstyle \pm 0.69}$
& $26.74_{\scriptscriptstyle \pm 1.41}$
& $\mathbf{38.27}_{\scriptscriptstyle \pm 0.32}$
& $17.21_{\scriptscriptstyle \pm 0.18}$
& $\underline{25.83}_{\scriptscriptstyle \pm 0.14}$ \\

StatLoRA
& $\mathbf{35.03}_{\scriptscriptstyle \pm 2.50}$
& $\mathbf{15.85}_{\scriptscriptstyle \pm 1.07}$
& $\mathbf{28.33}_{\scriptscriptstyle \pm 1.96}$
& $\underline{38.15}_{\scriptscriptstyle \pm 0.80}$
& $\mathbf{17.75}_{\scriptscriptstyle \pm 0.42}$
& $\mathbf{25.94}_{\scriptscriptstyle \pm 0.52}$ \\
\bottomrule
\end{tabular}
\end{table}

As shown in Table~\ref{tab:nlg_results}, StatLoRA achieves competitive performance on both summarization datasets under the same rank budget. On XSum, StatLoRA obtains the highest mean score across all three ROUGE metrics. The improvements over vanilla LoRA are modest, especially relative to the empirical standard deviations, but the results indicate that the proposed statistical allocation rule can retain useful LoRA components in the highly compressed summarization setting of XSum.

On CNN/DailyMail, StatLoRA obtains the highest mean ROUGE-2 and ROUGE-L scores and the second-highest mean ROUGE-1 score. Compared with vanilla LoRA, StatLoRA improves all three reported metrics under the same rank budget. IGU-LoRA gives a slightly higher ROUGE-1 score, whereas StatLoRA provides the best overall performance across ROUGE-2 and ROUGE-L. This suggests that StatLoRA remains competitive on longer-document summarization tasks, where the model must generate more detailed summaries while preserving semantic content.

Overall, the NLG results show that the proposed rank allocation rule is not limited to classification-style NLU tasks. Across both summarization datasets, StatLoRA matches or improves upon existing LoRA variants under the same rank budget.
This suggests that incorporating uncertainty estimates into component selection can provide a reliable alternative to heuristic importance-score rules for adaptive LoRA rank allocation.

\subsection{Question Answering}

We further evaluate StatLoRA on question answering tasks using a larger pretrained language model. Compared with the previous experiments based on medium-scale pretrained models, this setting is closer to practical large language model adaptation, where parameter-efficient fine-tuning is important due to the computational and memory costs of updating all model parameters. Question answering tasks require the model to select correct answers from natural language questions, and may involve reading comprehension, factual knowledge, scientific reasoning, or commonsense reasoning. These tasks, therefore, provide a complementary setting for evaluating whether the proposed statistically guided rank allocation method remains effective for larger models and reasoning-oriented tasks.

We conduct experiments on four question answering datasets: BoolQ~\cite{clark2019boolq}, ARC-Easy~\cite{clark2018think}, OpenBookQA~\cite{mihaylov2018can}, and CommonsenseQA~\cite{talmor2019commonsenseqa}. BoolQ is a binary question answering dataset in which the model answers yes/no questions based on a short passage. ARC-Easy consists of grade-school science questions with multiple-choice answers and evaluates basic scientific reasoning. OpenBookQA is a multiple-choice benchmark that requires combining elementary science facts with commonsense reasoning. CommonsenseQA evaluates commonsense reasoning through multiple-choice questions constructed from commonsense knowledge relations. Together, these datasets cover reading comprehension, scientific question answering, and commonsense reasoning.

Qwen2.5-7B~\cite{yang2024qwen2} is used as the pretrained language model, and we apply LoRA-based fine-tuning to the same set of selected modules for all compared methods. The baselines include vanilla LoRA, AdaLoRA, and IGU-LoRA. For these large-model experiments, each LoRA module is initialized with rank $r=16$, and the final rank budget is set to one-half of the initial total LoRA rank. For StatLoRA, we use the batch-means variance estimator with batch size $b=100$, and set the threshold quantile to $q=0.3$. All methods are trained with the same optimizer, learning-rate schedule, warm-up phase, rank-reduction phase, and final fine-tuning phase. Following standard evaluation practice for binary and multiple-choice question answering tasks, we report accuracy as the evaluation metric.

\begin{table}[t!]
\centering
\caption{
Results on question answering tasks using Qwen2.5-7B. We report accuracy as the mean over five random seeds, with the empirical standard deviation shown in the subscript. 
Here, $\uparrow$ indicates that higher values are better. 
The best result is shown in bold, and the second-best result is underlined.}
\label{tab:qa_results}
\begin{tabular}{lcccc}
\toprule
Method & BoolQ ($\uparrow$) & ARC-Easy ($\uparrow$) & OpenBookQA ($\uparrow$) & CommonsenseQA ($\uparrow$) \\
\midrule
LoRA
& $89.25_{\scriptscriptstyle \pm 0.12}$
& $92.39_{\scriptscriptstyle \pm 0.83}$
& $91.72_{\scriptscriptstyle \pm 0.83}$
& $86.37_{\scriptscriptstyle \pm 0.53}$ \\

AdaLoRA
& $\mathbf{89.42}_{\scriptscriptstyle \pm 0.22}$
& $\underline{92.95}_{\scriptscriptstyle \pm 0.23}$
& $\underline{92.20}_{\scriptscriptstyle \pm 0.47}$
& $85.90_{\scriptscriptstyle \pm 0.13}$ \\

IGU-LoRA
& $89.31_{\scriptscriptstyle \pm 0.19}$
& $92.74_{\scriptscriptstyle \pm 0.16}$
& $90.52_{\scriptscriptstyle \pm 0.54}$
& $\underline{86.49}_{\scriptscriptstyle \pm 0.51}$ \\

StatLoRA
& $\underline{89.40}_{\scriptscriptstyle \pm 0.20}$
& $\mathbf{93.02}_{\scriptscriptstyle \pm 0.15}$
& $\mathbf{92.40}_{\scriptscriptstyle \pm 0.37}$
& $\mathbf{86.60}_{\scriptscriptstyle \pm 0.55}$ \\
\bottomrule
\end{tabular}
\end{table}

As shown in Table~\ref{tab:qa_results}, StatLoRA achieves the highest mean accuracy on three of the four question answering tasks and the second-highest mean accuracy on BoolQ. On ARC-Easy and OpenBookQA, StatLoRA improves over both the fixed-rank LoRA and the adaptive-rank methods. On CommonsenseQA, StatLoRA also attains the highest mean accuracy, suggesting that the proposed allocation rule remains competitive on commonsense reasoning tasks. On BoolQ, AdaLoRA gives the highest mean accuracy, while StatLoRA obtains a close second-best result.

Overall, these results indicate that StatLoRA remains competitive when applied to a larger pretrained language model. Compared with the medium-scale experiments above, this setting is more representative of practical large language model fine-tuning, where full-parameter updates are often computationally expensive and parameter-efficient adaptation is desirable. The empirical pattern across ARC-Easy, OpenBookQA, and CommonsenseQA suggests that uncertainty-aware component selection can identify useful adaptation directions even in larger models and reasoning-oriented tasks.

\subsection{Statistical Validation and Sensitivity Analysis}

In this subsection, we examine the statistical behavior of the proposed StatLoRA method beyond its performance over other baseline methods. We first study the sensitivity of StatLoRA to its statistical hyperparameters, including the batch size $b$ used in the batch-means variance estimator and the quantile level $q$ used to construct the threshold in hypothesis testing. We then investigate the stability of the rank allocation across repeated runs and use the resulting allocation patterns to study how different layers and modules contribute to fine-tuning on specific datasets.
Finally, we provide empirical diagnostics for the asymptotic theory of component scores in \Cref{theorem_clt_score}.

\subsubsection{Sensitivity to Hyperparameters}
\label{sec_experiment_hyperparameters}

We examine the sensitivity of StatLoRA to hyperparameters used in the statistical inference for LoRA components. In particular, StatLoRA involves two hyperparameters that directly enter the hypothesis-testing step: the batch size $b$ in the batch-means variance estimator and the quantile level $q$ used to determine the adaptive threshold.
Since $Mb$ determines the number of past score values used for estimating the long-run variance, we vary $b$ to study its effect.

The batch size $b$ controls the length of each batch used to estimate the long-run variance of the averaged score process. A smaller $b$ uses more local information but may lead to noisier variance estimates, whereas a larger $b$ produces more stable averages but reduces the number of batches.
The quantile level $q$ determines the threshold $\Delta^{(t)}$ in the hypothesis testing by taking the $q$-quantile of the averaged component scores among active components. Thus, $q$ controls the relative stringency of the component's hypothesis testing criterion.

We conduct this ablation study on SST-2 from the GLUE datasets, using the same DeBERTaV3-base model and training protocol as in the main NLU experiments. We vary $b$ and $q$ separately while keeping all other hyperparameters fixed, including the initial LoRA rank, final rank budget, optimizer, learning-rate schedule, and rank-reduction schedule. This design isolates the effect of the variance-estimation and threshold-selection parameters on the final performance of StatLoRA.

\begin{figure}[t!]
\centering

\begin{subfigure}{0.48\textwidth}
\centering
\begin{tikzpicture}
\begin{axis}[
    width=\textwidth,
    height=0.72\textwidth,
    xlabel={Batch size $b$},
    ylabel={Accuracy},
    ymin=95.5,
    ymax=96.3,
    ytick={95.6,95.8,96.0,96.2},
    symbolic x coords={40,60,80,100,120,140,160,200},
    xtick=data,
    xticklabel style={font=\small},
    yticklabel style={font=\small},
    label style={font=\small},
    grid=both,
    grid style={gray!20},
]

\addplot+[
    only marks,
    mark=*,
    mark size=2pt,
    black,
    error bars/.cd,
    y dir=both,
    y explicit,
    error bar style={gray!55, line width=0.9pt}
]
coordinates {
    (40,95.757) +- (0,0.1619275)
    (60,95.7855) +- (0,0.124917)
    (80,95.814) +- (0,0.1902538)
    (100,95.84275) +- (0,0.1489083)
    (120,95.92875) +- (0,0.0991599)
    (140,95.843) +- (0,0.1077259)
    (160,95.81425) +- (0,0.1282388)
    (200,95.90025) +- (0,0.1203774)
};

\end{axis}
\end{tikzpicture}
\caption{Sensitivity to batch size $b$.}
\label{fig:ablation_b}
\end{subfigure}
\hfill
\begin{subfigure}{0.48\textwidth}
\centering
\begin{tikzpicture}
\begin{axis}[
    width=\textwidth,
    height=0.72\textwidth,
    xlabel={Threshold quantile $q$},
    ylabel={Accuracy},
    ymin=95.5,
    ymax=96.3,
    ytick={95.6,95.8,96.0,96.2},
    symbolic x coords={10,20,30,40,50,80,90},
    xtick=data,
    xticklabel style={font=\small},
    yticklabel style={font=\small},
    label style={font=\small},
    grid=both,
    grid style={gray!20},
]

\addplot+[
    only marks,
    mark=*,
    mark size=2pt,
    black,
    error bars/.cd,
    y dir=both,
    y explicit,
    error bar style={gray!55, line width=0.9pt}
]
coordinates {
    (10,95.87125) +- (0,0.1810406)
    (20,95.8715) +- (0,0.1322132)
    (30,95.92875) +- (0,0.0991599)
    (40,95.843) +- (0,0.1443745)
    (50,95.87125) +- (0,0.2144777)
    (80,95.90025) +- (0,0.2203774)
    (90,95.843) +- (0,0.1877259)
};

\end{axis}
\end{tikzpicture}
\caption{Sensitivity to threshold quantile $q$.}
\label{fig:ablation_q}
\end{subfigure}
\caption{
Sensitivity analysis of StatLoRA on the SST-2 dataset in the NLU task. Each setting is evaluated over five independent random seeds. Points represent mean accuracy, and gray error bars represent one empirical standard deviation.}
\label{fig:ablation_stat_params}
\end{figure}
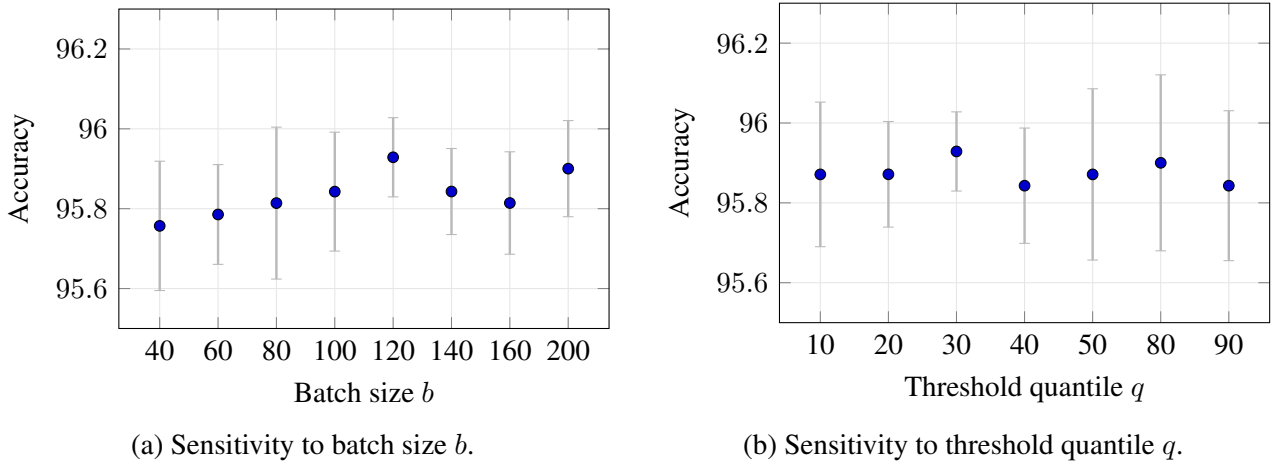

As shown in \Cref{fig:ablation_stat_params}, StatLoRA is stable over the tested ranges of $b$ and $q$. \Cref{fig:ablation_b} reports the sensitivity to the batch size $b$ in the batch-means variance estimator. When $b$ is small, the block averages may not adequately capture the temporal dependence of the score trajectory, which can lead to unstable variance estimates. When $b$ is relatively large, it can still maintain satisfactory performances. 
In this experiment, $b=120$ gives the highest mean accuracy and the smallest empirical standard deviation among the tested values, suggesting a favorable batch size given a fixed number of total iterations.

\Cref{fig:ablation_q} reports the sensitivity to the threshold quantile $q$. The performance remains stable across a range of values, indicating that StatLoRA is not highly sensitive to the precise choice of $q$. Very small or very large quantile levels lead to slightly lower mean accuracy or larger variability. This behavior is consistent with the role of $q$ in determining the adaptive threshold: a smaller value of $q$ yields a more permissive pruning criterion, whereas a larger value yields a more stringent criterion. The default value $q=0.3$ gives the highest mean accuracy with relatively small variability in this experiment.

Overall, these results suggest that StatLoRA is robust to moderate changes in the batch means and thresholding parameters. Based on this sensitivity analysis, we use $b=120$, and $q=0.3$ as the default configuration in the main NLU experiments.

\subsubsection{Stability of Rank Allocation}

Beyond performance comparisons with existing methods, we further examine the stability of the rank allocation produced by StatLoRA. Adaptive rank allocation is inherently data- and optimization-dependent, and different random seeds may lead to different selected components even under the same total rank budget. 
For a statistical inference-based pruning procedure, it is therefore important to evaluate whether the resulting allocation patterns are reproducible and stable across independent runs, rather than being driven primarily by stochastic variation from random initialization, data sampling, or optimization noise.

We conduct this analysis on the OpenBookQA dataset in the question answering task using Qwen2.5-7B as the pretrained language model. The experiment is repeated over five independent random seeds, while keeping the pretrained model, optimizer, training schedule, initial rank, final rank budget, and StatLoRA hyperparameters fixed. For each run, we record the final retained rank of each LoRA module after the rank-reduction phase. We then compute the mean and empirical standard deviation of the retained ranks across runs. The mean heatmap in \Cref{fig:rank_mean} summarizes the typical allocation pattern selected by StatLoRA, whereas the standard-deviation heatmap in \Cref{fig:rank_std} measures the stability of these allocation decisions across independent training trajectories.

\begin{figure}[t!]
\centering

\begin{subfigure}{\textwidth}
    \centering
    \includegraphics[width=\textwidth]{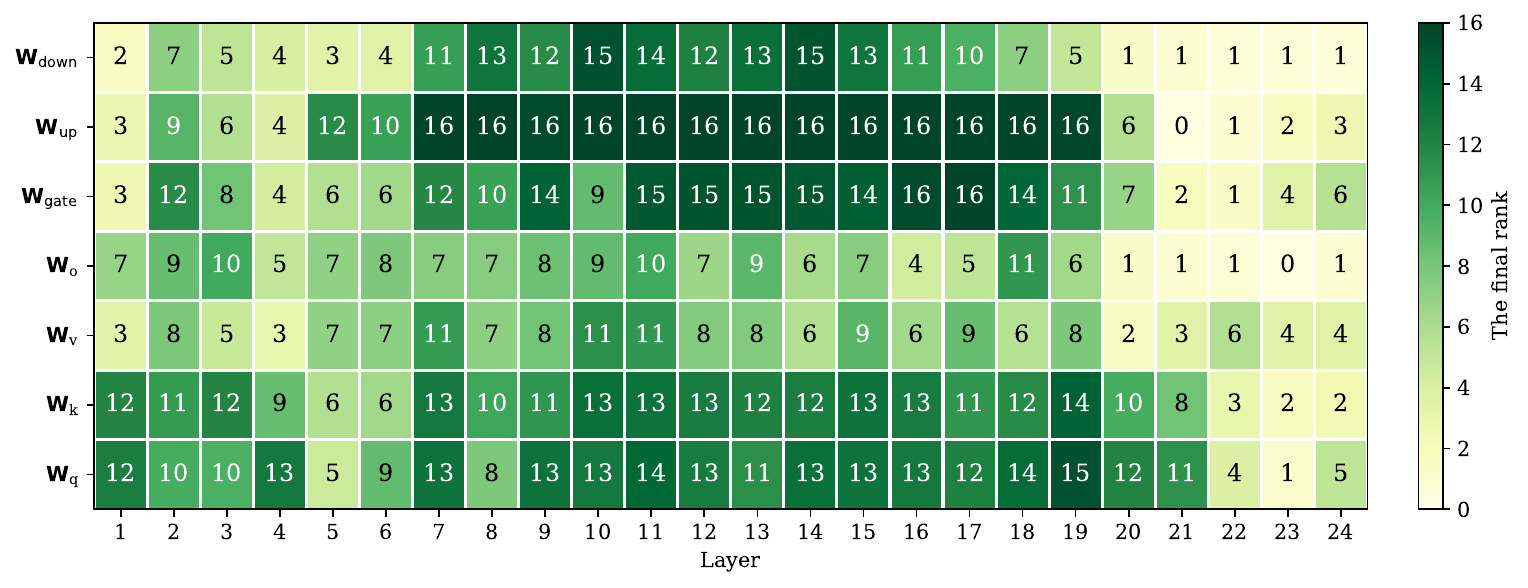}
    \caption{Mean retained rank across five independent runs.}
    \label{fig:rank_mean}
\end{subfigure}

\vspace{0.8em}

\begin{subfigure}{\textwidth}
    \centering
    \includegraphics[width=\textwidth]{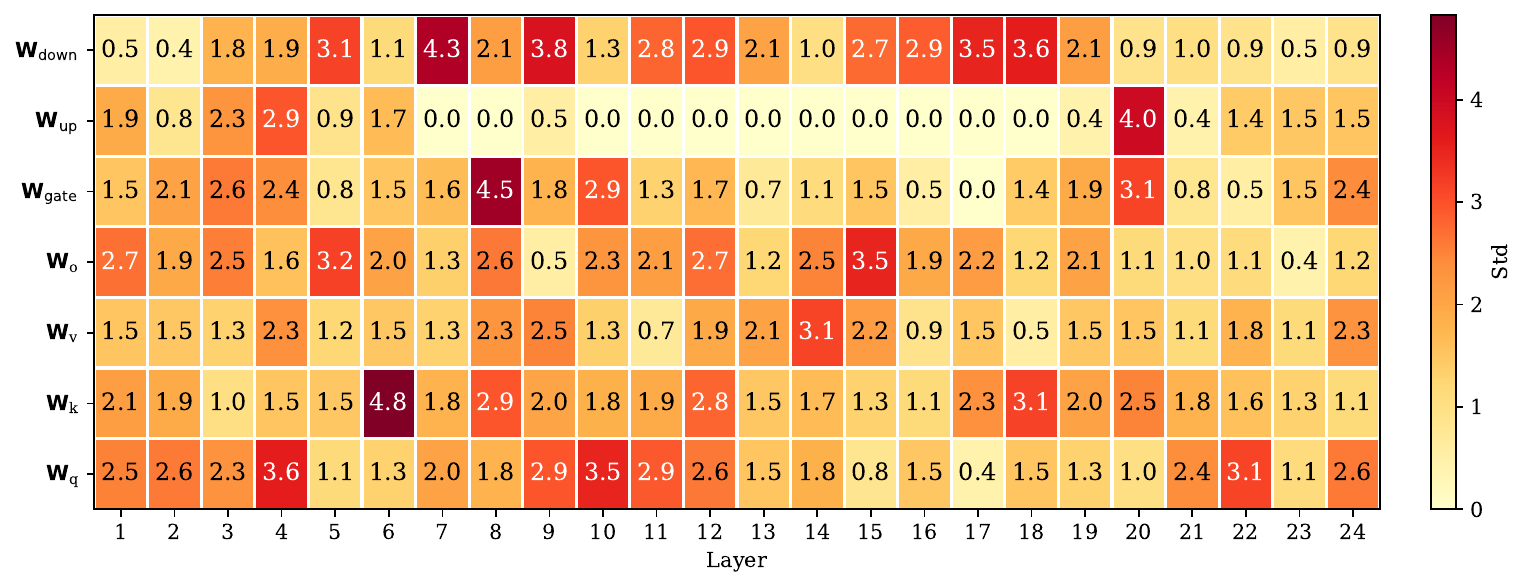}
    \caption{Standard deviation of retained rank across five independent runs.}
    \label{fig:rank_std}
\end{subfigure}

\caption{Stability analysis of the selected rank allocation produced by StatLoRA on OpenBookQA with Qwen2.5-7B. Figure (a) reports the mean retained rank of each LoRA module across five independent runs, while Figure (b) reports the corresponding standard deviation.}
\label{fig:rank_allocation_stability}
\end{figure}

This analysis complements the accuracy results in Table~\ref{tab:qa_results}. While the performance measures the final task outcome, the rank allocation heatmaps provide a more direct view of the component-selection behavior of StatLoRA. Stable allocation patterns indicate that the hypothesis-testing-based pruning rule tends to select similar modules across repeated stochastic training runs, whereas large variability would indicate sensitivity to randomness during training and pruning.

As shown in \Cref{fig:rank_allocation_stability}, the rank allocation produced by StatLoRA exhibits a structured and stable pattern. The mean-rank heatmap in \Cref{fig:rank_mean} shows that the retained ranks are not uniformly distributed across layers and module types. Larger ranks are assigned mainly to middle layers, particularly from layers 7 to 19, and later layers (layers 22-24) tend to receive less adaptation. 
The corresponding standard deviation heatmap in \Cref{fig:rank_std} shows that the variability is generally moderate relative to the scale of the retained ranks. This suggests that the proposed StatLoRA method is robust to randomness during fine-tuning.

The heatmaps also reveal module-wise heterogeneity. In particular, the MLP-related modules, especially $\bW_{\mathrm{up}}$ and $\bW_{\mathrm{gate}}$, tend to receive larger retained ranks than many of the attention projection matrices. This pattern suggests that, for OpenBookQA fine-tuning with Qwen2.5-7B, the feed-forward transformation blocks may play an important role in task-specific adaptation. This observation is consistent with the functional role of these modules: $\bW_{\mathrm{up}}$ expands the hidden representation into a higher-dimensional intermediate space, while $\bW_{\mathrm{gate}}$ controls the gated nonlinear transformation. 
We also observe that the key and value projection matrices, $\bW_{\mathrm{k}}$ and $\bW_{\mathrm{v}}$, receive relatively large rank allocations, indicating that the attention mechanism also contributes substantially to task adaptation.
We regard these patterns as descriptive evidence derived from the learned rank allocations.

The standard-deviation heatmap provides further information about the uncertainty of the allocation. Modules with consistently large retained ranks, such as several $\bW_{\mathrm{up}}$ modules in earlier and middle layers, tend to have relatively small variability across runs, indicating that these allocation decisions are stable under different random seeds. Larger standard deviations appear mainly in transition regions where the mean retained rank is neither clearly high nor clearly low. This suggests that the main source of allocation variability is concentrated around statistically ambiguous modules, where the statistical evidence for retaining rank-one components is less decisive.

Overall, the rank-allocation analysis indicates that StatLoRA produces structured and stable allocation patterns across layers and module types. These observations provide additional evidence that the proposed hypothesis-testing-based pruning strategy is not merely optimizing evaluation accuracy, but also providing interpretable component selection behavior under a fixed rank budget.

\subsubsection{P-value Diagnostics for Pruned Components}

We next examine the p-values associated with the components removed by StatLoRA. This analysis is intended as a diagnostic check of the proposed hypothesis-testing interpretation. Recall that StatLoRA is based on the null hypothesis
\[
    \mathcal{H}^{0}_{\ell,j}: s_{\ell,j}^{\star}\geq \Delta,
\]
where $s_{\ell,j}^{\star}$ denotes the population score of the $(\ell,j)$-th rank-one component and $\Delta$ is the threshold. A small one-sided p-value provides evidence against this null hypothesis, and hence indicates that the corresponding component has insufficient statistical support for being retained. 
During the rank-reduction phase, StatLoRA removes components with the smallest p-values. We collect the p-values of the removed components and visualize their empirical distribution.

\begin{figure}[t!]
    \centering
    \includegraphics[width=0.7\textwidth]{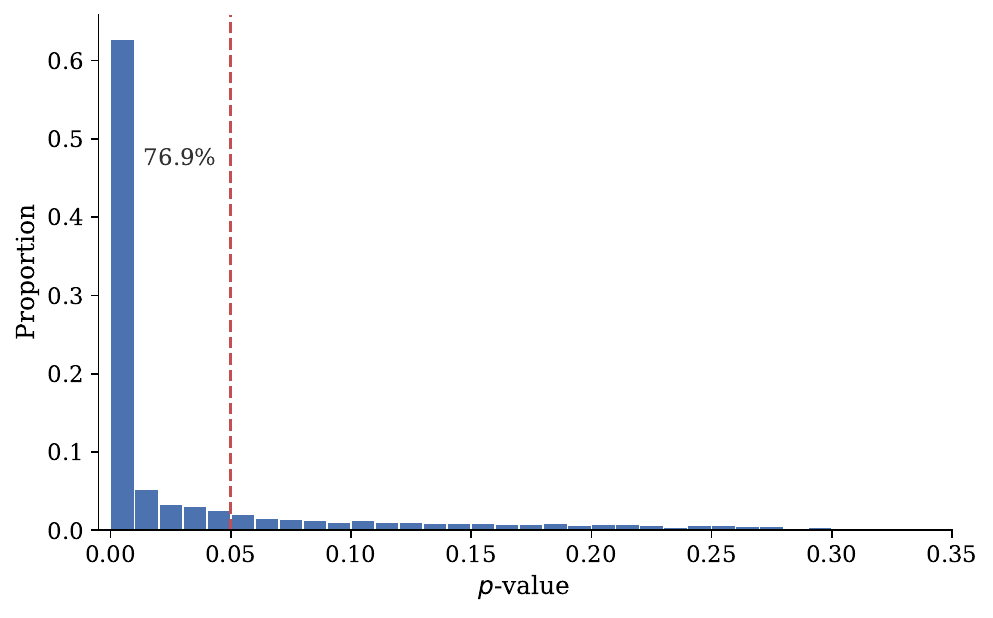}
    \caption{Distribution of p-values for pruned LoRA components. The red dashed line indicates the classical significance level $0.05$. The majority of pruned components have p-values concentrated near zero, with $76.9\%$ falling below $0.05$, indicating that most pruning decisions are associated with strong evidence against the null hypothesis.}
    \label{fig:pruned_pvalue_hist}
\end{figure}

As shown in \Cref{fig:pruned_pvalue_hist}, the p-values of the pruned components are concentrated near zero. In particular, $76.9\%$ of the removed components have p-values below $0.05$. This pattern is consistent with the intended statistical interpretation of StatLoRA: the components selected for removal are typically those for which the null hypothesis is not well supported by the empirical score trajectory.
In other words, the pruning rule is not based only on the magnitude of the empirical score, but also incorporates the estimated uncertainty of that score through the standardized test statistic.

This analysis should be interpreted as a diagnostic. The purpose of this diagnostic is to verify that the empirical pruning behavior is consistent with the proposed hypothesis testing criterion. 
The observed concentration near zero suggests that, in this experiment, most removed components have weak statistical evidence for retention relative to the threshold.

\subsubsection{Empirical Evidence for Asymptotic Normality}

We further examine the asymptotic normality underlying the hypothesis testing for component scores in StatLoRA. The central limit theorem in \Cref{theorem_clt_score} is asymptotic and is derived for each empirical component score. 
Therefore, an empirical validation that closely reflects the asymptotic theorem requires repeated training runs under the same and suitable experimental setting.

To provide such a diagnostic, we conduct an additional experiment on OpenBookQA with Qwen2.5-7B. We repeat the training procedure over 200 independent random seeds without pruning, and randomly select a LoRA component for evaluation. For each run, we compute the averaged empirical score and construct the corresponding standardized statistic. 
Since the population score $s_{\ell,j}^{\star}$ is unknown, we approximate it by the average score over the final 100 iterations. The variance is estimated using the same batch-means estimator as in StatLoRA.

















\begin{figure}[t!]
\centering
\includegraphics[width=0.55\textwidth]{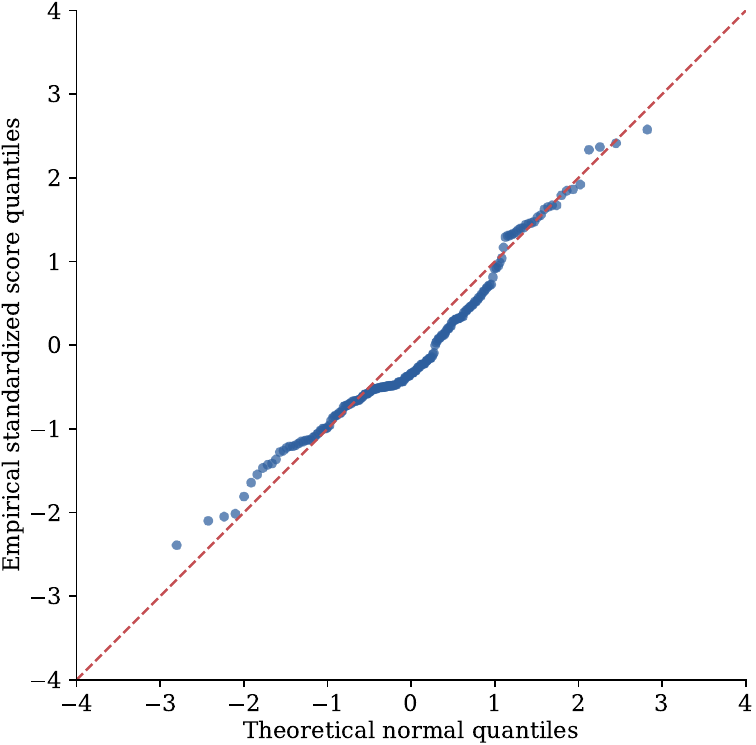}

\caption{
Q--Q plot for the standardized averaged score statistic of a representative LoRA component on OpenBookQA with Qwen2.5-7B. The training procedure is repeated over 200 independent random seeds without pruning. The empirical quantiles are plotted against the theoretical quantiles of the standard normal distribution. The approximate agreement with the reference line provides empirical evidence consistent with the asymptotic normality established in \Cref{theorem_clt_score}.
}

\label{fig:score_gaussian_assessment}
\end{figure}

As shown in \Cref{fig:score_gaussian_assessment}, the empirical distribution of the standardized statistics is broadly consistent with the asymptotic normality result in \Cref{theorem_clt_score}. The histogram is centered near zero and approximately follows the general shape of the standard normal density, while the Q-Q plot shows that most empirical quantiles lie close to the diagonal reference line.
Deviations are still visible, especially in the tails and around the central region. Such discrepancies are expected because the theoretical result is asymptotic, whereas the experiment is based on a finite training horizon of approximately 6000 iterations. 
In addition, $s_{\ell,j}^{\star}$ is approximated by a late-stage average, and the variance is estimated by batch means, both of which introduce finite-sample approximation error.

Overall, this experiment provides empirical support for the asymptotic normality of the averaged score statistic derived in \Cref{theorem_clt_score}. We interpret the result as an empirical diagnostic for the asymptotic theory.

\section{Conclusion and Discussion}

In this paper, we developed StatLoRA, a statistically guided method for LoRA rank allocation. The key idea is to formulate component selection as a hypothesis testing problem, where rank-one LoRA components are retained or pruned according to their estimated contribution and the associated uncertainty. This provides a statistical interpretation of rank allocation, in contrast to existing adaptive rank methods based primarily on carefully designed importance scores.
To support this hypothesis-testing-based procedure, we established asymptotic theory for stochastic optimization algorithms commonly used in deep learning, with particular emphasis on AdamW. Within the LoRA fine-tuning setting, we further derived central limit theory for empirical LoRA scores using the delta method. To make the inference procedure practically implementable, we used a batch-means estimator to estimate the long-run variance of the averaged score statistic and established its consistency under suitable conditions. These theoretical ingredients lead to a practical p-value-based pruning rule, which is integrated into the StatLoRA fine-tuning procedure.
Comprehensive experiments across natural language understanding, natural language generation, and question answering tasks demonstrate that StatLoRA is competitive with, and often improves upon, existing fixed-rank and adaptive-rank LoRA methods under matched rank budgets. Beyond performance comparisons, we conducted sensitivity analyses for the statistical hyperparameters, examined the stability of the selected rank allocation across repeated runs, and provided empirical diagnostics for the asymptotic normality of component scores. These results support the practical usefulness of the proposed uncertainty-aware and hypothesis-testing-based rank allocation method.

The proposed framework also has several limitations. The p-values used by StatLoRA should be interpreted as evidence of pruning for each component, rather than as formal guarantees for simultaneous multiple testing. 
In addition, the testing procedure relies on asymptotic normality and finite-sample variance estimation from stochastic training trajectories. 
Although our empirical results suggest that the resulting method is stable and practically useful, the calibration of the tests may depend on the training dynamics, batch-means estimation, and finite-sample behavior of large-scale fine-tuning. 
Thus, StatLoRA should be viewed as a statistically guided rank allocation method.

Several directions remain for future work. First, it would be natural to incorporate multiple-testing correction or false discovery rate control into the component selection rule, especially when the number of LoRA components is large. 
The current framework performs hypothesis testing for each individual component, whereas a joint or multiple testing formulation may provide a more principled approach to rank allocation by accounting for dependence and interactions among components.
Second, finite-sample calibration and inference methods would be valuable. Possible directions include self-normalized statistics, bootstrap calibration, and refined concentration bounds. Such results would strengthen the statistical guarantees beyond the asymptotic regime considered in this paper. 
Third, future work may consider adaptive procedures that both prune and reallocate components during training, allowing the rank budget to move more flexibly across layers and modules. 

\newpage 
\bibliographystyle{plain}
\bibliography{main}

\newpage
\appendix

\section{Proofs of Main Theorems}

\subsection{Preparation Lemmas}
We first present several auxiliary lemmas that are used in the proofs of the main theorems. Their proofs are deferred to Appendix~\ref{appendix_proof_technical_lemmas}.
\begin{lemma}[Lemma~B.3 in \cite{na2025statistical}]\label{lemma1}

Let $\alpha_k = \iota_1 k^{-b_1}$ and $\beta_k = \iota_2 k^{-b_2}$ for some $\iota_1, \iota_2 >0$ and $b_1, b_2 > 0$. 
Define $\chi = 0$ if $0 < b_2 < 1$ and $\chi = -b_1/\iota_2$ if $b_2 = 1$. If $\sum_{t=1}^{l} a_t + \chi> 0$, then
\begin{equation*}
\lim_{k \to \infty} \frac{1}{\alpha_k} \sum_{i=1}^{k} \prod_{j=i+1}^{k} \prod_{t = 1}^{l} \left( 1 - a_{t} \beta_{j} \right) \beta_{i} \alpha_{i} =  \frac{1}{\sum_{t=1}^{l} a_t + \chi}.
\end{equation*}		
Moreover, for any $b \in \mathbb{R}$ and any sequence $e_k$ satisfying $\lim\limits_{k\to\infty}e_{k} \to 0$, we have
\begin{equation*}
\lim_{k \to \infty} \left\{\frac{1}{\alpha_k} \sum_{i=1}^{k} \prod_{j=i+1}^{k} \prod_{t = 1}^{l} \left( 1 - a_{t} \beta_{j} \right) \beta_{i} \alpha_{i} e_{i} + b \prod_{j=1}^{k} \prod_{t = 1}^{l} \left( 1 - a_{t} \beta_{j} \right) \right\} = 0.
\end{equation*}
\end{lemma}

\begin{lemma}[Stability estimate for the transition matrices]
\label{lemma2}
Let $\bJ\in\mathbb{R}^{d\times d}$ be a Hurwitz matrix, and let
\[
\Phi_{t,k}:=\prod_{j=k}^t (\bI+\gamma_j \bJ),
\qquad 1\le k\le t,
\]
with the convention $\Phi_{t,t+1}=\bI$. Assume that $\gamma_j>0$ and $\gamma_j\to 0$.
Then there exist a vector norm $\|\cdot\|_*$ on $\mathbb{R}^d$, the associated induced matrix norm still denoted by $\|\cdot\|_*$, a constant $c>0$, and $N\in\mathbb{N}$ such that for all $t \ge k\ge N$,
\begin{equation*}
\|\bI+\gamma_k \bJ\|_* \le 1-c\gamma_k,
\end{equation*}
and consequently
\begin{equation*}
\|\Phi_{t,k}\|_*
\leq
\prod_{j=k}^t (1-c\gamma_j).
\end{equation*}
Moreover, by norm equivalence, there exists a constant $C_0>0$ such that
\begin{equation}
\label{eq1}
\|\Phi_{t,k}\|_2
\le
C_0 \prod_{j=k}^t (1-c\gamma_j),
\qquad t \ge k\ge N.
\end{equation}
\end{lemma}

\begin{lemma}
\label{lemma3}
Let $\bJ\in\mathbb{R}^{d\times d}$ be Hurwitz, and let $\bQ\in\mathbb{R}^{d\times d}$ be a symmetric positive semidefinite matrix. Let
\[
\gamma_t=\gamma_0 t^{-\kappa},
\qquad
\gamma_0>0,\quad \kappa\in(1/2,1),
\]
and define
\[
\Phi_{t,k}:=\prod_{j=k}^{t}(\bI+\gamma_j \bJ),
\qquad 1\le k\le t,
\]
with the convention $\Phi_{t,t+1}=\bI$.
Define
\begin{equation*}
\bU_t
:=
\sum_{k=1}^{t}
\frac{\gamma_k^2}{\gamma_t}\,
\Phi_{t,k+1}\bQ\Phi_{t,k+1}^\top.
\end{equation*}
Then $\bU_t$ converges to a matrix $\bSig$, where $\bSig$ is the unique positive semidefinite solution to the Lyapunov equation
\begin{equation*}
\bJ\bSig+\bSig \bJ^\top + \bQ = \bzero.
\end{equation*}
\end{lemma}

\subsection{Main Lemmas}
We present several key lemmas that are used in the proof of the asymptotic theory for the general stochastic recursion. Their detailed proofs are deferred to \Cref{appendix_proof_main_lemmas}.
\begin{lemma}[Martingale central limit theory for the weighted noise term]
\label{lemma4}
Let
\[
\mathcal J_t
:=
\sum_{k=1}^t
\Phi_{t,k+1}\frac{\gamma_k}{\sqrt{\gamma_t}}\bxi_k,
\qquad
\Phi_{t,k}:=\prod_{j=k}^t (\bI+\gamma_j \bJ).
\]
Under \Cref{assumption1}, we have
\[
\mathcal J_t \xrightarrow{d} \mathcal N(\bzero,\bSig),
\]
where $\bSig$ is characterized by the following Lyapunov equation:
\begin{equation*}
\bJ\bSig + \bSig \bJ^\top + \bQ = \bzero.
\end{equation*}
\end{lemma}

\begin{lemma}
\label{lemma5}
Define
\begin{equation*}
\be_{t}
=
\bz_t - \bz^{\star},
\end{equation*}
where $\{\bz_t\}$ is the stochastic approximation sequence. Under \Cref{assumption1}, we have
\begin{equation*}
    \left\|\sum_{k=1}^{t} \Phi_{t,k+1} \gamma_k \left(\rho(\be_k)+ \br_k\right) \right\| = o_{p}\left(\gamma_t^{1/2}\right).
\end{equation*}
\end{lemma}

\subsection{Proof of \Cref{theorem1}}
\label{appendix_proof1}
Define the error process
\begin{equation*}
\be_t := \bz_t - \bz^\star.
\end{equation*}
Since $F(\bz^\star)=\bzero$, the recursion
\[
\bz_{t+1}
=
\bz_t + \gamma_t \big(F(\bz_t)+\bxi_{t}+\br_{t}\big)
\]
can be rewritten as
\begin{equation*}
\be_{t+1}
=
\be_t + \gamma_t \big(F(\bz^\star+\be_t)-F(\bz^\star)+\bxi_{t}+\br_{t}\big).
\end{equation*}
By (A2) in \Cref{assumption1}, $F$ is continuously differentiable in a neighborhood of $\bz^\star$, and hence admits the local expansion
\begin{equation*}
F(\bz^\star+\be_t)
=
F(\bz^\star)+\bJ \be_t+\rho(\be_t)
=
\bJ \be_t+\rho(\be_t),
\end{equation*}
where $\|\rho(\be_t)\| = \mathcal{O}(\|\be_t\|^2)$ as $\be_t\to \bzero$. Therefore,
\begin{equation}
\label{eq2}
\be_{t+1}
=
(\bI+\gamma_t \bJ)\be_t
+
\gamma_t \bxi_{t}
+
\gamma_t\big(\rho(\be_t)+\br_{t}\big).
\end{equation}

Let
\begin{equation*}
\Phi_{t,k}:=\prod_{j=k}^t (I+\gamma_j \bJ).
\end{equation*}
By repeated substitution of \eqref{eq2}, we obtain the variation-of-constants representation
\begin{equation*}
\gamma_t^{-1/2}\be_{t+1}
=
\gamma_t^{-1/2}\Phi_{t,1}\be_1
+
\sum_{k=1}^t
\Phi_{t,k+1}\frac{\gamma_k}{\sqrt{\gamma_t}}\bxi_k
+
\gamma_t^{-1/2}\sum_{k=1}^t
\Phi_{t,k+1}\gamma_k\big(\rho(\be_{k})+\br_k\big).
\end{equation*}

For the first term, it follows from \eqref{eq1} that
\begin{equation*}
\begin{split}
    \left\| \gamma_t^{-1/2}\Phi_{t,1}\be_1 \right\| & \leq C_0 \gamma_t^{-1/2} \prod_{j=1}^{t}\left(1 - c\gamma_j\right) \left\| \be_1 \right\| \\
    & \leq C_0 \gamma_t^{-1/2} \exp\left(1-c\sum_{j=1}^{t} \gamma_{j}\right) \left\| \be_1 \right\| \to 0.
\end{split}
\end{equation*}
We then consider the third term, by \Cref{lemma5}, we have
\[
\gamma_t^{-1/2}\sum_{k=1}^t
\Phi_{t,k+1}\gamma_k\big(\rho(\be_{k})+\br_k\big) \xrightarrow{p} 0.
\]

It therefore remains to analyze the martingale term
\begin{equation*}
\mathcal{J}_t
:=
\sum_{k=1}^t
\Phi_{t,k+1}\frac{\gamma_k}{\sqrt{\gamma_t}}\bxi_k.
\end{equation*}
By \Cref{lemma4}, under \Cref{assumption1}, we have
\begin{equation*}
\mathcal{J}_t \xrightarrow{d} \mathcal{N}(\bzero,\bSig),
\end{equation*}
where $\bSig$ is the unique positive semidefinite solution to the corresponding Lyapunov equation. Hence, by Slutsky's theorem,
\begin{equation*}
\gamma_t^{-1/2}(\bz_t-\bz^\star)
=
\gamma_t^{-1/2}\be_t
\xrightarrow{d}
\mathcal{N}(\bzero,\bSig).
\end{equation*}

Now, we proceed to prove the asymptotic normality for the Polyak-Ruppert average. We continue from
\begin{equation*}
\be_{t+1}
=
\be_t + \gamma_t \big(\bJ\be_{t}+\bxi_{t}+\rho(\be_t)+\br_{t}\big).
\end{equation*}
For simplicity, we denote $\bdelta_t = \rho(\be_t)+\br_{t}$. 
Rearranging gives
\begin{equation}
\label{eq:telescoping-identity}
    \bJ\be_t = \frac{\be_{t+1}-\be_t}{\gamma_t} - \bxi_{t} - \bdelta_t.
\end{equation}
Summing \eqref{eq:telescoping-identity} from $t=1$ to $T$ yields
\begin{equation}
\label{eq:sum-identity}
    \bJ\sum_{t=1}^T \be_t
    =
    \sum_{t=1}^T \frac{\be_{t+1}-\be_t}{\gamma_t}
    -
    \sum_{t=1}^T \bxi_{t}
    -
    \sum_{t=1}^T \bdelta_t.
\end{equation}
Hence
\begin{equation}
\label{eq:avg-decomposition}
    \sqrt{T}\,(\bar \bz_T-\bz^\star)
    =
    -\bJ^{-1}\frac{1}{\sqrt{T}}\sum_{t=1}^T \bxi_{t}
    + \bJ^{-1}\bR_{T,1}
    - \bJ^{-1}\bR_{T,2},
\end{equation}
where
\[
    \bR_{T,1}:=\frac{1}{\sqrt{T}}\sum_{t=1}^T \frac{\be_{t+1}-\be_t}{\gamma_t},
    \qquad
    \bR_{T,2}:=\frac{1}{\sqrt{T}}\sum_{t=1}^T \bdelta_t.
\]
We now show that both remainder terms are negligible.

From the last-iterate result proved above, we have $\be_t=O_p(\gamma_t^{1/2})$. Then, Assumption~\ref{assumption1}
gives
\[
    \rho(\be_t)=O_p(\left\| \be_t\right\|^2)=O_p(\gamma_t).
\]
The negligible-remainder condition from Theorem 1 yields $\br_t=O_p(\gamma_t)$, hence
\[
    \bdelta_t=O_p(\gamma_t).
\]
Therefore,
\[
    \left\|\bR_{T,2}\right\|
    \le
    \frac{1}{\sqrt{T}}\sum_{t=1}^T O_p(\gamma_t)
    =
    O_p\!\left(T^{-1/2}\sum_{t=1}^T t^{-\kappa}\right)
    =
    O_p\!\left(T^{1/2-\kappa}\right) =  o_{p}(1),
\]
since $\kappa>1/2$.

For $\bR_{T,1}$, summation by parts gives
\[
    \sum_{t=1}^T \frac{\be_{t+1}-\be_t}{\gamma_t}
    =
    \frac{\be_{T+1}}{\gamma_T}
    -
    \frac{\be_1}{\gamma_{1}}
    +
    \sum_{t=2}^T \be_t\left(\frac{1}{\gamma_{t-1}}-\frac{1}{\gamma_t}\right).
\]
Since $\be_t=O_p(\gamma_t^{1/2})$ and $\gamma_t=\gamma_0 t^{-\kappa}$,
\[
    \frac{\left\|\be_{T+1}\right\|}{\sqrt{T}\gamma_T}
    =
    O_p\!\left(T^{\kappa/2-1/2}\right) = o_{p}(1),
\]
and
\[
    \frac{1}{\sqrt{T}}\sum_{t=2}^T
    \left\| \be_t\right\|\left(\frac{1}{\gamma_t}-\frac{1}{\gamma_{t-1}}\right)
    =
    O_p\!\left(
        \frac{1}{\sqrt{T}}
        \sum_{t=2}^T t^{-\kappa/2} t^{\kappa-1}
    \right)
    =
    O_p\!\left(T^{\kappa/2-1/2}\right) = o_{p}(1).
\]
Thus $R_{T,1}= o_{p}(1)$.

Finally, by the martingale central limit theorem and the conditional covariance assumption,
\[
    \frac{1}{\sqrt{T}}\sum_{t=1}^T \bxi_{t}\xrightarrow{d}
\mathcal{N}(\bzero,\bQ).
\]
Applying Slutsky's theorem to \eqref{eq:avg-decomposition} proves
\[
    \sqrt{T}\,(\bar \bz_T-\bz^\star)\xrightarrow{d}
\mathcal{N}(\bzero,\bJ^{-1}\bQ\bJ^{-\top}).
\]

\subsection{Proof of \Cref{theorem3}}
\label{appendix_proof2}
We verify Assumptions~(A1)--(A6) in \Cref{assumption1} one by one.

\paragraph{Verification of (A1).}
By Assumption~(B1), the augmented iterate $\bz_t$ converges almost surely to a limiting point
\[
\bz^\star=(\bx^\star,\bmm^\star,\bv^\star),
\]
satisfying \eqref{eq:adamw_eq_conditions} and \eqref{eq:adamw_eq_x}. In particular,
\[
F(\bz^\star)=\bzero,
\]
and
\[
\bz_t\to \bz^\star,
\quad \text{a.s. as}\quad t\to\infty.
\]
Hence Assumption~(A1) holds.

\paragraph{Verification of (A2).}
By Assumption~(B2), the objective function $f$ is twice continuously differentiable and the function $\bq(\bx)$ is continuously differentiable in a neighborhood of $\bx^\star$. Therefore, the mappings
\[
\bx\mapsto \nabla f(\bx),
\qquad
\bx\mapsto \bq(\bx)
\]
are continuously differentiable in a neighborhood of $\bx^\star$. Moreover, since $\varepsilon>0$
, the map
\[
(\bmm,\bv)\mapsto \frac{\bmm}{\sqrt{\bv}+\varepsilon}
\]
is continuously differentiable on any neighborhood where the denominator remains positive componentwise. It follows that the mean-field mapping
\[
F(\bz)
=
\begin{pmatrix}
-\dfrac{\bmm}{\sqrt{\bv}+\varepsilon}-\lambda \bx\\[1mm]
\alpha\bigl(\nabla f(\bx)-\bmm\bigr)\\[1mm]
\beta\bigl(\bq(\bx)-\bv\bigr)
\end{pmatrix}
\]
is continuously differentiable in a neighborhood of $\bz^\star$.
Moreover, due to the local Lipschitzness of $\nabla^2 f$ and $\nabla \bq$, $\nabla F$ is also locally Lipschitz continuous.
Hence, there exists a Jacobian matrix
\[
\bJ:=\nabla F(\bz^\star),
\]
such that
\[
F(\bz^\star+\be)
=
F(\bz^\star)+\bJ\be+\rho(\be),
\qquad
\|\rho(\be)\|=\mathcal{O}(\|\be\|^2),
\quad \be\to\bzero.
\]
Thus Assumption~(A2) holds.

\paragraph{Verification of (A3).}
Recall that the augmented noise is
\[
\bxi_t=
\begin{pmatrix}
\bzero\\[1mm]
\alpha \bxi_t^{(g)}\\[1mm]
\beta\Bigl(2\nabla f(\bx_t)\odot \bxi_t^{(g)}+\hat{\bxi}_t^{(g)}\Bigr)
\end{pmatrix},
\]
where
\[
\hat{\bxi}_t^{(g)}
=
\bxi_t^{(g)\odot 2}
-
\mathbb E\!\left[\bxi_t^{(g)\odot 2}\mid \mathcal F_{t}\right].
\]
By Assumption~(B3),
\[
\mathbb E[\bxi_t^{(g)}\mid \mathcal F_{t}]=\bzero.
\]
Also,
\[
\mathbb E[\hat{\bxi}_t^{(g)}\mid \mathcal F_{t}]
=
\mathbb E[\bxi_t^{(g)\odot 2}\mid \mathcal F_{t}]
-
\mathbb E[\bxi_t^{(g)\odot 2}\mid \mathcal F_{t}]
=
\bzero.
\]
Therefore,
\[
\mathbb E[\bxi_t\mid \mathcal F_{t}]
=
\bzero,
\]
so that $\{\bxi_t\}$ is a martingale difference sequence.

Next, we verify the bounded $(2+\delta)$-moment. Since $\bz_t\to \bz^\star$ almost surely, the sequence $\{\bx_t\}$ is almost surely bounded for all sufficiently large $t$, and hence $\{\nabla f(\bx_t)\}$ is also almost surely bounded. Therefore,
\[
\|\bxi_t\|
\le
C\Bigl(\|\bxi_t^{(g)}\|+\|\bxi_t^{(g)}\|^2\Bigr)
\]
for some deterministic constant $C>0$. By Assumption~(B3),
\[
\sup_{t\ge 1}\mathbb E\!\left[\|\bxi_t^{(g)}\|^{4+\delta}\mid \mathcal F_{t}\right]<\infty,
\qquad \text{a.s.}
\]
It follows that,
\[
\sup_{t\ge 1}\mathbb E\!\left[\|\bxi_t\|^{2+\delta/2}\mid \mathcal F_{t}\right]<\infty,
\qquad \text{a.s.}
\]
Hence, Assumption~(A3) holds.

\paragraph{Verification of (A4).}
By definition,
\[
\bxi_t=
\begin{pmatrix}
\bzero\\[1mm]
\alpha \bxi_t^{(g)}\\[1mm]
\beta\Bigl(2\nabla f(\bx_t)\odot \bxi_t^{(g)}+\hat{\bxi}_t^{(g)}\Bigr)
\end{pmatrix},
\]
where
\[
\hat{\bxi}_t^{(g)}
=
\bxi_t^{(g)\odot 2}
-
\mathbb E\!\left[\bxi_t^{(g)\odot 2}\mid \mathcal F_{t}\right].
\]

Since $\bx_t\to \bx^\star$ almost surely and $\nabla f$ is continuous, we have
\[
\nabla f(\bx_t)\to \nabla f(\bx^\star),
\qquad \text{a.s.}
\]
Hence,
\[
2\nabla f(\bx_t)\odot \bxi_t^{(g)}
=
2\,\operatorname{diag}\!\bigl(\nabla f(\bx_t)\bigr)\bxi_t^{(g)}
\to
2\operatorname{diag}\!\bigl(\nabla f(\bx^{\star})\bigr) \bxi_t^{(g)} = 2\operatorname{diag}\!\bigl(\bmm^\star\bigr) \bxi_t^{(g)}
\]
in the conditional second moment. Therefore,
\[
\mathbb E[\bxi_t\bxi_t^\top\mid \mathcal F_{t}]
\to \bQ,
\qquad \text{a.s.},
\]
where
\begin{equation}
\label{eq_Q_adamw}
\bQ=
\begin{pmatrix}
\bzero & \bzero & \bzero\\
\bzero & \bQ_{22} & \bQ_{23}\\
\bzero & \bQ_{32} & \bQ_{33}
\end{pmatrix},
\end{equation}
with
\begin{align*}
\bQ_{22}
&= \alpha^2 \bQ^{(g)},\\[1mm]
\bQ_{23}
&= \alpha\beta\bigl(2\bQ^{(g)}\operatorname{diag}\!\bigl(\bmm^{\star}\bigr)+\bR^{(g)}\bigr),\\[1mm]
\bQ_{32}
&= \alpha\beta\bigl(2\operatorname{diag}\!\bigl(\bmm^{\star}\bigr)\bQ^{(g)}+\bR^{(g)\top}\bigr),\\[1mm]
\bQ_{33}
&=\beta^2\Bigl(
4\operatorname{diag}\!\bigl(\bmm^{\star}\bigr)\bQ^{(g)}\operatorname{diag}\!\bigl(\bmm^{\star}\bigr)
+2\operatorname{diag}\!\bigl(\bmm^{\star}\bigr)\bR^{(g)}
+2\bR^{(g)\top}\operatorname{diag}\!\bigl(\bmm^{\star}\bigr)
+\hat{\bQ}^{(g)}
\Bigr).
\end{align*}
Thus, Assumption~(A4) holds.

\paragraph{Verification of (A5).}
Assumption~(B5) states exactly that the Jacobian matrix $\bJ$, defined in \eqref{eq_J_adamw}, is Hurwitz. Hence Assumption~(A5) holds.

\paragraph{Verification of (A6).}
The remainder term is
\[
\br_t=
\begin{pmatrix}
\br_t^{(x)}\\
\bzero\\
\bzero
\end{pmatrix},
\qquad
\br_t^{(x)}
=
\frac{\bmm_t}{\sqrt{\bv_t}+\varepsilon}
-
\frac{\bmm_{t+1}}{\sqrt{\bv_{t+1}}+\varepsilon}.
\]
Since
\[
\bmm_{t+1}-\bmm_t
=
\alpha_t\bigl(\nabla f(\bx_t)-\bmm_t\bigr)+\alpha_t\bxi_t^{(g)},
\]
and
\[
\bv_{t+1}-\bv_t
=
\beta_t\bigl(\bq(\bx_t)-\bv_t\bigr)
+
\beta_t\Bigl(2\nabla f(\bx_t)\odot \bxi_t^{(g)}+\hat{\bxi}_t^{(g)}\Bigr),
\]
we obtain
\[
\mathbb E\!\left[\|\bmm_{t+1}-\bmm_t\|^2\mid \mathcal F_{t}\right]
=
O(\gamma_t^2),
\]
and
\[
\mathbb E\!\left[\|\bv_{t+1}-\bv_t\|^2\mid \mathcal F_{t}\right]
=
O(\gamma_t^2),
\qquad \text{a.s.}
\]
Since the map
\[
(\bmm,\bv)\mapsto \frac{\bmm}{\sqrt{\bv}+\varepsilon}
\]
is locally Lipschitz in a neighborhood of $(\bmm^\star,\bv^\star)$, there exists $C>0$ such that
\[
\|\br_t^{(x)}\|
\le
C\Bigl(\|\bmm_{t+1}-\bmm_t\|+\|\bv_{t+1}-\bv_t\|\Bigr),
\qquad \text{a.s.}
\]
Hence
\[
\|\br_t\|^2
\le
2C^2\|\bmm_{t+1}-\bmm_t\|^2
+
2C^2\|\bv_{t+1}-\bv_t\|^2,
\qquad \text{a.s.}
\]
and therefore
\[
\mathbb E\!\left[\|\br_t\|^2\mid \mathcal F_{t}\right]
=
O(\gamma_t^2),
\qquad \text{a.s.}
\]
which implies the required negligibility of the remainder term. Thus Assumption~(A6) holds.

Combining the above verifications, we conclude that Assumptions~(A1)--(A6) in \Cref{assumption1} are satisfied. Therefore, \Cref{theorem1} applies directly to the AdamW iterates, and
\[
\gamma_t^{-1/2}(\bz_t-\bz^\star)\xrightarrow{d}\mathcal N(\bzero,\bSig),
\]
\begin{equation*}
    \sqrt{t}\left(\Bar{\bz}_t - \bz^{\star}\right)\;\xrightarrow{d}\;
\mathcal N(\bzero,\bJ^{-1}\bQ\bJ^{-\top}),
\end{equation*}
where $\bSig$ is the unique positive semidefinite solution to
\[
\bJ\bSig+\bSig\bJ^\top+\bQ=\bzero.
\]
This completes the proof.

\subsection{Proof of \Cref{theorem_clt_score}}
\label{appendix_proof3}
Fix a LoRA component $(\ell,j)$. Recall that the full optimizer state is denoted by $\bz_t$, and let $\bP_{\ell,j}$ be the selection matrix such that
\[
    \bw_{\ell,j}(t)
    :=
    \begin{pmatrix}
        \ba_{\ell,j}(t)\\
        \bb_{\ell,j}(t)
    \end{pmatrix}
    =
    \bP_{\ell,j}\bz_t .
\]
Similarly,
\[
    \bw_{\ell,j}^{\star}
    =
    \bP_{\ell,j}\bz^{\star}.
\]
By the central limit theorem for the optimizer state under Assumption~\ref{assumption3}, we have
\[
    \gamma_t^{-1/2}(\bz_t-\bz^\star)
    \xrightarrow{d}
    \mathcal N(\bzero,\bSig),
\]
and
\[
    \sqrt t(\bar\bz_t-\bz^\star)
    \xrightarrow{d}
    \mathcal N(0,\bJ^{-1}\bQ\bJ^{-\top}),
    \qquad
    \bar\bz_t:=\frac1t\sum_{\tau=1}^t \bz_\tau .
\]
Applying the continuous mapping theorem to the linear map $\bP_{\ell,j}$, we obtain
\[
    \gamma_t^{-1/2}
    \left(
        \bw_{\ell,j}(t)-\bw_{\ell,j}^{\star}
    \right)
    \xrightarrow{d}
    \mathcal N(0,\bSig_{\ell,j}^{w}),
\]
where
\[
    \bSig_{\ell,j}^{w}
    :=
    \bP_{\ell,j}\bSig\bP_{\ell,j}^{\top}.
\]
Similarly,
\[
    \sqrt t
    \left(
        \bar\bw_{\ell,j,t}-\bw_{\ell,j}^{\star}
    \right)
    \xrightarrow{d}
    \mathcal N(0,\bar{\bSig}_{\ell,j}^{w}),
\]
where
\[
    \bar\bw_{\ell,j,t}
    :=
    \frac1t\sum_{\tau=1}^t \bw_{\ell,j}(\tau),
    \qquad
    \bar{\bSig}_{\ell,j}^{w}
    :=
    \bP_{\ell,j}\bJ^{-1}\bQ\bJ^{-\top}\bP_{\ell,j}^{\top}.
\]

Define
\[
    h(\bw_{\ell,j})
    :=
    \|\ba_{\ell,j}\|^2\|\bb_{\ell,j}\|^2 .
\]
Then
\[
    s_{\ell,j}(t)=h(\bw_{\ell,j}(t)),
    \qquad
    s_{\ell,j}^{\star}=h(\bw_{\ell,j}^{\star}).
\]
Since
\[
    s_{\ell,j}^{\star}
    =
    \|\ba_{\ell,j}^{\star}\|^2
    \|\bb_{\ell,j}^{\star}\|^2
    \geq \Delta>0,
\]
both $\ba_{\ell,j}^{\star}$ and $\bb_{\ell,j}^{\star}$ are nonzero. Hence $h$ is continuously differentiable in a neighborhood of $\bw_{\ell,j}^{\star}$, and its gradient is
\[
    \nabla h(\bw_{\ell,j}^{\star})
    =
    \begin{pmatrix}
        2\|\bb_{\ell,j}^{\star}\|^{2}\ba_{\ell,j}^{\star}\\
        2\|\ba_{\ell,j}^{\star}\|^{2}\bb_{\ell,j}^{\star}
    \end{pmatrix}.
\]
By a first-order Taylor expansion,
\[
    h(\bw_{\ell,j}(t))-h(\bw_{\ell,j}^{\star})
    =
    \nabla h(\bw_{\ell,j}^{\star})^\top
    \left(
        \bw_{\ell,j}(t)-\bw_{\ell,j}^{\star}
    \right)
    +
    r_t,
\]
where
\[
    r_t
    =
    o_p\left(
        \|\bw_{\ell,j}(t)-\bw_{\ell,j}^{\star}\|
    \right).
\]
Since
\[
    \bw_{\ell,j}(t)-\bw_{\ell,j}^{\star}
    =
    O_p(\gamma_t^{1/2}),
\]
we have
\[
    \gamma_t^{-1/2}r_t=o_p(1).
\]
Therefore,
\[
\begin{aligned}
    \gamma_t^{-1/2}
    \left(
        s_{\ell,j}(t)-s_{\ell,j}^{\star}
    \right)
    &=
    \nabla h(\bw_{\ell,j}^{\star})^\top
    \gamma_t^{-1/2}
    \left(
        \bw_{\ell,j}(t)-\bw_{\ell,j}^{\star}
    \right)
    +o_p(1).
\end{aligned}
\]
Combining this expansion with the component-wise CLT and Slutsky's theorem yields
\[
    \gamma_t^{-1/2}
    \left(
        s_{\ell,j}(t)-s_{\ell,j}^{\star}
    \right)
    \xrightarrow{d}
    \mathcal N(0,\sigma_{\ell,j}^2),
\]
where
\[
    \sigma_{\ell,j}^{2}
    =
    \nabla h(\bw_{\ell,j}^{\star})^\top
    \bSig_{\ell,j}^{w}
    \nabla h(\bw_{\ell,j}^{\star}).
\]
This proves \eqref{eq_score_clt_last}.

We next prove the averaged-score result. For each $\tau$, the Taylor expansion gives
\[
    s_{\ell,j}(\tau)-s_{\ell,j}^{\star}
    =
    \nabla h(\bw_{\ell,j}^{\star})^\top
    \left(
        \bw_{\ell,j}(\tau)-\bw_{\ell,j}^{\star}
    \right)
    +
    r_\tau,
\]
where
\[
    r_\tau
    =
    O_p\left(
        \|\bw_{\ell,j}(\tau)-\bw_{\ell,j}^{\star}\|^2
    \right)
\]
under the local smoothness of $h$. Since the optimizer-state central limit theory implies
\[
    \bw_{\ell,j}(\tau)-\bw_{\ell,j}^{\star}
    =
    O_p(\gamma_\tau^{1/2}),
\]
we have
\[
    r_\tau=O_p(\gamma_\tau).
\]
Therefore,
\[
    \frac{1}{\sqrt t}\sum_{\tau=1}^t r_\tau
    =
    O_p\left(
        \frac{1}{\sqrt t}\sum_{\tau=1}^t \gamma_\tau
    \right).
\]
Because $\gamma_\tau=\gamma_0\tau^{-\kappa}$ with $\kappa\in(1/2,1)$,
\[
    \frac{1}{\sqrt t}\sum_{\tau=1}^t \gamma_\tau
    =
    O\left(t^{1/2-\kappa}\right)
    =
    o(1).
\]
Hence
\[
    \frac{1}{\sqrt t}\sum_{\tau=1}^t r_\tau=o_p(1).
\]
It follows that
\[
\begin{aligned}
    \sqrt t
    \left(
        \frac1t\sum_{\tau=1}^t s_{\ell,j}(\tau)
        -
        s_{\ell,j}^{\star}
    \right)
    &=
    \nabla h(\bw_{\ell,j}^{\star})^\top
    \sqrt t
    \left(
        \frac1t\sum_{\tau=1}^t
        \bw_{\ell,j}(\tau)
        -
        \bw_{\ell,j}^{\star}
    \right)
    +o_p(1).
\end{aligned}
\]
Using the averaged-state central limit theory and Slutsky's theorem, we obtain
\[
    \sqrt t
    \left(
        \frac1t\sum_{\tau=1}^t s_{\ell,j}(\tau)
        -
        s_{\ell,j}^{\star}
    \right)
    \xrightarrow{d}
    \mathcal N(0,\bar\sigma_{\ell,j}^2),
\]
where
\[
    \bar\sigma_{\ell,j}^{2}
    =
    \nabla h(\bw_{\ell,j}^{\star})^\top
    \bar{\bSig}_{\ell,j}^{w}
    \nabla h(\bw_{\ell,j}^{\star}).
\]
This proves \eqref{eq_score_clt_avg}.

Finally, since $\bSig_{\ell,j}^{w}$ and $\bar{\bSig}_{\ell,j}^{w}$ are positive semidefinite covariance matrices, both limiting variances are nonnegative:
\[
    \sigma_{\ell,j}^{2}\geq 0,
    \qquad
    \bar\sigma_{\ell,j}^{2}\geq 0.
\]
The limiting normal distribution is understood to be degenerate when the corresponding variance is zero.

\newpage
\section{Adam Optimizer}
\label{appendix_adam}

\subsection{Main Results}

We consider the Adam optimizer \cite{kingma2015adam} applied to the stochastic optimization problem \eqref{eq0}, where only noisy gradients are available.

Given a sequence of stochastic gradients $\bg_t := \nabla \mathcal{L}(\bx_t;\zeta_t)$, Adam updates are given by
\begin{equation}
\label{eq_adam}
    \begin{split}
    \bmm_{t+1} &= (1-\alpha_t) \bmm_t + \alpha_t\bg_{t}, \\
    \bv_{t+1} &= (1-\beta_t) \bv_t + \beta_t\bg_{t}^{\odot 2}, \\
    \bx_{t+1} &= \bx_t - \gamma_t \frac{\bmm_{t+1}}{\sqrt{\bv_{t+1}} + \varepsilon},
    \end{split}
\end{equation}
where $\alpha_t = \alpha\gamma_t,\beta_t = \beta \gamma_t$ are step sizes, $\varepsilon>0$ is a small number, and the division is understood elementwise.
We define the augmented state
\begin{equation}
\label{eq_state_adam}
\bz_t := (\bx_t, \bmm_t, \bv_t) \in \mathbb{R}^{3d},
\end{equation}
which are considered and updated during Adam's iteration.
Let $\nabla f(\bx_t)$ denote the true gradient, and write the stochastic gradient as
\[
\bg_{t} = \nabla f(\bx_t) + \bxi_{t}^{(g)},
\]
where $\{\bxi_{t}^{(g)}\}$ is a martingale difference sequence for gradients.
We denote
\begin{equation*}
    \bq(\bx) = \mathbb{E}\left[\nabla \mathcal{L}(\bx;\zeta)^{\odot 2}\right].
\end{equation*}
Then the updates can be rewritten as
\begin{align*}
& \bmm_{t+1}
= \bmm_t + \alpha_t\big(\nabla f(\bx_t) - \bmm_t\big)
+ \alpha_t\bxi_{t}^{(g)}, \\
& \bv_{t+1}
= \bv_t + \beta_t\big(\bq(\bx_t) - \bv_t\big)
+ \beta_t\left(2\nabla f(\bx_t)\odot \bxi_{t}^{(g)} + \bxi_{t}^{(g)\odot 2} - \mathbb{E}\left[\bxi_t^{(g)\odot 2}\mid \mathcal{F}_{t}\right]\right).
\end{align*}
The update of $\bx_t$ can be written as
\begin{equation*}
\bx_{t+1}
=
\bx_t
-
\gamma_t
\frac{\bmm_t}{\sqrt{\bv_t}+\varepsilon}
+
\gamma_t\br^{(x)}_{t},
\end{equation*}
where $\br^{(x)}_{t} = \frac{\bmm_t}{\sqrt{\bv_t}+\varepsilon} - \frac{\bmm_{t+1}}{\sqrt{\bv_{t+1}}+\varepsilon}$ collects the higher-order terms arising from the difference between $(\bmm_{t+1},\bv_{t+1})$ and $(\bmm_t,\bv_t)$.

Based on the analysis and reformulation of the Adam algorithm,
we define the mean-field mapping $F:\mathbb{R}^{3d}\to\mathbb{R}^{3d}$ as
\begin{equation}
F(\bz)
=
\begin{pmatrix}
-\dfrac{\bmm}{\sqrt{\bv}+\varepsilon} \\
\alpha\big(\nabla f(\bx)-\bmm\big) \\
\beta\big(\bq(\bx)-\bv\big)
\end{pmatrix}.
\end{equation}
Then, the corresponding martingale noise is given by
\begin{equation}
\bxi_{t}
=
\begin{pmatrix}
\bzero \\
\alpha\bxi_{t}^{(g)} \\
\beta\big(2\nabla f(\bx_t)\odot \bxi_{t}^{(g)} + \bxi_{t}^{(g)\odot 2} - \mathbb{E}[\bxi_{t}^{(g)\odot 2}|\mathcal{F}_t]\big)
\end{pmatrix}.
\end{equation}
The remainder term $\br_{t}$ collects all higher-order corrections, defined as
\begin{equation}
    \br_t = \begin{pmatrix}
\br_t^{(x)}\\
\bzero\\
\bzero
\end{pmatrix} = 
\begin{pmatrix}
\frac{\bmm_t}{\sqrt{\bv_t}+\varepsilon} - \frac{\bmm_{t+1}}{\sqrt{\bv_{t+1}}+\varepsilon}\\
\bzero\\
\bzero
\end{pmatrix}.
\end{equation}
With the above definitions, the Adam iteration can be written in the stochastic approximation form
\begin{equation*}
\bz_{t+1}
=
\bz_t
+
\gamma_t\big(F(\bz_t)+\bxi_{t}+\br_{t}\big).
\end{equation*}

\begin{proposition}
\label{proposition1}
The following statements hold for the Adam optimizer~\eqref{eq_adam}:
\begin{enumerate}
    \item The point $\bz^{\star} = (\bx^{\star},\bzero,\bq(\bx^{\star}))$ satisfies the Adam mean-field equation $F(\bz^\star)=\bzero$, where $\bx^{\star}$ is a stationary point of the objective function.

    \item Assume that $\bq(\bx)$ is differentiable in a neighborhood of $\bx^{\star}$. Then, the Jacobian matrix at $\bz^{\star}$ admits
\begin{equation}
\label{eq_J_adam}
\bJ=
\begin{pmatrix}
\bzero & - \operatorname{diag}\!\bigl(\frac{1}{\sqrt{\bq^{\star}} + \varepsilon}\bigr) & \bzero \\
\alpha \nabla^2 f(\bx^\star) & -\alpha \bI_d & \bzero\\
\beta \nabla \bq(\bx^{\star}) & \bzero & -\beta \bI_d
\end{pmatrix},
\end{equation}
where we denote $\bq(\bx^{\star})$ as $\bq^{\star}$. Moreover, $\bJ$ is Hurwitz if $\nabla^2 f(\bx^{\star}) \succ 0$.
\end{enumerate}

\end{proposition}

\begin{assumption}
\label{assumption2}
For the Adam recursion \eqref{eq_adam}, we impose the following conditions.

\begin{enumerate}[
    label=(C\arabic*),
    leftmargin=3.2em,
    labelwidth=2.2em,
    labelsep=0.6em,
    itemsep=0.25em,
    topsep=0.25em
]
\item \textbf{Almost sure convergence.} The augmented iterate $\bz_t=(\bx_t,\bmm_t,\bv_t)$ converges almost surely to a limiting point
\[
\bz^\star=(\bx^\star,\bzero,\bq^{\star}),
\]
where $\bx^{\star}$ is a strict local minimizer of the objective function, i.e., $\nabla f(\bx^{\star}) = 0$, $\nabla^2 f(\bx^{\star}) \succ 0$, and $\bq^{\star} = \bq(\bx^{\star})$.

\item \textbf{Local smoothness.} The objective function $f(\bx)$ is twice continuously differentiable, the Hessian $\nabla^2 f$ is Lipschitz continuous, the function $\bq(\bx)$ is continuously differentiable, and the Jacobian $\nabla \bq$ is Lipschitz continuous in a neighborhood of $\bx^{\star}$.

\item \textbf{Gradient noise with bounded moments.} The stochastic gradient admits the decomposition
\[
\bg_t = \nabla f(\bx_t) + \bxi_t^{(g)},
\]
where $\{\bxi_t^{(g)}\}$ is a martingale difference sequence with respect to $\{\mathcal F_t\}$, i.e.,
\[
\mathbb E[\bxi_t^{(g)} \mid \mathcal F_{t}] = \bzero.
\]
Moreover, there exists $\delta>0$ such that
\[
\sup_{t\ge 1}
\mathbb E\!\left[\|\bxi_t^{(g)}\|^{4+\delta}\mid \mathcal F_{t}\right]
<\infty,
\qquad \text{a.s.}
\]

\item \textbf{Asymptotic conditional covariance.} Denote $\hat{\bxi}_t^{(g)} = \bxi_{t}^{(g)\odot 2} - \mathbb{E}[\bxi_{t}^{(g)\odot 2}|\mathcal{F}_t]$. There exist positive semidefinite matrices $\bQ^{(g)}$ and $\hat{\bQ}^{(g)}$, and a matrix $\bR^{(g)}$ such that
\begin{equation}
\begin{split}
    & \mathbb{E}\big[\bxi_{t}^{(g)}\bxi_{t}^{(g)\top} \mid \mathcal{F}_t\big]
\to \bQ^{(g)},\\
& \mathbb{E}\big[\hat{\bxi}_t^{(g)}\hat{\bxi}_{t}^{(g)\top} \mid \mathcal{F}_t\big] = \mathbb{E}\big[\bxi_t^{(g)\odot 2}\bxi_{t}^{(g)\odot 2 \top} \mid \mathcal{F}_t\big] - \mathbb{E}\big[\bxi_t^{(g)\odot 2}\mid \mathcal{F}_t\big] ~\mathbb{E}\big[\bxi_t^{(g)\odot 2 }\mid \mathcal{F}_t\big]^{\top}
\to \hat{\bQ}^{(g)},\\
& \mathbb{E}\big[\bxi_t^{(g)}\hat{\bxi}_{t}^{(g)\top} \mid \mathcal{F}_t\big] = \mathbb{E}\big[\bxi_t^{(g)}\bxi_{t}^{(g)\odot 2 \top} \mid \mathcal{F}_t\big]
\to \bR^{(g)},\quad\text{a.s.}
\end{split}
\end{equation}
\end{enumerate}

\end{assumption}

\begin{theorem}
\label{theorem2}
Suppose that \Cref{assumption2} holds. Then \Cref{assumption1} is satisfied for the Adam iterates generated by \eqref{eq_adam} with the state variable $\bz_t$ defined in \eqref{eq_state_adam}.
Consequently, \Cref{theorem1} applies directly:
\begin{equation*}
\gamma_t^{-1/2}(\bz_t-\bz^{\star})
\;\xrightarrow{d}\;
\mathcal{N}(\bzero,\bSig),
\end{equation*}
and
\begin{equation*}
    \sqrt{t}\left(\Bar{\bz}_t - \bz^{\star}\right)\;\xrightarrow{d}\;
\mathcal N(\bzero,\bJ^{-1}\bQ\bJ^{-\top}),
\end{equation*}
where the covariance matrix $\bSig$ is the unique positive semidefinite solution to the Lyapunov equation
\begin{equation*}
\bJ\bSig + \bSig \bJ^\top + \bQ = \bzero,
\end{equation*}
with $\bJ$ given in \eqref{eq_J_adam}, and 
\begin{equation*}
    \bQ = \begin{pmatrix}
    \bzero & \bzero & \bzero\\
    \bzero & \alpha^2 \bQ^{(g)} & \alpha\beta \bR^{(g)} \\
    \bzero & \alpha\beta \bR^{(g)\top}  & \beta^2 \hat{\bQ}^{(g)}
    \end{pmatrix}.
\end{equation*}
\end{theorem}

\subsection{Proof of \Cref{theorem2}}

We verify (A1)--(A6) in \Cref{assumption1} one by one.

\paragraph{Verification of (A1).}
Assumption~(C1) states that the augmented iterate $\bz_t=(\bx_t,\bmm_t,\bv_t)$ converges almost surely to
\[
\bz^\star=(\bx^\star,\bzero,\bq^{\star}),
\]
where $\nabla f(\bx^\star)=\bzero$ and $\bq^{\star} = \bq(\bx^{\star})$. By the definition of the Adam mean-field mapping
\[
F(\bz)=
\begin{pmatrix}
-\dfrac{\bmm}{\sqrt{\bv}+\varepsilon}\\[1mm]
\alpha(\nabla f(\bx)-\bmm)\\[1mm]
\beta(\bq(\bx)-\bv)
\end{pmatrix},
\]
we have
\[
F(\bz^\star)
=
\begin{pmatrix}
\bzero\\
\alpha(\nabla f(\bx^\star)-\bzero)\\
\beta(\bq(\bx^{\star})-\bq^{\star})
\end{pmatrix}
=
\bzero.
\]
Hence Assumption~(A1) holds.

\paragraph{Verification of (A2).}
By Assumption~(C2), the objective function $f$ is twice continuously differentiable in a neighborhood of $\bx^\star$. Therefore, the mappings
\[
\bx \mapsto \nabla f(\bx),
\qquad
\bx \mapsto \nabla f(\bx)^{\odot 2}
\]
are continuously differentiable in a neighborhood of $\bx^\star$. Since the map
\[
(\bmm,\bv)\mapsto \frac{\bmm}{\sqrt{\bv}+\varepsilon}
\]
is continuously differentiable for $\varepsilon>0$ and $\bv > 0$, and the function $\bq(\bx)$ is continuously differentiable in a neighborhood of $\bx^{\star}$, it follows that the mean-field mapping $F$ is continuously differentiable in a neighborhood of $\bz^\star$. 
Moreover, the Lipschitzness of $\nabla^2 f$ and $\nabla \bq$ implies that the Jacobian $\nabla F$ is Lipschitz continuous in a neighborhood of $\bz^\star$.
Consequently, its first-order expansion
\[
F(\bz^\star+\be)
=
F(\bz^\star)+\bJ\be+\rho(\be),
\qquad
\|\rho(\be)\|=\mathcal{O}(\|\be\|^2),
\]
holds, and Assumption~(A2) is satisfied.

\paragraph{Verification of (A3).}
By definition,
\[
\bxi_t=
\begin{pmatrix}
\bzero\\[1mm]
\alpha\bxi_t^{(g)}\\[1mm]
\beta\bigl(2\nabla f(\bx_t)\odot \bxi_t^{(g)}+\hat{\bxi}_t^{(g)}\bigr)
\end{pmatrix},
\qquad
\hat{\bxi}_t^{(g)}
=
\bxi_t^{(g)\odot 2}
-
\mathbb E[\bxi_t^{(g)\odot 2}\mid \mathcal F_{t}].
\]
Since $\{\bxi_t^{(g)}\}$ is a martingale difference sequence,
\[
\mathbb E[\bxi_t^{(g)}\mid \mathcal F_{t}] = \bzero.
\]
Moreover,
\[
\mathbb E[\hat{\bxi}_t^{(g)}\mid \mathcal F_{t}]
=
\mathbb E[\bxi_t^{(g)\odot 2}\mid \mathcal F_{t}]
-
\mathbb E[\bxi_t^{(g)\odot 2}\mid \mathcal F_{t}]
=
\bzero.
\]
Therefore,
\[
\mathbb E[\bxi_t\mid \mathcal F_{t}] = \bzero,
\]
so that $\{\bxi_t\}$ is a martingale difference sequence.

It remains to verify the bounded $(2+\delta)$-moment. By Assumption~(C3),
\[
\sup_{t\ge1}\mathbb E\!\left[\|\bxi_t^{(g)}\|^{4+\delta}\mid \mathcal F_{t}\right] < \infty,
\qquad \text{a.s.}
\]
Since $\bx_t\to \bx^\star$ a.s. and the iterates are locally bounded, $\nabla f(\bx_t)$ is almost surely bounded for all sufficiently large $t$. Hence
\[
\|2\nabla f(\bx_t)\odot \bxi_t^{(g)}+\hat{\bxi}_t^{(g)}\|
\le
C\bigl(\|\bxi_t^{(g)}\|+\|\bxi_t^{(g)}\|^2\bigr)
\]
for some deterministic constant $C>0$. It follows that
\[
\sup_{t\ge1}\mathbb E\!\left[\|\bxi_t\|^{2+\delta}\mid \mathcal F_{t}\right] < \infty,
\qquad \text{a.s.},
\]
possibly after reducing $\delta$ if necessary. Therefore Assumption~(A3) holds.

\paragraph{Verification of (A4).}
By definition, the noise term admits
\[
\bxi_t=
\begin{pmatrix}
\bzero\\
\alpha \bxi_t^{(g)}\\
\beta\bigl(2\nabla f(\bx_t)\odot \bxi_t^{(g)}+\hat{\bxi}_t^{(g)}\bigr)
\end{pmatrix}.
\]
Hence,
\begin{equation}
\bxi_t\bxi_t^\top=
\begin{pmatrix}
\bzero & \bzero & \bzero\\
\bzero & \alpha^2\bxi_t^{(g)}\bxi_t^{(g)\top}
&
\alpha\beta\,\bxi_t^{(g)}
\bigl(2\nabla f(\bx_t)\odot \bxi_t^{(g)}+\hat{\bxi}_t^{(g)}\bigr)^\top
\\[1mm]
\bzero &
\alpha\beta\bigl(2\nabla f(\bx_t)\odot \bxi_t^{(g)}+\hat{\bxi}_t^{(g)}\bigr)\bxi_t^{(g)\top}
&
\beta^2
\bigl(2\nabla f(\bx_t)\odot \bxi_t^{(g)}+\hat{\bxi}_t^{(g)}\bigr)
\bigl(2\nabla f(\bx_t)\odot \bxi_t^{(g)}+\hat{\bxi}_t^{(g)}\bigr)^\top
\end{pmatrix}.
\end{equation}
Since $\bx_t\to \bx^\star$ almost surely and $\nabla f(\bx^\star)=\bzero$, we have
\[
\nabla f(\bx_t)\to \bzero,
\qquad \text{a.s.}
\]
Therefore, every term involving the factor $\nabla f(\bx_t)$ vanishes asymptotically. Taking conditional expectation with respect to $\mathcal F_{t}$ gives
\[
\mathbb E[\bxi_t\bxi_t^\top\mid \mathcal F_{t}]
=
\begin{pmatrix}
\bzero & \bzero & \bzero\\
\bzero &
\alpha^2\mathbb E[\bxi_t^{(g)}\bxi_t^{(g)\top}\mid \mathcal F_{t}]
&
o(1) + \alpha\beta\mathbb E\!\left[
\bxi_t^{(g)}\hat{\bxi}_t^{(g)\top}
\mid
\mathcal F_{t}
\right]
\\
\bzero &
o(1) + \alpha\beta\mathbb E\!\left[
\hat{\bxi}_t^{(g)}\bxi_t^{(g)\top}
\mid
\mathcal F_{t-1}
\right]
&
\beta^2\mathbb E[\hat{\bxi}_t^{(g)}\hat{\bxi}_t^{(g)\top}\mid \mathcal F_{t}] + o(1)
\end{pmatrix},
\qquad \text{a.s.}
\]
Using the assumed limits,
\[
\mathbb E[\bxi_t^{(g)}\bxi_t^{(g)\top}\mid \mathcal F_{t}]
\to \bQ^{(g)},
\qquad
\mathbb E[\hat{\bxi}_t^{(g)}\hat{\bxi}_t^{(g)\top}\mid \mathcal F_{t}]
\to \hat{\bQ}^{(g)},
\]
we conclude that
\begin{equation}
\label{eq9}
    \mathbb E[\bxi_t\bxi_t^\top\mid \mathcal F_{t}]
\to
\begin{pmatrix}
\bzero & \bzero & \bzero\\
\bzero & \alpha^2 \bQ^{(g)} & \alpha\beta \bR^{(g)}\\
\bzero & \alpha\beta \bR^{(g)\top} & \beta^2 \hat{\bQ}^{(g)}
\end{pmatrix}
=
\bQ,
\qquad \text{a.s.}
\end{equation}

\paragraph{Verification of (A5).}
Since $\nabla^2 f(\bx^\star)\succ \bzero$ by Assumption~(C1), the matrix $\bJ$ is Hurwitz, according to \Cref{proposition1}. Hence, Assumption~(A5) is satisfied.

\paragraph{Verification of (A6).}
The remainder term is given by
\[
\br_t=
\begin{pmatrix}
\br_t^{(x)}\\
\bzero\\
\bzero
\end{pmatrix},
\qquad
\br_t^{(x)}
=
\frac{\bmm_t}{\sqrt{\bv_t}+\varepsilon}
-
\frac{\bmm_{t+1}}{\sqrt{\bv_{t+1}}+\varepsilon}.
\]
Since
\[
\bmm_{t+1}-\bmm_t
=
\alpha_t(\nabla f(\bx_t)-\bmm_t)+\alpha_t\bxi_t^{(g)},
\]
and
\[
\bv_{t+1}-\bv_t
=
\beta_t(\bq(\bx_t)-\bv_t)
+
\beta_t\bigl(2\nabla f(\bx_t)\odot \bxi_t^{(g)}+\hat{\bxi}_t^{(g)}\bigr),
\]
we have
\[
\mathbb{E}\left[\|\bmm_{t+1}-\bmm_t\|^2 \mid \mathcal{F}_{t}\right]=\mathcal O(\gamma_t^2),
\quad
\mathbb{E}\left[\|\bv_{t+1}-\bv_t\|^2 \mid \mathcal{F}_{t}\right]=\mathcal O(\gamma_t^2), \quad \text{a.s.},
\]
under Assumptions~(C1)--(C3). Since the map
\[
(\bmm,\bv)\mapsto \frac{\bmm}{\sqrt{\bv}+\varepsilon}
\]
is locally Lipschitz, it follows that
\[
\|\br_t^{(x)}\| 
\le
C\bigl(\|\bmm_{t+1}-\bmm_t\|+\|\bv_{t+1}-\bv_t\|\bigr),
\quad \text{a.s.}
\]
Therefore,
\[
\mathbb{E}\left[\|\br_t\|^2\mid \mathcal{F}_{t}\right]=\mathcal O(\gamma_t^2),
\qquad \text{a.s.},
\]
and Assumption~(A6) holds.

Combining the above verifications, we conclude that Assumptions~(A1)--(A6) in \Cref{assumption1} are satisfied under \Cref{assumption2}. Therefore, \Cref{theorem1} applies directly to the Adam optimizer, and the corresponding asymptotic covariance matrix $\bQ$ is given by \eqref{eq9}.

\newpage

\section{Adafactor Optimizer}
\label{appendix_adafactor}

\subsection{Main Results}
We consider the AdaFactor optimizer applied to the stochastic optimization problem \eqref{eq0}, where only noisy gradients are available.

Given a sequence of stochastic gradients $\bg_t := \nabla \mathcal{L}(\bx_t;\zeta_t)$, AdaFactor updates are given by
\begin{equation}
\label{eq_adafactor}
\begin{split}
\bv_{t+1}
&=
(1-\beta_t)\bv_t+\beta_t \bg_t^{\odot 2},\\
\bx_{t+1}
&=
\bx_t-\gamma_t
\frac{\bg_t}{\sqrt{\bv_{t+1}}+\varepsilon},
\end{split}
\end{equation}
where $\beta_t=\beta\gamma_t$, $\varepsilon>0$ is a small constant, and the division is understood elementwise.
We define the augmented state
\begin{equation*}
\bz_t:=(\bx_t,\bv_t)\in\mathbb R^{2d}.
\end{equation*}
Let $\nabla f(\bx_t)$ denote the true gradient, and write
\[
\bg_t=\nabla f(\bx_t)+\bxi_t^{(g)},
\]
where $\{\bxi_t^{(g)}\}$ is a martingale difference sequence.
Then,
\begin{align*}
& \bv_{t+1}
=
\bv_t+\beta_t\bigl(\bq(\bx_t)-\bv_t\bigr)
+\beta_t\Bigl(
2\nabla f(\bx_t)\odot \bxi_t^{(g)}
+\bxi_t^{(g)\odot 2} - \mathbb{E}\left[\bxi_t^{(g)\odot 2}\mid \mathcal{F}_{t}\right]
\Bigr),\\
& \bx_{t+1}
=
\bx_t
-
\gamma_t
\frac{\nabla f(\bx_t)}{\sqrt{\bv_t}+\varepsilon}
- \gamma_t
\frac{\bxi_t^{(g)}}{\sqrt{\bv_t}+\varepsilon} +
\gamma_t \br_t^{(x)},
\end{align*}
where 
\[
\br_t^{(x)}
=
\frac{\bg_t}{\sqrt{\bv_t}+\varepsilon}
-
\frac{\bg_t}{\sqrt{\bv_{t+1}}+\varepsilon}.
\]

We define the mean-field mapping $F:\mathbb R^{2d}\to\mathbb R^{2d}$ by
\begin{equation}
F(\bz)
=
\begin{pmatrix}
-\dfrac{\nabla f(\bx)}{\sqrt{\bv}+\varepsilon}\\[2mm]
\beta\bigl(\bq(\bx)-\bv\bigr)
\end{pmatrix}.
\end{equation}
The corresponding martingale noise is
\begin{equation}
\bxi_t=
\begin{pmatrix}
-\frac{\bxi_t^{(g)}}{\sqrt{\bv_t}+\varepsilon}\\[1mm]
\beta\Bigl(
2\nabla f(\bx_t)\odot \bxi_t^{(g)}
+\hat{\bxi}_t^{(g)}
\Bigr)
\end{pmatrix},
\end{equation}
where
\[
\hat{\bxi}_t^{(g)}
=
\bxi_t^{(g)\odot 2}
-
\mathbb E\!\left[\bxi_t^{(g)\odot 2}\mid \mathcal F_{t}\right].
\]
The remainder term is
\begin{equation}
\br_t=
\begin{pmatrix}
\br_t^{(x)}\\
\bzero
\end{pmatrix}.
\end{equation}
Hence, the AdaFactor iteration can be written in the stochastic approximation form
\begin{equation*}
\bz_{t+1}
=
\bz_t+\gamma_t\bigl(F(\bz_t)+\bxi_t+\br_t\bigr).
\end{equation*}

\begin{proposition}
\label{proposition3}
The following statements hold for the AdaFactor optimizer~\eqref{eq_adafactor}.

\begin{enumerate}
\item The point
\(
\bz^\star=(\bx^\star,\bv^\star),
\)
satisfies the Adafactor mean-field equation $F(\bz^\star)=\bzero$,
where
\[
\nabla f(\bx^\star)=\bzero,
\qquad
\bv^\star=\bq(\bx^\star).
\]

\item The Jacobian matrix at $\bz^\star$ is
\begin{equation}
\label{eq_J_adafactor}
\bJ=
\begin{pmatrix}
-\operatorname{diag}\!\left(\frac{1}{\sqrt{\bv^\star}+\varepsilon}\right)
\nabla^2 f(\bx^\star) & \bzero\\[2mm]
\beta \nabla \bq(\bx^\star) & -\beta \bI_d
\end{pmatrix}.
\end{equation}
Moreover, the Jacobian matrix $\bJ$ is Hurwitz if $\nabla^2 f(\bx^{\star}) \succ 0$.
\end{enumerate}
\end{proposition}

\begin{assumption}
\label{assumption4}
For the Adafactor recursion \eqref{eq_adafactor}, we impose the following conditions.
\begin{enumerate}[
    label=(D\arabic*),
    leftmargin=3.2em,
    labelwidth=2.2em,
    labelsep=0.6em,
    itemsep=0.25em,
    topsep=0.25em
]
\item \textbf{Almost sure convergence.} The augmented iterate $\bz_t=(\bx_t,\bv_t)$ converges almost surely to a point
\[
\bz^\star=(\bx^\star,\bv^\star),
\]
where $\bx^{\star}$ is a strict local minimizer of the objective function, i.e., $\nabla f(\bx^{\star}) = 0$ and $\nabla^2 f(\bx^{\star}) \succ 0$, and $\bv^\star = \bq(\bx^\star)$.

\item \textbf{Local smoothness.} The objective function $f(\bx)$ is twice continuously differentiable, the Hessian $\nabla^2 f$ is Lipschitz continuous, the function $\bq(\bx)$ is continuously differentiable, and the Jacobian $\nabla \bq$ is Lipschitz continuous in a neighborhood of $\bx^{\star}$.

\item \textbf{Gradient noise with bounded moments.} The stochastic gradient admits the decomposition
\[
\bg_t = \nabla f(\bx_t) + \bxi_t^{(g)},
\]
where $\{\bxi_t^{(g)}\}$ is a martingale difference sequence with respect to $\{\mathcal F_t\}$, i.e.,
\[
\mathbb E[\bxi_t^{(g)} \mid \mathcal F_{t}] = \bzero.
\]
Moreover, there exists $\delta>0$ such that
\[
\sup_{t\ge 1}
\mathbb E\!\left[\|\bxi_t^{(g)}\|^{4+\delta}\mid \mathcal F_{t}\right]
<\infty,
\qquad \text{a.s.}
\]

\item \textbf{Asymptotic conditional covariance.} Denote $\hat{\bxi}_t^{(g)} = \bxi_{t}^{(g)\odot 2} - \mathbb{E}[\bxi_{t}^{(g)\odot 2}|\mathcal{F}_t]$. There exist positive semidefinite matrices $\bQ^{(g)}$ and $\hat{\bQ}^{(g)}$, and a matrix $\bR^{(g)}$ such that
\begin{equation*}
\begin{split}
    & \mathbb{E}\big[\bxi_{t}^{(g)}\bxi_{t}^{(g)\top} \mid \mathcal{F}_t\big]
\to \bQ^{(g)},\\
& \mathbb{E}\big[\hat{\bxi}_t^{(g)}\hat{\bxi}_{t}^{(g)\top} \mid \mathcal{F}_t\big] = \mathbb{E}\big[\bxi_t^{(g)\odot 2}\bxi_{t}^{(g)\odot 2 \top} \mid \mathcal{F}_t\big] - \mathbb{E}\big[\bxi_t^{(g)\odot 2}\mid \mathcal{F}_t\big] ~\mathbb{E}\big[\bxi_t^{(g)\odot 2 }\mid \mathcal{F}_t\big]^{\top}
\to \hat{\bQ}^{(g)},\\
& \mathbb{E}\big[\bxi_t^{(g)}\hat{\bxi}_{t}^{(g)\top} \mid \mathcal{F}_t\big] = \mathbb{E}\big[\bxi_t^{(g)}\bxi_{t}^{(g)\odot 2 \top} \mid \mathcal{F}_t\big]
\to \bR^{(g)},\quad\text{a.s.}
\end{split}
\end{equation*}
\end{enumerate}
\end{assumption}

\begin{theorem}
\label{theorem4}
Suppose that \Cref{assumption4} holds. Then Assumptions~(A1)--(A6) in \Cref{assumption1} are satisfied for the AdaFactor iterates generated by \eqref{eq_adafactor}. Consequently,
\begin{equation*}
\gamma_t^{-1/2}(\bz_t-\bz^\star)
\;\xrightarrow{d}\;
\mathcal N(\bzero,\bSig),
\end{equation*}
and
\begin{equation*}
    \sqrt{t}\left(\Bar{\bz}_t - \bz^{\star}\right)\;\xrightarrow{d}\;
\mathcal N(\bzero,\bJ^{-1}\bQ\bJ^{-\top}),
\end{equation*}
where $\bSig$ is the unique positive semidefinite solution to
\begin{equation*}
\bJ\bSig+\bSig\bJ^\top+\bQ=\bzero,
\end{equation*}
with $\bJ$ given in \eqref{eq_J_adafactor}, and
\begin{equation*}
\bQ=
\begin{pmatrix}
\text{diag}\left(\frac{1}{\sqrt{\bv^\star} + \varepsilon} \right)\bQ^{(g)}\text{diag}\left(\frac{1}{\sqrt{\bv^\star} + \varepsilon} \right) & -\text{diag}\left(\frac{\beta}{\sqrt{\bv^\star} + \varepsilon} \right)\bR^{(g)}\\[1mm]
-\text{diag}\left(\frac{\beta}{\sqrt{\bv^\star} + \varepsilon} \right)\bR^{(g)\top} & \beta^2 \hat{\bQ}^{(g)}
\end{pmatrix}.
\end{equation*}
\end{theorem}

\subsection{Proof of \Cref{theorem4}}

We verify Assumptions~(A1)--(A6) in \Cref{assumption1} one by one.

\paragraph{Verification of (A1).}
By Assumption~(D1), the augmented iterate $\bz_t=(\bx_t,\bv_t)$ converges almost surely to a limiting point
\[
\bz^\star=(\bx^\star,\bv^\star),
\]
where $\nabla f(\bx^\star)=\bzero$. By the definition of the AdaFactor mean-field mapping
\[
F(\bz)=
\begin{pmatrix}
-\dfrac{\nabla f(\bx)}{\sqrt{\bv}+\varepsilon}\\[2mm]
\beta\bigl(\bq(\bx)-\bv\bigr)
\end{pmatrix},
\]
we have
\[
F(\bz^\star)
=
\begin{pmatrix}
-\dfrac{\nabla f(\bx^\star)}{\sqrt{\bv^\star} + \varepsilon}\\[2mm]
\beta\bigl(\bq(\bx^\star)-\bv^\star\bigr)
\end{pmatrix}
=
\bzero.
\]
Hence Assumption~(A1) holds.

\paragraph{Verification of (A2).}
By Assumption~(D2), the objective function $f$ is twice continuously differentiable and the function $\bq(\bx)$ is continuously differentiable in a neighborhood of $\bx^\star$. Therefore, the mappings
\[
\bx \mapsto \nabla f(\bx),
\qquad
\bx \mapsto \bq(\bx)
\]
are continuously differentiable in a neighborhood of $\bx^\star$. Since $\varepsilon>0$ and $\bv > 0$, the map
\[
\bv \mapsto \frac{1}{\sqrt{\bv}+\varepsilon}
\]
is also continuously differentiable. Hence, the mean-field mapping $F$ is continuously differentiable in a neighborhood of $\bz^\star$.
Moreover, the Lipschitzness of $\nabla^2 f$ and $\nabla \bq$ implies that the Jacobian $\nabla F$ is Lipschitz continuous in a neighborhood of $\bz^\star$.
Consequently, its first-order expansion admits
\[
F(\bz^\star+\be)
=
F(\bz^\star)+\bJ\be+\rho(\be),
\qquad
\|\rho(\be)\|=\mathcal{O}(\|\be\|^2),
\quad \be\to\bzero.
\]
Thus Assumption~(A2) holds.

\paragraph{Verification of (A3).}
Recall that the augmented noise is
\[
\bxi_t=
\begin{pmatrix}
-\frac{\bxi_t^{(g)}}{\sqrt{\bv_t}+\varepsilon}\\[1mm]
\beta\Bigl(2\nabla f(\bx_t)\odot \bxi_t^{(g)}+\hat{\bxi}_t^{(g)}\Bigr)
\end{pmatrix},
\]
where
\[
\hat{\bxi}_t^{(g)}
=
\bxi_t^{(g)\odot 2}
-
\mathbb E\!\left[\bxi_t^{(g)\odot 2}\mid \mathcal F_t\right].
\]
By Assumption~(D3),
\[
\mathbb E[\bxi_t^{(g)}\mid \mathcal F_t]=\bzero.
\]
Moreover,
\[
\mathbb E[\hat{\bxi}_t^{(g)}\mid \mathcal F_t]
=
\mathbb E[\bxi_t^{(g)\odot 2}\mid \mathcal F_t]
-
\mathbb E[\bxi_t^{(g)\odot 2}\mid \mathcal F_t]
=
\bzero.
\]
Hence,
\[
\mathbb E[\bxi_t\mid \mathcal F_t]=\bzero,
\]
so that $(\bxi_t)$ is a martingale difference sequence.

Next, since $\bx_t\to \bx^\star$ almost surely and $\nabla f$ is continuous, the sequence of $\nabla f(\bx_t)$ is almost surely bounded for all sufficiently large $t$. Therefore,
\[
\|\bxi_t\|
\le
C\Bigl(\|\bxi_t^{(g)}\|+\|\bxi_t^{(g)}\|^2\Bigr)
\]
for some deterministic constant $C>0$. By Assumption~(D3),
\[
\sup_{t\ge1}
\mathbb E\!\left[\|\bxi_t^{(g)}\|^{4+\delta}\mid \mathcal F_t\right]
<\infty,
\qquad \text{a.s.}
\]
It follows that
\[
\sup_{t\ge1}
\mathbb E\!\left[\|\bxi_t\|^{2+\delta/2}\mid \mathcal F_t\right]
<\infty,
\qquad \text{a.s.}
\]
Hence Assumption~(A3) holds.

\paragraph{Verification of (A4).}
By definition,
\[
\bxi_t=
\begin{pmatrix}
-\frac{\bxi_t^{(g)}}{\sqrt{\bv_t}+\varepsilon}\\[1mm]
\beta\Bigl(2\nabla f(\bx_t)\odot \bxi_t^{(g)}+\hat{\bxi}_t^{(g)}\Bigr)
\end{pmatrix}.
\]
Since $\bx_t\to \bx^\star$ almost surely and $\nabla f(\bx^\star)=\bzero$, we have
\[
\nabla f(\bx_t)\to \bzero,
\qquad \text{a.s.}
\]
Therefore, the term $2\nabla f(\bx_t)\odot \bxi_t^{(g)}$ vanishes asymptotically in conditional second moment, and the conditional covariance matrix converges to
\[
\bQ=
\begin{pmatrix}
\text{diag}\left(\frac{1}{\sqrt{\bv^\star} + \varepsilon} \right)\bQ^{(g)}\text{diag}\left(\frac{1}{\sqrt{\bv^\star} + \varepsilon} \right) & -\text{diag}\left(\frac{\beta}{\sqrt{\bv^\star} + \varepsilon} \right)\bR^{(g)}\\[1mm]
-\text{diag}\left(\frac{\beta}{\sqrt{\bv^\star} + \varepsilon} \right)\bR^{(g)\top} & \beta^2 \hat{\bQ}^{(g)}
\end{pmatrix}.
\]
Hence
\[
\mathbb E[\bxi_t\bxi_t^\top\mid \mathcal F_t]
\to \bQ,
\qquad \text{a.s.}
\]
Thus Assumption~(A4) holds.

\paragraph{Verification of (A5).}
By the proposition for AdaFactor, the Jacobian matrix at the limiting point is
\[
\bJ=
\begin{pmatrix}
-\operatorname{diag}\!\left(\frac{1}{\sqrt{\bv^\star}+\varepsilon}\right)
\nabla^2 f(\bx^\star) & \bzero\\[2mm]
\beta \nabla \bq(\bx^\star) & -\beta \bI_d
\end{pmatrix}.
\]
Since $\nabla^2 f(\bx^\star)\succ \bzero$ by Assumption~(D1), the matrix $\bJ$ is Hurwitz, by \Cref{proposition3}. Therefore, Assumption~(A5) holds.

\paragraph{Verification of (A6).}
The remainder term is
\[
\br_t=
\begin{pmatrix}
\br_t^{(x)}\\
\bzero
\end{pmatrix},
\qquad
\br_t^{(x)}
=
\frac{\bg_t}{\sqrt{\bv_t}+\varepsilon}
-
\frac{\bg_t}{\sqrt{\bv_{t+1}}+\varepsilon}.
\]
We decompose
\[
\br_t^{(x)}
=
\left(
\frac{\nabla f(\bx_t)}{\sqrt{\bv_t}+\varepsilon}
-
\frac{\nabla f(\bx_t)}{\sqrt{\bv_{t+1}}+\varepsilon}
\right) + \left(\frac{\bxi_t^{(g)}}{\sqrt{\bv_{t}}+\varepsilon}
-
\frac{\bxi_t^{(g)}}{\sqrt{\bv_{t+1}}+\varepsilon}\right).
\]
Using the local Lipschitz continuity of $\bv\mapsto (\sqrt{\bv}+\varepsilon)^{-1}$, together with
\[
\bv_{t+1}-\bv_t
=
\beta_t\bigl(\bq(\bx_t)-\bv_t\bigr)
+
\beta_t\Bigl(2\nabla f(\bx_t)\odot \bxi_t^{(g)}+\hat{\bxi}_t^{(g)}\Bigr),
\]
one obtains
\[
\mathbb E\!\left[\|\bv_{t+1}-\bv_t\|^2\mid \mathcal F_t\right]
=
O(\gamma_t^2),
\qquad \text{a.s.}
\]
Hence the first term in the above decomposition of $\br_t^{(x)}$ is of order $O(\gamma_t)$ in conditional $L^2$. The second term is controlled by the bounded fourth moments of $\bxi_t^{(g)}$. Therefore,
\[
\mathbb E\!\left[\|\br_t\|^2\mid \mathcal F_t\right]
=
O(\gamma_t^2),
\qquad \text{a.s.}
\]
which implies the required negligibility of the remainder term. Thus Assumption~(A6) holds.

Combining the above verifications, we conclude that Assumptions~(A1)--(A6) in \Cref{assumption1} are satisfied. Therefore, \Cref{theorem1} applies directly to the AdaFactor iterates, and
\[
\gamma_t^{-1/2}(\bz_t-\bz^\star)
\;\xrightarrow{d}\;
\mathcal N(\bzero,\bSig),
\]
\begin{equation*}
    \sqrt{t}\left(\Bar{\bz}_t - \bz^{\star}\right)\;\xrightarrow{d}\;
\mathcal N(\bzero,\bJ^{-1}\bQ\bJ^{-\top}),
\end{equation*}
where $\bSig$ is the unique positive semidefinite solution to
\[
\bJ\bSig+\bSig\bJ^\top+\bQ=\bzero.
\]

\newpage

\section{Proofs of Main Lemmas}
\label{appendix_proof_main_lemmas}
\subsection{Proof of \Cref{lemma4}}
We aim to analyze the fluctuation term
\begin{equation*}
\mJ_t
:=
\sum_{k=1}^{t}\Phi_{t,k+1}\frac{\gamma_k}{\sqrt{\gamma_t}}\bxi_{k},
\end{equation*}
where
\[
\Phi_{t,k}:=\prod_{j=k}^{t}(\bI+\gamma_j\bJ).
\]

To prove the asymptotic normality of $\mJ_t$, we apply the Cram\'er--Wold theorem. Fix any $\bd \in \mathbb{R}^d$, and define
\begin{equation}
\label{eq3}
\Delta_{t,k}
:=
\bd^\top \Phi_{t,k+1}\frac{\gamma_k}{\sqrt{\gamma_t}}\bxi_{k},
\qquad
1\le k\le t.
\end{equation}
Then
\begin{equation*}
\bd^\top \mJ_t = \sum_{k=1}^{t}\Delta_{t,k}.
\end{equation*}

Since $\Phi_{t,k+1}\gamma_k/\sqrt{\gamma_t}$ is $\mathcal{F}_k$-measurable and
\[
\mathbb{E}[\bxi_{k}\mid\mathcal{F}_k]=\bzero,
\]
it follows that
\[
\mathbb{E}[\Delta_{t,k}\mid \mathcal{F}_k]=0.
\]
Hence, for each fixed $t$, $(\Delta_{t,k})_{1\le k\le t}$ is a martingale difference array.

We now verify the two standard conditions of the martingale central limit theorem.

\medskip
\noindent
\textbf{Step 1: conditional variance convergence.}
Define
\begin{equation*}
V_t(\bd)
:=
\sum_{k=1}^{t}\mathbb{E}\big[\Delta_{t,k}^2\mid\mathcal{F}_k\big].
\end{equation*}
Using the definition of $\Delta_{t,k}$ in \eqref{eq3}, we obtain
\begin{equation*}
V_t(\bd)
=
\sum_{k=1}^{t}
\frac{\gamma_k^2}{\gamma_t}\,
\bd^\top
\Phi_{t,k+1}
\mathbb{E}\big[\bxi_{k}\bxi_{k}^\top\mid\mathcal{F}_k\big]
\Phi_{t,k+1}^\top
\bd.
\end{equation*}
Note that
\begin{equation*}
    \begin{split}
        & \left\|V_t(\bd)
-
\sum_{k=1}^{t}
\frac{\gamma_k^2}{\gamma_t}\,
\bd^\top
\Phi_{t,k+1}\bQ\Phi_{t,k+1}^\top
\bd \right\| \\
& \leq \gamma_t^{-1} \left\| \bd\right\|^2 \sum_{k=1}^{t} \left\|\Phi_{t,k+1} \right\|^2\gamma_k^2 \left\| \mathbb{E}\big[\bxi_{k}\bxi_{k}^\top\mid\mathcal{F}_k\big] - \bQ\right\|\\
& \leq C_0^2 \gamma_t^{-1} \left\| \bd\right\|^2 \sum_{k=1}^{t} \prod_{j=k+1}^{t}\left(1-c\gamma_j\right)^2\gamma_k^2 \left\| \mathbb{E}\big[\bxi_{k}\bxi_{k}^\top\mid\mathcal{F}_k\big] - \bQ\right\| \to 0,
    \end{split}
\end{equation*}
where the last inequality comes from \eqref{eq1} and the convergence is implied by \Cref{lemma1} and \Cref{assumption1} (A4).
Hence,
\begin{equation*}
V_t(\bd)
-
\sum_{k=1}^{t}
\frac{\gamma_k^2}{\gamma_t}\,
\bd^\top
\Phi_{t,k+1}\bQ\Phi_{t,k+1}^\top
\bd
\;\xrightarrow[]{}\; 0,
\end{equation*}

Let
\begin{equation*}
\bU_t
:=
\sum_{k=1}^{t}
\frac{\gamma_k^2}{\gamma_t}\,
\Phi_{t,k+1}\bQ\Phi_{t,k+1}^\top.
\end{equation*}
Then, by the standard discrete Lyapunov argument in \Cref{lemma3},
\[
\bU_t \to \bSig,
\]
where $\bSig$ is the unique positive semidefinite solution of
\begin{equation*}
\bJ \bSig + \bSig \bJ^\top + \bQ = \bzero.
\end{equation*}
Therefore,
\begin{equation*}
V_t(\bd) \to \bd^\top \bSig \bd.
\end{equation*}

\medskip
\noindent
\textbf{Step 2: Lindeberg condition.}
For any $\bar\varepsilon>0$, consider
\begin{equation}
L_t(\bar\varepsilon,\bd)
:=
\sum_{k=1}^{t}
\mathbb{E}\Big[
\Delta_{t,k}^2
\mathbf{1}_{\{|\Delta_{t,k}|>\bar\varepsilon\}}
\;\Big|\;
\mathcal{F}_k
\Big].
\end{equation}
By H\"older's inequality and the elementary bound
\[
\mathbf{1}_{\{|x|>\bar\varepsilon\}}
\le
\frac{|x|^\delta}{\bar\varepsilon^\delta},
\qquad \delta>0,
\]
we have
\[
\Delta_{t,k}^2\mathbf{1}_{\{|\Delta_{t,k}|>\bar\varepsilon\}}
\le
\bar\varepsilon^{-\delta}|\Delta_{t,k}|^{2+\delta}.
\]
Therefore,
\begin{equation*}
L_t(\bar\varepsilon,\bd)
\le
\bar\varepsilon^{-\delta}
\sum_{k=1}^{t}
\mathbb{E}\big[
|\Delta_{t,k}|^{2+\delta}\mid\mathcal{F}_k
\big].
\end{equation*}
Since
\[
|\Delta_{t,k}|
\le
\|\bd\|\,
\left\|\Phi_{t,k+1}\right\|
\frac{\gamma_k}{\sqrt{\gamma_t}}
\|\bxi_{k}\|,
\]
we obtain
\begin{equation*}
\mathbb{E}\big[
|\Delta_{t,k}|^{2+\delta}\mid\mathcal{F}_k
\big]
\le
\left\| \bd\right\|^{2+\delta}
\left\|\Phi_{t,k+1}\right\|^{2+\delta}
\left(\frac{\gamma_k}{\sqrt{\gamma_t}}\right)^{2+\delta}
\mathbb{E}\big[\|\bxi_{k}\|^{2+\delta}\mid\mathcal{F}_k\big].
\end{equation*}
By \Cref{assumption1} (A4),
\[
\sup_k \mathbb{E}\big[\|\bxi_{k}\|^{2+\delta}\mid\mathcal{F}_k\big] < \infty
\qquad \text{a.s.}
\]
Hence
\begin{equation*}
L_t(\bar\varepsilon,\bd)
\le
C
\sum_{k=1}^{t}
\left\|\Phi_{t,k+1}\right\|^{2+\delta}
\left(\frac{\gamma_k}{\sqrt{\gamma_t}}\right)^{2+\delta},
\end{equation*}
for some constants $C>0$.
Using the exponential stability estimate
\[
\|\Phi_{t,k+1}\|
\le
C_0\!\prod_{j=k+1}^{t}\left(1-c\gamma_j\right),
\]
and \Cref{lemma1},
it follows that
\[
L_t(\bar\varepsilon,\bd) = \mathcal{O}\left(\gamma_{t}^{\delta/2}\right)\to 0.
\]
Thus, the Lindeberg condition holds.

\medskip
\noindent
By the martingale central limit theorem for triangular arrays, we conclude that
\begin{equation*}
\bd^\top \mJ_t
\xrightarrow{d}
\mathcal{N}(0,\bd^\top \bSig \bd).
\end{equation*}
Since $\bd\in\mathbb{R}^d$ is arbitrary, the Cram\'er--Wold theorem yields
\begin{equation*}
\mJ_t
\xrightarrow{d}
\mathcal{N}(\bzero,\bSig).
\end{equation*}

\subsection{Proof of \Cref{lemma5}}
Since $\bJ$ is Hurwitz, there exists a symmetric positive definite matrix $\bP$ such that
\begin{equation*}
\bJ^\top \bP + \bP \bJ = -\bI.
\end{equation*}
Define the function
\begin{equation*}
V_t := \be_t^\top \bP \be_t.
\end{equation*}
Since $\bP \succ \bzero$, there exist constants $0 < c_1 \le c_2 < \infty$ such that
\begin{equation*}
c_1 \|\be_t\|^2 \le V_t \le c_2 \|\be_t\|^2.
\end{equation*}

By the definition of $\rho(\be)$ and the mean-value theorem, there exists matrix $\bH_t$ such that $\rho(\be_t) = \bH_t\be_t$ and $\bH_t \to \bzero$ as $t \to \infty$, almost surely. 
We first introduce the following conditions:
\begin{enumerate}
    \item $\|\bz_{t}\| \leq M$,

    \item $\left\| \bH_t\right\| \leq M$,

\end{enumerate}
for some constants $M > 0$.
According to \Cref{assumption1}, these conditions hold for all sufficiently large $t$, almost surely. In other words, we define the stopping time
\begin{equation*}
    \tau_{t_0} = \inf_{j}\{j \geq t_0:~\text{any of the above conditions does not hold at j-th iteration}\}.
\end{equation*}
Then there exists a (potentially random) finite time $t_0 < \infty$ such that
$\tau_{t_0} = \infty$, almost surely. We denote by $\bm{1}_{\mathcal{A}}$ the indicator function of an event $\mathcal{A}$, which equals $1$ if the event holds and $0$ otherwise.

We fix a sufficiently large $t_0 >0$.
We define $\bG_t = \bJ + \bH_t$, and consider the recursion
\[
\be_{t+1} = \be_t + \gamma_t \bG_t \be_t + \gamma_t \bxi_{t} + \gamma_t \br_t,
\]
we have
\begin{align*}
V_{t+1}
&=
(\be_t + \gamma_t \bG_t \be_t + \gamma_t \bxi_{t} + \gamma_t \br_t)^\top
\bP
(\be_t + \gamma_t \bG_t \be_t + \gamma_t \bxi_{t} + \gamma_t \br_t) \\
&=
\be_t^\top \bP \be_t
+ 2\gamma_t \be_t^\top \bP \bG_t \be_t
+ 2\gamma_t \be_t^\top \bP \bxi_{t} + 2 \gamma_t \be_t^\top \bP \br_{t}\\
& 
\hspace{1em} + \gamma_t^2 (\bG_t \be_t + \bxi_{t} + \br_t)^\top \bP (\bG_t \be_t + \bxi_{t} + \br_t).
\end{align*}
Taking conditional expectation with respect to $\mathcal{F}_t$, and using
$\mathbb{E}[\bxi_{t}\mid \mathcal{F}_t] = \bzero$,
we get
\begin{equation*}
\begin{split}
    \mathbb{E}[V_{t+1}\bone_{\tau_{t_0} \geq t+1}\mid \mathcal{F}_t]
& \leq
V_t \bone_{\tau_{t_0} \geq t}
+ 2\gamma_t \be_t^\top \bP \bG_t \be_t \bone_{\tau_{t_0} \geq t} + 2\gamma_t \be_t^\top \bP \mathbb{E}\big[\br_t\bone_{\tau_{t_0} \geq t} \mid \mathcal{F}_t
\big]\\
& \hspace{1em} + \gamma_t^2
\mathbb{E}\big[
(\bG_t \be_t + \bxi_{t} + \br_t)^\top \bP (\bG_t \be_t + \bxi_{t} + \br_t)\bone_{\tau_{t_0} \geq t} 
\mid \mathcal{F}_t
\big].
\end{split}
\end{equation*}

We now estimate each term separately.

\medskip
\noindent
\textbf{Step 1: estimate of the drift term.}

Since $\bG_t = \bJ + \bH_t$, we have
\begin{align*}
2 \be_t^\top \bP \bG_t \be_t
&=
2 \be_t^\top \bP \bJ \be_t
+
2 \be_t^\top \bP \bH_t \be_t.
\end{align*}
For the first term,
\begin{align*}
2 \be_t^\top \bP \bJ \be_t
=
\be_t^\top (\bJ^\top \bP + \bP \bJ)\be_t =
- \left\|\be_t\right\|^2.
\end{align*}
For the perturbation term,
\begin{equation*}
\left|2 \be_t^\top \bP \bH_t \be_t\right|
\le
2 \|\bP\|\,\|\bH_t\|\,\|\be_t\|^2.
\end{equation*}
Hence
\begin{equation*}
    2 \be_t^\top \bP \bG_t \be_t \leq
- \|\be_t\|^2
+
2\|\bP\|\,\|\bH_t\|\,\|\be_t\|^2 \leq -\frac{1}{2}\|\be_t\|^2
\end{equation*}
Using the norm equivalence $V_t \le c_2 \|\be_t\|^2$, we further obtain
\begin{equation}
\label{eq4}
2 \be_t^\top \bP \bG_t \be_t \bone_{\tau_{t_0} \geq t}
\le
- \frac{1}{2c_2}V_t \bone_{\tau_{t_0} \geq t} := -c_0V_t \bone_{\tau_{t_0} \geq t},
\end{equation}
where we denote $c_0 = \frac{1}{2c_2}$

\medskip
\noindent
\textbf{Step 2: estimate of the quadratic term.}

Using
\[
(\bu+\bv + \bw)^\top \bP (\bu+\bv + \bw)
\le
3 \left\| \bP\right\|\left(\left\| \bu\right\|^2 + \left\| \bv\right\|^2 + \left\| \bw\right\|^2\right),
\]
we get
\begin{align*}
(\bG_t \be_t + \bxi_{t} + \br_t)^\top \bP (\bG_t \be_t + \bxi_{t} + \br_t)
\le
3 \|\bP\|\,\left(\|\bG_t \be_t\|^2 + \|\bxi_{t}\|^2+ \left\| \br_t\right\|^2\right).
\end{align*}
Hence
\begin{align*}
& \mathbb{E}\big[
(\bG_t \be_t + \bxi_{t} + \br_t)^\top \bP (\bG_t \be_t + \bxi_{t} + \br_t)
\mid \mathcal{F}_t
\big]\\
&
\leq
3 \|\bP\| \left(\|\bG_t\|^2 \|\be_t\|^2
+ \mathbb{E}\big[\|\bxi_{t}\|^2 \mid \mathcal{F}_t\big] + \mathbb{E}\big[\|\br_t\|^2 \mid \mathcal{F}_t\big] \right).
\end{align*}
By \Cref{assumption1} (A3) and (A6),
\[
\mathbb{E}[\|\bxi_{t}\|^2\mid \mathcal{F}_t] \le C_\xi,\qquad \mathbb{E}[\|\br_{t}\|^2\mid \mathcal{F}_t] \le C_r \gamma_t^2,
\]
for some $C_{\xi}, C_r > 0$.
Therefore, for some constant $C>0$,
\begin{equation*}
\mathbb{E}\big[
(\bG_t \be_t + \bxi_{t} + \br_t)^\top \bP (\bG_t \be_t + \bxi_{t} + \br_t)
\mid \mathcal{F}_t
\big]
\le
C(\|\be_t\|^2+1).
\end{equation*}
Using again the equivalence between $V_t$ and $\|\be_t\|^2$, we obtain
\begin{equation}
\label{eq5}
\mathbb{E}\big[
(\bG_t \be_t + \bxi_{t} + \br_t)^\top \bP (\bG_t \be_t + \bxi_{t} + \br_t)
\mid \mathcal{F}_t
\big]
\le
C(V_t+1).
\end{equation}

\medskip
\noindent
\textbf{Step 3: one-step recursive inequality.}

Combining the previous bounds in \eqref{eq4} and \eqref{eq5} yields
\begin{align*}
\mathbb{E}[V_{t+1} \bone_{\tau_{t_0} \geq t+1}\mid \mathcal{F}_t]
&\le
V_t \bone_{\tau_{t_0} \geq t} - c_0 \gamma_t V_t\bone_{\tau_{t_0} \geq t} + C \gamma_t^2 (V_t+1) \bone_{\tau_{t_0} \geq t}.
\end{align*}
Since $\gamma_t \to 0$, for all sufficiently large $t$,
\[
C\gamma_t^2 V_t \le \frac{c_0}{2}\gamma_t V_t.
\]
Thus,
\begin{equation*}
\mathbb{E}[V_{t+1}\bone_{\tau_{t_0} \geq t+1}\mid \mathcal{F}_t]
\le
(1-c_1\gamma_t)V_t \bone_{\tau_{t_0} \geq t} + C\gamma_t^2,
\end{equation*}
for some constant $c_1>0$.

Taking full expectation, we obtain
\begin{equation*}
\mathbb{E}[V_{t+1} \bone_{\tau_{t_0} \geq t+1}]
\le
(1-c_1\gamma_t)\mathbb{E}[V_t \bone_{\tau_{t_0} \geq t}] + C\gamma_t^2.
\end{equation*}

\medskip
\noindent
\textbf{Step 4: the convergence rates for $\be_t$ and $\rho(\be_t)$.}

The scalar recursion
\[
u_{t+1}\le (1-c_1\gamma_t)u_t + C\gamma_t^2
\]
implies
\[
u_t = \mathcal{O}(\gamma_t).
\]
Applying this with $u_t=\mathbb{E}[V_t \bone_{\tau_{t_0} \geq t}]$, we conclude that
\begin{equation*}
\mathbb{E}[V_t \bone_{\tau_{t_0} \geq t}] = \mathcal{O}(\gamma_t).
\end{equation*}
Finally, since $c_1 \|\be_t\|^2 \le V_t$, we obtain
\begin{equation*}
\mathbb{E}[\|\be_t\|^2 \bone_{\tau_{t_0} \geq t}] = \mathcal{O}(\gamma_t).
\end{equation*}
Therefore, $\|\be_t\| \bone_{\tau_{t_0} \geq t} = \mathcal{O}_{p}\left(\gamma_t^{1/2}\right)$ and $\left\|\rho(\be_t)\right\| \bone_{\tau_{t_0} \geq t} = o_{p}\left(\gamma_t^{1/2}\right)$. 

\medskip
\noindent
\textbf{Step 5: the convergence rate for the weighted term.}

Finally, we are ready to bound
\begin{equation*}
\begin{split}
        & \left\|\sum_{k=1}^{t} \Phi_{t,k+1} \gamma_k \rho(\be_k) \right\|\\
        & \leq \sum_{k=0}^{t} \left\| \Phi_{t,k+1}\right\| \gamma_k \left\| \rho(\be_k) \right\| \\
        & \leq C_0 \sum_{k=0}^{t} \prod_{j=k+1}^{t}\left(1-c\gamma_j\right) \gamma_k \left\| \rho(\be_k) \right\|.
\end{split}
\end{equation*}
Since $\left\|\rho(\be_t)\right\| \bone_{\tau_{t_0} \geq t} = o_{p}\left(\gamma_t^{1/2}\right)$, we have $\left\|\sum_{k=1}^{t} \Phi_{t,k+1} \gamma_k \rho(\be_k) \right\| = o_{p}\left(\gamma_t^{1/2}\right)$, by Lemma C.3 in \cite{na2025derivative}. We also note that $\left\|\br_t\right\| = \mathcal{O}_{p}(\gamma_t) = o_p(\gamma_t^{1/2})$ implies $\left\|\sum_{k=1}^{t} \Phi_{t,k+1} \gamma_k \br_k \right\| = o_{p}\left(\gamma_t^{1/2}\right)$, which completes the proof.

\subsection{Proof of \Cref{proposition2}}
\label{appendix_prop2}

We prove the two statements separately.

\paragraph{Characterization of the limiting point.}
Let $\bz^\star=(\bx^\star,\bmm^\star,\bv^\star)$ be a solution point of the AdamW mean-field dynamics. By definition,
\[
F(\bz^\star)=\bzero,
\]
that is,
\begin{align}
-\frac{\bmm^\star}{\sqrt{\bv^\star}+\varepsilon}-\lambda \bx^\star &= \bzero, \label{eq:adamw_eq_block1}\\
\alpha\bigl(\nabla f(\bx^\star)-\bmm^\star\bigr) &= \bzero, \label{eq:adamw_eq_block2}\\
\beta\bigl(\bq(\bx^\star)-\bv^\star\bigr) &= \bzero. \label{eq:adamw_eq_block3}
\end{align}
Since $\alpha,\beta>0$, \eqref{eq:adamw_eq_block2} and \eqref{eq:adamw_eq_block3} imply
\[
\bmm^\star=\nabla f(\bx^\star),
\qquad
\bv^\star=\bq(\bx^\star)=\bq^\star.
\]
Substituting these identities into \eqref{eq:adamw_eq_block1}, we obtain
\[
-\frac{\nabla f(\bx^\star)}{\sqrt{\bv^\star}+\varepsilon}-\lambda \bx^\star=\bzero,
\]
or equivalently,
\[
\frac{\nabla f(\bx^\star)}{\sqrt{\bv^\star}+\varepsilon}+\lambda \bx^\star=\bzero.
\]
This proves the first statement.

\paragraph{Computation of the Jacobian matrix.}
Write
\[
F(\bz)=
\begin{pmatrix}
F_1(\bx,\bmm,\bv)\\[1mm]
F_2(\bx,\bmm,\bv)\\[1mm]
F_3(\bx,\bmm,\bv)
\end{pmatrix},
\]
where
\begin{align}
F_1(\bx,\bmm,\bv) &= -\frac{\bmm}{\sqrt{\bv}+\varepsilon}-\lambda \bx,\\
F_2(\bx,\bmm,\bv) &= \alpha\bigl(\nabla f(\bx)-\bmm\bigr),\\
F_3(\bx,\bmm,\bv) &= \beta\bigl(\bq(\bx)-\bv\bigr).
\end{align}

We compute the block derivatives of $F$ at $\bz^\star=(\bx^\star,\bmm^\star,\bv^\star)$.

\medskip
\noindent
\emph{Derivative of the first block.}
Since
\[
F_1(\bx,\bmm,\bv)
=
-\lambda \bx
-
\operatorname{diag}\!\left(\frac{1}{\sqrt{\bv}+\varepsilon}\right)\bmm,
\]
we have
\[
\frac{\partial F_1}{\partial \bx}(\bz^\star)=-\lambda \bI_d.
\]
Differentiating with respect to $\bmm$ gives
\[
\frac{\partial F_1}{\partial \bmm}(\bz^\star)
=
-\operatorname{diag}\!\left(
\frac{1}{\sqrt{\bv^\star}+\varepsilon}
\right)
=: -\bD^\star.
\]
Next, for each component,
\[
[F_1]_i(\bx,\bmm,\bv)
=
-\frac{m_i}{\sqrt{v_i}+\varepsilon}-\lambda x_i.
\]
Hence, for $v_i^\star>0$,
\[
\frac{\partial [F_1]_i}{\partial v_i}(\bz^\star)
=
\frac{m_i^\star}{2\sqrt{v_i^\star}(\sqrt{v_i^\star}+\varepsilon)^2}.
\]
Therefore,
\[
\frac{\partial F_1}{\partial \bv}(\bz^\star)
=
\operatorname{diag}\!\left(
\frac{\bmm^\star}{2\sqrt{\bv^\star}(\sqrt{\bv^\star}+\varepsilon)^{\odot 2}}
\right)
=: \bC^\star,
\]
where the division and multiplication are conducted elementwise for vectors in $\frac{\bmm^\star}{2\sqrt{\bv^\star}(\sqrt{\bv^\star}+\varepsilon)^{\odot 2}}$, i.e., $\left[\frac{\bmm^\star}{2\sqrt{\bv^\star}(\sqrt{\bv^\star}+\varepsilon)^{\odot 2}}\right]_{i} = \frac{m_i^\star}{2 \sqrt{v_i^\star}(\sqrt{v_i^\star} + \varepsilon)^2}$.

\medskip
\noindent
\emph{Derivative of the second block.}
Since
\[
F_2(\bx,\bmm,\bv)=\alpha\bigl(\nabla f(\bx)-\bmm\bigr),
\]
we immediately obtain
\[
\frac{\partial F_2}{\partial \bx}(\bz^\star)=\alpha \nabla^2 f(\bx^\star),
\qquad
\frac{\partial F_2}{\partial \bmm}(\bz^\star)=-\alpha \bI_d,
\qquad
\frac{\partial F_2}{\partial \bv}(\bz^\star)=\bzero.
\]

\medskip
\noindent
\emph{Derivative of the third block.}
For
\[
F_3(\bx,\bmm,\bv)=\beta\bigl(\bq(\bx)-\bv\bigr),
\]
we have
\[
\frac{\partial F_3}{\partial \bmm}(\bz^\star)=\bzero,
\qquad
\frac{\partial F_3}{\partial \bv}(\bz^\star)=-\beta \bI_d,
\]
and
\[
\frac{\partial F_3}{\partial \bx}(\bz^\star) = \beta \nabla \bq(\bx^{\star}).
\]

Collecting all block derivatives, we conclude that
\[
\bJ=DF(\bz^\star)=
\begin{pmatrix}
-\lambda \bI_d & -\bD^\star & \bC^\star\\[2mm]
\alpha \nabla^2 f(\bx^\star) & -\alpha \bI_d & \bzero\\[2mm]
\beta\nabla \bq(\bx^\star)
& \bzero & -\beta \bI_d
\end{pmatrix}.
\]

\paragraph{Hurwitz property.}
Now we prove the Hurwitzness of the Jacobian matrix $\bJ$, under certain conditions.
We decompose $\bJ$ as
\[
\bJ=\bJ_0+\bE,
\]
where
\begin{equation}
\label{eq:J0_adamw}
\bJ_0=
\begin{pmatrix}
-\lambda \bI_d & -\bD^\star & \bzero\\[1mm]
\alpha \nabla^2 f(\bx^\star) & -\alpha \bI_d & \bzero\\[1mm]
\beta \nabla \bq(\bx^\star) & \bzero & -\beta \bI_d
\end{pmatrix},
\end{equation}
and
\begin{equation}
\label{eq:E_adamw}
\bE=
\begin{pmatrix}
\bzero & \bzero & \bC^\star\\[1mm]
\bzero & \bzero & \bzero\\[1mm]
\bzero & \bzero & \bzero
\end{pmatrix}.
\end{equation}

We first show that $\bJ_0$ is Hurwitz. Observe that $\bJ_0$ is block lower triangular, so its eigenvalues are the union of those of
\[
\bM=
\begin{pmatrix}
-\lambda \bI_d & -\bD^\star\\
\alpha \nabla^2 f(\bx^\star) & -\alpha \bI_d
\end{pmatrix}
\]
and those of $-\beta \bI_d$. Since $\beta>0$, the block $-\beta \bI_d$ is clearly Hurwitz. It therefore suffices to show that $\bM$ is Hurwitz. The proof is similar to that for \Cref{proposition1}.

Let $\mu\in\mathbb C$ be an eigenvalue of $\bM$ with corresponding eigenvector $(\bu^\top,\bv^\top)^\top\neq \bzero$. Then
\[
\begin{pmatrix}
-\lambda \bI_d & -\bD^\star\\
\alpha \nabla^2 f(\bx^\star) & -\alpha \bI_d
\end{pmatrix}
\begin{pmatrix}
\bu\\
\bv
\end{pmatrix}
=
\mu
\begin{pmatrix}
\bu\\
\bv
\end{pmatrix},
\]
which is equivalent to
\begin{align}
-(\lambda+\mu)\bu-\bD^\star\bv&=\bzero, \label{eq:adamw_hurwitz_1}\\
\alpha \nabla^2 f(\bx^\star)\bu-(\alpha+\mu)\bv&=\bzero. \label{eq:adamw_hurwitz_2}
\end{align}
From \eqref{eq:adamw_hurwitz_2}, we obtain
\[
\bv=\frac{\alpha}{\alpha+\mu}\nabla^2 f(\bx^\star)\bu.
\]
Substituting this into \eqref{eq:adamw_hurwitz_1}, we get
\[
(\mu+\lambda)\bu+\frac{\alpha}{\alpha+\mu}\bD^\star \nabla^2 f(\bx^\star)\bu=\bzero.
\]
Equivalently,
\[
(\mu+\lambda)(\mu+\alpha)\bu+\alpha \bD^\star \nabla^2 f(\bx^\star)\bu=\bzero.
\]

Now define
\[
\bH^\star:=\nabla^2 f(\bx^\star),
\qquad
\bB:=(\bD^\star)^{1/2}\bH^\star(\bD^\star)^{1/2}.
\]
Since $\bD^\star\succ \bzero$ and $\bH^\star\succ \bzero$, we have $\bB\succ \bzero$. Hence every eigenvalue $\theta$ of $\bB$ satisfies $\theta>0$. It follows that every eigenvalue $\mu$ of $\bM$ must satisfy
\[
(\mu+\lambda)(\mu+\alpha)+\alpha\theta=0,
\]
that is,
\[
\mu^2+(\alpha+\lambda)\mu+\alpha(\lambda+\theta)=0,
\qquad \theta>0.
\]
Since all coefficients are strictly positive, both roots have strictly negative real parts. Therefore, every eigenvalue of $\bM$ lies in the open left half-plane, and thus $\bM$ is Hurwitz. Consequently, $\bJ_0$ is Hurwitz.

Since $\bJ_0$ is Hurwitz, by the Lyapunov theorem there exists a symmetric positive definite matrix $\bP\succ \bzero$ such that
\[
\bJ_0^\top \bP+\bP\bJ_0=-\bI.
\]
For the full matrix $\bJ=\bJ_0+\bE$, we have
\[
\bJ^\top \bP+\bP\bJ
=
-\bI+\bE^\top \bP+\bP\bE.
\]
Thus, if
\[
\|\bE^\top \bP+\bP\bE\|<1,
\]
then $\bJ^\top \bP+\bP\bJ\prec \bzero$, and therefore $\bJ$ is Hurwitz.

Using the submultiplicative property of the operator norm,
\[
\|\bE^\top \bP+\bP\bE\|
\le
2\|\bP\|\,\|\bE\|
=
2\|\bP\|\,\|\bC^\star\|.
\]
Hence it suffices to require
\[
2\|\bP\|\,\|\bC^\star\|<1.
\]

It remains to estimate $\|\bC^\star\|$. Since
\[
\bmm^\star=\nabla f(\bx^\star)
\]
and $v_i^\star\ge \underline v>0$, we have
\[
|C_{ii}^\star|
=
\left|
\frac{m_i^\star}{2\sqrt{v_i^\star}(\sqrt{v_i^\star}+\varepsilon)^2}
\right|
\le
\frac{\|\nabla f(\bx^\star)\|_\infty}{2\sqrt{\underline v}(\sqrt{\underline v}+\varepsilon)^2}.
\]
Therefore,
\[
\|\bC^\star\|
\le
\frac{\|\nabla f(\bx^\star)\|_\infty}{2\sqrt{\underline v}(\sqrt{\underline v}+\varepsilon)^2}.
\]
It follows that if
\[
\|\nabla f(\bx^\star)\|_\infty
<
\frac{\sqrt{\underline v}(\sqrt{\underline v}+\varepsilon)^2}{\|\bP\|},
\]
then
\[
2\|\bP\|\,\|\bC^\star\|<1,
\]
and hence $\bJ$ is Hurwitz.

\subsection{Proof of \Cref{prop_batch_mean}}
\label{appendix_batch_mean}

We provide a proof of the consistency of the batch-means estimator used in
\Cref{prop_batch_mean}. The argument follows the standard batch-means
variance-estimation principle for dependent stochastic processes. We state the
regularity conditions explicitly for completeness.

Let $s_t=s_{\ell,j}(t)$ be the empirical score process for a fixed LoRA
component and let $s^\star=s_{\ell,j}^{\star}$. Define
\[
    X_t := s_t-s^\star .
\]
We assume that the centered score process satisfies a functional central limit
theorem with long-run variance $\bar\sigma_{\ell,j}^{2}$, that is,
\[
    \frac{1}{\sqrt{t}}\sum_{\tau=1}^{\lfloor ct\rfloor} X_\tau
    \Rightarrow
    \bar\sigma_{\ell,j} B(c),
    \qquad c\in[0,1],
\]
where $B(\cdot)$ is a standard Brownian motion. In addition, we assume the
usual moment and weak-dependence conditions for non-overlapping batch-means
consistency. Equivalently, for batch length $b\to\infty$, number of batches
$M\to\infty$, and $t=Mb$ with $b/t\to0$, the block variables
\[
    Y_{m,b}
    :=
    \frac{1}{\sqrt b}
    \sum_{\tau=(m-1)b+1}^{mb} X_\tau,
    \qquad m=1,\ldots,M,
\]
satisfy
\[
    \frac1M\sum_{m=1}^M Y_{m,b}^2
    \xrightarrow{p}
    \bar\sigma_{\ell,j}^{2},
    \qquad
    \frac1M\sum_{m=1}^M Y_{m,b}
    \xrightarrow{p}
    0 .
\]
These conditions hold under standard mixing and moment assumptions used for
batch-means estimators of long-run variances.

Recall that
\[
    \bar s_{m,b}
    =
    \frac1b\sum_{\tau=(m-1)b+1}^{mb}s_\tau,
    \qquad
    \bar s_t
    =
    \frac1t\sum_{\tau=1}^{t}s_\tau .
\]
Since subtracting the constant $s^\star$ does not change centered batch
differences, we have
\[
    \bar s_{m,b}-\bar s_t
    =
    \bar X_{m,b}-\bar X_t,
\]
where
\[
    \bar X_{m,b}
    =
    \frac1b\sum_{\tau=(m-1)b+1}^{mb}X_\tau,
    \qquad
    \bar X_t
    =
    \frac1t\sum_{\tau=1}^{t}X_\tau .
\]
By the definition of $Y_{m,b}$,
\[
    \bar X_{m,b}
    =
    \frac{1}{\sqrt b}Y_{m,b}.
\]
Moreover,
\[
    \bar X_t
    =
    \frac1M\sum_{m=1}^M \bar X_{m,b}
    =
    \frac{1}{\sqrt b}
    \left(
        \frac1M\sum_{m=1}^M Y_{m,b}
    \right).
\]
Therefore,
\[
    \bar s_{m,b}-\bar s_t
    =
    \frac{1}{\sqrt b}
    \left(
        Y_{m,b}
        -
        \bar Y_{M,b}
    \right),
    \qquad
    \bar Y_{M,b}
    :=
    \frac1M\sum_{m=1}^M Y_{m,b}.
\]
Substituting this expression into the batch-means estimator gives
\[
\begin{aligned}
    \widehat{\bar\sigma}_{\ell,j}^{2}
    &=
    \frac{b}{M-1}
    \sum_{m=1}^M
    \left(
        \bar s_{m,b}-\bar s_t
    \right)^2                                                \\
    &=
    \frac{1}{M-1}
    \sum_{m=1}^M
    \left(
        Y_{m,b}-\bar Y_{M,b}
    \right)^2 .
\end{aligned}
\]
Expanding the centered empirical second moment, we obtain
\[
\begin{aligned}
    \widehat{\bar\sigma}_{\ell,j}^{2}
    &=
    \frac{M}{M-1}
    \left[
        \frac1M\sum_{m=1}^M Y_{m,b}^2
        -
        \left(
            \frac1M\sum_{m=1}^M Y_{m,b}
        \right)^2
    \right].
\end{aligned}
\]
By the batch-means regularity conditions,
\[
    \frac1M\sum_{m=1}^M Y_{m,b}^2
    \xrightarrow{p}
    \bar\sigma_{\ell,j}^{2},
    \qquad
    \left(
        \frac1M\sum_{m=1}^M Y_{m,b}
    \right)^2
    \xrightarrow{p}
    0 .
\]
Since $M/(M-1)\to1$, Slutsky's theorem yields
\[
    \widehat{\bar\sigma}_{\ell,j}^{2}
    \xrightarrow{p}
    \bar\sigma_{\ell,j}^{2}.
\]
This proves the consistency of the batch-means variance estimator.

\subsection{Proof for \Cref{proposition1}}

We prove the two statements separately.

\paragraph{Characterization of the limiting point.}
Recall that the mean-field mapping for Adam is given by
\begin{equation*}
F(\bz)=
\begin{pmatrix}
-\dfrac{\bmm}{\sqrt{\bv}+\varepsilon}\\[1mm]
\alpha\bigl(\nabla f(\bx)-\bmm\bigr)\\[1mm]
\beta\bigl(\bq(\bx)-\bv\bigr)
\end{pmatrix},
\qquad
\bz=(\bx,\bmm,\bv)\in\mathbb{R}^{3d},
\end{equation*}
where the division is understood componentwise.

A solution $\bz^\star=(\bx^\star,\bmm^\star,\bv^\star)$ of the mean-field function satisfies
\[
F(\bz^\star)=\bzero.
\]
Hence its three blocks must satisfy
\begin{align}
-\frac{\bmm^\star}{\sqrt{\bv^\star}+\varepsilon} &= \bzero, \label{eq:eq1}\\
\alpha\bigl(\nabla f(\bx^\star)-\bmm^\star\bigr) &= \bzero, \label{eq:eq2}\\
\beta\bigl(\bq(\bx^{\star})-\bv^\star\bigr) &= \bzero. \label{eq:eq3}
\end{align}
Since $\alpha,\beta>0$ and $\sqrt{\bv^\star}+\varepsilon$ is strictly positive componentwise, \eqref{eq:eq1} implies
\[
\bmm^\star=\bzero.
\]
Substituting this into \eqref{eq:eq2} yields
\[
\nabla f(\bx^\star)=\bzero,
\]
so that $\bx^\star$ is a stationary point of the objective function. Finally, \eqref{eq:eq3} gives
\[
\bv^\star=\bq(\bx^\star) = \bq^{\star}.
\]
Therefore, every solution is of the form
\[
\bz^\star=(\bx^\star,\bzero,\bq^{\star}),
\qquad
\nabla f(\bx^\star)=\bzero.
\]

\paragraph{Jacobian matrix at the limiting point.}
We now compute the Jacobian matrix $\bJ=DF(\bz^\star)$ at the limiting point
\[
\bz^\star=(\bx^\star,\bzero,\bq^{\star}),
\qquad
\nabla f(\bx^\star)=\bzero.
\]
Write
\[
F(\bz)=
\begin{pmatrix}
F_1(\bx,\bmm,\bv)\\[1mm]
F_2(\bx,\bmm,\bv)\\[1mm]
F_3(\bx,\bmm,\bv)
\end{pmatrix},
\]
where
\begin{align}
F_1(\bx,\bmm,\bv) &= -\frac{\bmm}{\sqrt{\bv}+\varepsilon},\\
F_2(\bx,\bmm,\bv) &= \alpha\bigl(\nabla f(\bx)-\bmm\bigr),\\
F_3(\bx,\bmm,\bv) &= \beta\bigl(\bq(\bx)-\bv\bigr).
\end{align}

We compute each block derivative.

\medskip
\noindent
\emph{Derivative of the first block.}
Since $F_1$ does not depend explicitly on $\bx$, we have
\[
\frac{\partial F_1}{\partial \bx}(\bz^\star)=\bzero.
\]
Differentiating with respect to $\bmm$ gives
\[
\frac{\partial F_1}{\partial \bmm}(\bz^\star)
=
-\operatorname{diag}\!\left(\frac{1}{\sqrt{\bq^\star}+\varepsilon}\right).
\]
Finally, differentiating $F_1$ with respect to $\bv$ produces terms proportional to $\bmm$, and therefore
\[
\frac{\partial F_1}{\partial \bv}(\bz^\star)=\bzero
\]
since $\bmm^\star=\bzero$.

\medskip
\noindent
\emph{Derivative of the second block.}
For
\[
F_2(\bx,\bmm,\bv)=\alpha\bigl(\nabla f(\bx)-\bmm\bigr),
\]
we immediately obtain
\[
\frac{\partial F_2}{\partial \bx}(\bz^\star)=\alpha \nabla^2 f(\bx^\star),
\qquad
\frac{\partial F_2}{\partial \bmm}(\bz^\star)=-\alpha \bI_d,
\qquad
\frac{\partial F_2}{\partial \bv}(\bz^\star)=\bzero.
\]

\medskip
\noindent
\emph{Derivative of the third block.}
For
\[
F_3(\bx,\bmm,\bv)=\beta\bigl(\bq(\bx)-\bv\bigr),
\]
we have
\[
\frac{\partial F_3}{\partial \bmm}(\bz^\star)=\bzero,
\qquad
\frac{\partial F_3}{\partial \bv}(\bz^\star)=-\beta \bI_d.
\]
Moreover, the derivative of $\bq(\bx)$ with respect to $\bx$ is
$\beta\nabla \bq(\bx^{\star})$.

Collecting all the block derivatives, we conclude that
\[
\bJ = DF(\bz^\star)
=
\begin{pmatrix}
\bzero & - \operatorname{diag}\!\bigl(\frac{1}{\sqrt{\bq^{\star}} + \varepsilon}\bigr) & \bzero \\
\alpha \nabla^2 f(\bx^\star) & -\alpha \bI_d & \bzero\\
\beta \nabla \bq(\bx^{\star}) & \bzero & -\beta \bI_d
\end{pmatrix},
\]

\paragraph{Hurwitz property.}
Write $\bJ$ in the block form
\[
\bJ=
\begin{pmatrix}
\bM & \bzero\\
\bN & -\beta \bI_d
\end{pmatrix},
\]
where
\[
\bM=
\begin{pmatrix}
\bzero & -\bD^\star\\
\alpha \nabla^2 f(\bx^\star) & -\alpha \bI_d
\end{pmatrix},
\qquad
\bN=
\begin{pmatrix}
\beta \nabla \bq(\bx^\star) & \bzero
\end{pmatrix}.
\]
Since $\bJ$ is block lower triangular, its eigenvalues are the union of those of $\bM$ and those of $-\beta \bI_d$. The latter are all equal to $-\beta<0$. Therefore, it remains to show that $\bM$ is Hurwitz.

Let $\mu\in\mathbb C$ be an eigenvalue of $\bM$ with corresponding eigenvector $(\bu^\top,\bv^\top)^\top\neq \bzero$. Then
\[
\begin{pmatrix}
\bzero & -\bD^\star\\
\alpha \nabla^2 f(\bx^\star) & -\alpha \bI_d
\end{pmatrix}
\begin{pmatrix}
\bu\\
\bv
\end{pmatrix}
=
\mu
\begin{pmatrix}
\bu\\
\bv
\end{pmatrix},
\]
which yields
\begin{align}
-\bD^\star \bv &= \mu \bu, \label{eq:adam_eig1}\\
\alpha \nabla^2 f(\bx^\star)\bu - \alpha \bv &= \mu \bv. \label{eq:adam_eig2}
\end{align}
From \eqref{eq:adam_eig1}, we obtain
\[
\bv=-\mu (\bD^\star)^{-1}\bu.
\]
Substituting this into \eqref{eq:adam_eig2} gives
\[
\alpha \nabla^2 f(\bx^\star)\bu + \alpha\mu (\bD^\star)^{-1}\bu
=
-\mu^2 (\bD^\star)^{-1}\bu.
\]
Equivalently,
\[
\alpha \nabla^2 f(\bx^\star)\bu + (\alpha\mu+\mu^2)(\bD^\star)^{-1}\bu=\bzero.
\]

Now define
\[
\bw:=(\bD^\star)^{-1/2}\bu,
\qquad
\bB:=(\bD^\star)^{1/2}\nabla^2 f(\bx^\star)(\bD^\star)^{1/2}.
\]
Then the above equation becomes
\[
\alpha \bB \bw + (\alpha\mu+\mu^2)\bw=\bzero.
\]
Since $\nabla^2 f(\bx^\star)\succ \bzero$ and $\bD^\star\succ \bzero$, it follows that $\bB\succ \bzero$. Hence every eigenvalue $\theta$ of $\bB$ is strictly positive, and $\mu$ must satisfy
\[
\mu^2+\alpha\mu+\alpha\theta=0,
\qquad \theta>0.
\]
Because both coefficients $\alpha$ and $\alpha\theta$ are strictly positive, both roots of this quadratic have strictly negative real parts. Therefore, every eigenvalue of $\bM$ lies in the open left half-plane, and hence $\bM$ is Hurwitz.

Combining this with the fact that $-\beta \bI_d$ is Hurwitz, we conclude that $\bJ$ is Hurwitz.

\subsection{Proof for \Cref{proposition3}}

We prove the two statements separately.

\paragraph{Characterization of the limiting point.}
Recall that the mean-field mapping for AdaFactor is
\begin{equation}
F(\bz)=
\begin{pmatrix}
-\dfrac{\nabla f(\bx)}{\sqrt{\bv}+\varepsilon}\\[2mm]
\beta\bigl(\bq(\bx)-\bv\bigr)
\end{pmatrix},
\qquad
\bz=(\bx,\bv)\in\mathbb R^{2d},
\end{equation}
where the division is understood componentwise.

Let $\bz^\star=(\bx^\star,\bv^\star)$ be a solution of the mean-field function. Then
\[
F(\bz^\star)=\bzero,
\]
which is equivalent to
\begin{align}
-\frac{\nabla f(\bx^\star)}{\sqrt{\bv^\star}+\varepsilon} &= \bzero, \label{eq:adafactor_eq1}\\
\beta\bigl(\bq(\bx^\star)-\bv^\star\bigr) &= \bzero. \label{eq:adafactor_eq2}
\end{align}
Since $\varepsilon>0$, the vector $\sqrt{\bv^\star}+\varepsilon$ is strictly positive componentwise. Therefore, \eqref{eq:adafactor_eq1} implies
\[
\nabla f(\bx^\star)=\bzero.
\]
Substituting this identity into \eqref{eq:adafactor_eq2}, and using $\beta>0$, we obtain
\[
\bv^\star=\bq(\bx^\star) = \bq^\star.
\]
Hence, every solution is of the form
\[
\bz^\star=(\bx^\star,\bv^\star),
\quad
\nabla f(\bx^\star)=\bzero, \quad \bv^\star = \bq^\star.
\]

\paragraph{Computation of the Jacobian matrix.}
We now compute the Jacobian matrix $\bJ=DF(\bz^\star)$ at the limiting point
\[
\bz^\star=(\bx^\star,\bv^\star),
\qquad
\nabla f(\bx^\star)=\bzero.
\]
Write
\[
F(\bz)=
\begin{pmatrix}
F_1(\bx,\bv)\\[1mm]
F_2(\bx,\bv)
\end{pmatrix},
\]
where
\begin{align}
F_1(\bx,\bv) &= -\frac{\nabla f(\bx)}{\sqrt{\bv}+\varepsilon},\\
F_2(\bx,\bv) &= \beta\bigl(\bq(\bx)-\bv\bigr).
\end{align}

We compute each block derivative.

\medskip
\noindent
\emph{Derivative of the first block.}
Since
\[
F_1(\bx,\bv)
=
-\operatorname{diag}\!\left(\frac{1}{\sqrt{\bv}+\varepsilon}\right)\nabla f(\bx),
\]
we have
\[
\frac{\partial F_1}{\partial \bx}(\bz^\star)
=
-\operatorname{diag}\!\left(\frac{1}{\sqrt{\bv^\star}+\varepsilon}\right)
\nabla^2 f(\bx^\star).
\]

Next, we compute the derivative with respect to $\bv$. Componentwise,
\[
[F_1]_i(\bx,\bv)
=
-\frac{\partial_i f(\bx)}{\sqrt{v_i}+\varepsilon}.
\]
Hence
\[
\frac{\partial [F_1]_i}{\partial v_i}(\bz^\star)
=
\frac{\partial_i f(\bx^\star)}{2\sqrt{v_i^\star}(\sqrt{v_i^\star}+\varepsilon)^2}.
\]
Since $\nabla f(\bx^\star)=\bzero$, we conclude that
\[
\frac{\partial F_1}{\partial \bv}(\bz^\star)=\bzero.
\]

\medskip
\noindent
\emph{Derivative of the second block.}
For
\[
F_2(\bx,\bv)=\beta\bigl(\bq(\bx)-\bv\bigr),
\]
we have
\begin{equation*}
    \frac{\partial F_2}{\partial \bx}(\bz^\star) = \beta \nabla \bq(\bx^\star),
\end{equation*}
and
\begin{equation*}
    \frac{\partial F_2}{\partial \bv}(\bz^\star) = -\beta\bI_d.
\end{equation*}

Collecting all block derivatives, we obtain
\[
\bJ=DF(\bz^\star)
=
\begin{pmatrix}
-\operatorname{diag}\!\left(\frac{1}{\sqrt{\bv^\star}+\varepsilon}\right)
\nabla^2 f(\bx^\star) & \bzero\\[2mm]
\beta \nabla \bq(\bx^\star) & -\beta \bI_d
\end{pmatrix}.
\]

\paragraph{Hurwitz property.}
The matrix $\bJ$ is block lower triangular. Therefore, its eigenvalues are the union of those of
\[
-\operatorname{diag}\!\left(\frac{1}{\sqrt{\bv^\star}+\varepsilon}\right)\nabla^2 f(\bx^\star)
\]
and those of
\[
-\beta \bI_d.
\]
Since $\beta>0$, the block $-\beta \bI_d$ is Hurwitz.

Now assume that $\nabla^2 f(\bx^\star)\succ \bzero$. Since $\varepsilon>0$, the diagonal matrix
\[
\bD^\star:=\operatorname{diag}\!\left(\frac{1}{\sqrt{\bv^\star}+\varepsilon}\right)
\]
is positive definite. Then
\[
\bD^\star \nabla^2 f(\bx^\star)
\]
has strictly positive eigenvalues, because it is similar to
\[
(\bD^\star)^{1/2}\nabla^2 f(\bx^\star)(\bD^\star)^{1/2}\succ \bzero.
\]
Hence
\[
-\bD^\star \nabla^2 f(\bx^\star)
\]
has strictly negative eigenvalues. Therefore, the upper-left block is Hurwitz. Since both diagonal blocks are Hurwitz, $\bJ$ is Hurwitz.

\newpage

\section{Proofs for Some Technical Lemmas}
\label{appendix_proof_technical_lemmas}
\subsection{Proof for \Cref{lemma2}}
\begin{proof}
Since $\bJ$ is Hurwitz, by the Lyapunov theorem, there exists a symmetric positive definite matrix $\bP\succ 0$ such that
\begin{equation*}
\bJ^\top \bP + \bP \bJ = -\bI.
\end{equation*}
Define the vector norm
\begin{equation*}
\|\bx\|_* := (\bx^\top \bP \bx)^{1/2},
\qquad \bx\in\mathbb{R}^d,
\end{equation*}
and let $\|\cdot\|_*$ denote the associated induced matrix norm:
\begin{equation*}
\|\bA\|_* := \sup_{\bx\neq \bzero}\frac{\|\bA\bx\|_*}{\|\bx\|_*}.
\end{equation*}
Since $\|\cdot\|_*$ is an induced matrix norm, it is submultiplicative:
\begin{equation}
\label{eq6}
\|\bA\bB\|_* \le \|\bA\|_*\,\|\bB\|_*.
\end{equation}

Now fix $\gamma>0$. For any $\bx\in\mathbb{R}^d$,
\[
\|(\bI+\gamma \bJ)\bx\|_*^2
=
\bx^\top (\bI+\gamma \bJ)^\top \bP (\bI+\gamma \bJ)\bx.
\]
Expanding,
\[
(\bI+\gamma \bJ)^\top \bP (\bI+\gamma \bJ)
=
\bP+\gamma(\bJ^\top \bP+\bP \bJ)+\gamma^2 \bJ^\top \bP \bJ.
\]
Using $\bJ^\top \bP+\bP\bJ=-\bI$, we get
\[
(\bI+\gamma \bJ)^\top \bP (\bI+\gamma \bJ)
=
\bP-\gamma \bI+\gamma^2 \bJ^\top \bP \bJ.
\]
Since $\bP\succ \bzero$, there exists $b>0$ such that
\[
\bJ^\top \bP \bJ \preceq b \bP.
\]
Also, since $\bP\succ \bzero$, there exists $a>0$ such that
\[
\bI \succeq a \bP.
\]
Therefore
\[
(\bI+\gamma \bJ)^\top \bP (\bI+\gamma \bJ)
\preceq
\bP-a\gamma \bP+b\gamma^2 \bP
=
(1-a\gamma+b\gamma^2)\bP.
\]
Hence
\[
\|\bI+\gamma \bJ\|_*^2 \le 1-a\gamma+b\gamma^2.
\]
Choose $\gamma_*>0$ small enough so that for all $0<\gamma\le \gamma_*$,
\[
1-a\gamma+b\gamma^2 \le (1-c\gamma)^2
\]
for some $c\in(0,a/2)$. Then
\[
\|\bI+\gamma \bJ\|_* \le 1-c\gamma,
\qquad 0<\gamma\le \gamma_*.
\]
Since $\gamma_k\to 0$, there exists $N$ such that $\gamma_k\le \gamma_*$ for all $k\ge N$. Thus for all $k\ge N$,
\begin{equation*}
\|\bI+\gamma_k \bJ\|_* \le 1-c\gamma_k.
\end{equation*}

By submultiplicativity in \eqref{eq6},
\[
\|\Phi_{t,k}\|_*
=
\left\|
\prod_{j=k}^t (\bI+\gamma_j \bJ)
\right\|_*
\le
\prod_{j=k}^t \|\bI+\gamma_j \bJ\|_*
\le
\prod_{j=k}^t (1-c\gamma_j).
\]

Finally, since all matrix norms on $\mathbb{R}^{d\times d}$ are equivalent, there exists $C_0>0$ such that
\[
\|\bA\|_2 \le C_0 \|\bA\|_*
\]
for all matrices $\bA$, which yields
\[
\|\Phi_{t,k}\|_2
\le
C_0\prod_{j=k}^n (1-c\gamma_j).
\]
\end{proof}

\subsection{Proof for \Cref{lemma3}}
Since $\bJ$ is Hurwitz, the Lyapunov equation
\begin{equation*}
\bJ\bSig+\bSig \bJ^\top + \bQ = \bzero
\end{equation*}
admits a unique positive semidefinite solution $\bSig$.
We show that $\bU_t\to \bSig$.

\medskip
\noindent
\textbf{Step 1: a recursion for $\bU_t$.}

By definition,
\[
\bU_{t+1}
=
\sum_{k=1}^{t+1}
\frac{\gamma_k^2}{\gamma_{t+1}}\,
\Phi_{t+1,k+1}\bQ\Phi_{t+1,k+1}^\top.
\]
Separating the last term $k=t+1$, and noting that $\Phi_{t+1,t+2}=\bI$, we obtain
\[
\bU_{t+1}
=
\frac{\gamma_t}{\gamma_{t+1}}
(\bI+\gamma_{t+1}\bJ)
\left(
\sum_{k=1}^{t}
\frac{\gamma_k^2}{\gamma_t}\,
\Phi_{t,k+1}\bQ\Phi_{t,k+1}^\top
\right)
(\bI+\gamma_{t+1}\bJ)^\top
+
\gamma_{t+1}\bQ.
\]
Hence
\begin{equation}
\label{eq7}
\bU_{t+1}
=
\frac{\gamma_t}{\gamma_{t+1}}
(\bI+\gamma_{t+1}\bJ)\bU_t(\bI+\gamma_{t+1}\bJ)^\top
+
\gamma_{t+1}\bQ.
\end{equation}

\medskip
\noindent
\textbf{Step 2: expand the ratio $\gamma_t/\gamma_{t+1}$.}

Since
\[
\gamma_t=\gamma_0 t^{-\kappa},
\qquad
\kappa\in(1/2,1),
\]
we have
\[
\frac{\gamma_t}{\gamma_{t+1}}
=
\left(\frac{t+1}{t}\right)^\kappa
=
1+\frac{\kappa}{t}+o(t^{-1}).
\]
Because $\kappa<1$, we have
\[
\frac{1}{t}=o(\gamma_t),
\]
and thus
\begin{equation*}
\frac{\gamma_t}{\gamma_{t+1}} = 1+o(\gamma_t).
\end{equation*}

Therefore, \eqref{eq7} implies that
\begin{align*}
\bU_{t+1}
&=
(1+o(\gamma_t))
(\bI+\gamma_{t+1}\bJ)\bU_t(\bI+\gamma_{t+1}\bJ)^\top
+\gamma_{t+1}\bQ.
\end{align*}
and expanding the product yields
\begin{align*}
\bU_{t+1}
&=
\bU_t
+\gamma_t(\bJ\bU_t+\bU_t\bJ^\top + \bQ)
+\bR_t,
\end{align*}
where the residual term satisfies
\begin{equation*}
\|\bR_t\|
\le
o(\gamma_t)(1+\|\bU_t\|).
\end{equation*}
Note that 
\begin{equation*}
    \left\|\bU_t \right\| \leq C\gamma_t^{-1} \sum_{k=1}^{t} \prod_{j=k+1}^{t}\left(1 - c\gamma_j\right)^2 \gamma_k^2 = \mathcal{O}\left(1\right),
\end{equation*}
we have 
\begin{equation*}
\|\bR_t\|
\le
o(\gamma_t).
\end{equation*}

\medskip
\noindent
\textbf{Step 3: compare $\bU_t$ with the Lyapunov solution $\bSig$.}

Define
\[
\bE_t:=\bU_t-\bSig.
\]
Since $\bSig$ solves
\[
\bJ\bSig+\bSig \bJ^\top+\bQ=\bzero,
\]
we get
\begin{equation}
\label{eq8}
    \begin{split}
        \bE_{t+1}
&=
\bE_t+\gamma_t(\bJ\bE_t+\bE_t\bJ^\top)+\bR_t\\
& = \left( \bI + \gamma_t\bJ\right)\bE_t\left( \bI + \gamma_t\bJ^{\top}\right) + \widetilde{\bR}_{t},
    \end{split}
\end{equation}
where $\widetilde{\bR}_{t} = \bR_t - \gamma_t^2 \bJ \bE_t\bJ^{\top}$ and thus $\left\|\widetilde{\bR}_{t} \right\| = o(\gamma_t)$.

\medskip
\noindent
\textbf{Step 4: stability of the homogeneous part.}

Applying the recursion for $\bE_t$ in \eqref{eq8} multiple times, we obtain
\begin{equation*}
\bE_t
=
\Phi_{t,1}\bE_1\Phi_{t,1}^{\top}
+
\sum_{k=1}^{t-1}\Phi_{t,k+1}\widetilde{\bR}_k\Phi_{t,k+1}^{\top}.
\end{equation*}
The first term vanishes because of the exponential stability:
\[
\Phi_{t,1}\bE_1\Phi_{t,1}^{\top} \to \bzero.
\]
For the second term, since $\left\|\widetilde{\bR}_k\right\|=o(\gamma_k)$, we have
\begin{equation*}
\begin{split}
        & \left\| \sum_{k=1}^{t-1}\Phi_{t,k+1}\widetilde{\bR}_k\Phi_{t,k+1}^{\top}\right\|\\
        & \leq \sum_{k=1}^{t-1} \left\|\Phi_{t,k+1} \right\|^2\left\| \widetilde{\bR}_k \right\|\\
        & \leq C \sum_{k=1}^{t-1} \prod_{j=k+1}^{t-1} \left(1-c\gamma_j\right)^2 \left\|\widetilde{\bR}_k\right\| = o(1).
\end{split}
\end{equation*}

Hence
\[
\bE_t\to \bzero,
\]
that is,
\[
\bU_t\to \bSig.
\]
This completes the proof.

\end{document}